\documentclass[12pt]{article}
\usepackage{fullpage,graphicx,psfrag,url,amsmath,amsthm,amssymb,algorithm}
\usepackage{algorithmicx,algpseudocode,comment}
\usepackage[english]{babel}
\usepackage[utf8]{inputenc}
\newcommand{\BEAS}{\begin{eqnarray*}}
\newcommand{\EEAS}{\end{eqnarray*}}
\newcommand{\BEA}{\begin{eqnarray}}
\newcommand{\EEA}{\end{eqnarray}}
\newcommand{\BEQ}{\begin{equation}}
\newcommand{\EEQ}{\end{equation}}
\newcommand{\BIT}{\begin{itemize}}
\newcommand{\EIT}{\end{itemize}}
\newcommand{\BNUM}{\begin{enumerate}}
\newcommand{\ENUM}{\end{enumerate}}

\newcommand{\beas}{\begin{eqnarray*}}
\newcommand{\eeas}{\end{eqnarray*}}
\newcommand{\bea}{\begin{eqnarray}}
\newcommand{\eea}{\end{eqnarray}}
\newcommand{\beq}{\begin{equation}}
\newcommand{\eeq}{\end{equation}}
\newcommand{\bit}{\begin{itemize}}
\newcommand{\eit}{\end{itemize}}
\newcommand{\bnum}{\begin{enumerate}}
\newcommand{\enum}{\end{enumerate}}

\newcommand{\ba}{\begin{array}}
\newcommand{\ea}{\end{array}}
\newcommand{\bbm}{\begin{bmatrix}}
\newcommand{\ebm}{\end{bmatrix}}

\newcommand{\eg}{{\it e.g.}}
\newcommand{\ie}{{\it i.e.}}

\newcommand{\ones}{\mathbf 1}

\newcommand{\reals}{{\mbox{\bf R}}}

\newcommand{\rank}{\mbox{\textrm{Rank}}}

\newcommand{\tr}{\mathop{\bf tr}}

\newcommand{\diag}{\mathop{\bf diag}}

\newcommand{\Expect}{\mathop{\bf E{}}}

\newcommand{\sign}{\mathop{\bf sign}}

\newcommand{\argmin}{\mathop{\rm argmin}}

\newcommand{\Card}{\mathop{\bf card}}

\newcommand{\card}{\mathop{\bf card}}

\newcommand{\argmax}{\mathop{\rm argmax}}

\newcommand{\prox}{{\bf prox}}
\newcommand{\update}{\mathop{\bf update}}

\newcommand{\huber}{\mathop{\bf huber}}
\newcommand{\sort}{\mathop{\bf sort}}
\newcommand{\round}{\mathop{\bf round}}

\newcommand{\half}{(1/2)}

\theoremstyle{definition} 
\newtheorem{theorem}{Theorem}
\newtheorem{lemma}{Lemma}

{\begin{quote}}{\end{quote}}

\makeatletter
\newlength \figwidth
\if@twocolumn
\else
\fi
\makeatother

\makeatletter
\long\def\@makecaption#1#2{
   \vskip 9pt 
   \begin{small}
   \setbox\@tempboxa\hbox{{\bf #1:} #2}
   \ifdim \wd\@tempboxa > 5.5in
        \begin{center}
        \begin{minipage}[t]{5.5in}
        \addtolength{\baselineskip}{-0.95pt}
        {\bf #1:} #2 \par
        \addtolength{\baselineskip}{0.95pt}
        \end{minipage}
        \end{center}
   \else 
	\hbox to\hsize{\hfil\box\@tempboxa\hfil}  
   \fi
   \end{small}\par
}
\makeatother

\usepackage{pgfplots}
\pgfplotsset{compat=newest}
\pgfplotsset{every axis legend/.append style={%
cells={anchor=west}}
}
\usetikzlibrary{arrows}
\tikzset{>=stealth'}
\pgfplotsset{width=12cm}

\newcommand{\figures}{.}

\excludecomment{unfinished}
\excludecomment{comment}
\excludecomment{timeseries}

\title{Generalized Low Rank Models}
\author{Madeleine Udell, Corinne Horn, Reza Zadeh, and Stephen Boyd}
\date{\today.  (Original version posted September 2014.)}

\begin{document}
\maketitle

\begin{abstract}
Principal components analysis (PCA) is a well-known technique for approximating
a tabular data set by a low rank matrix.
Here, we extend the idea of PCA to handle arbitrary data sets consisting
of numerical, Boolean, categorical, ordinal, and other data types.
This framework encompasses many well known techniques in data analysis,
such as nonnegative matrix factorization, 
matrix completion, sparse and robust PCA, $k$-means, $k$-SVD,
and maximum margin matrix factorization.
The method handles heterogeneous data sets,
and leads to coherent schemes for compressing, denoising, and imputing 
missing entries across all data types simultaneously.
It also admits a number of interesting interpretations of the low rank factors,
which allow clustering of examples or of features.
We propose several parallel algorithms
for fitting generalized low rank models,
and describe implementations and numerical results.
\end{abstract}

This manuscript is a draft. Comments sent to \verb|udell@stanford.edu| are welcome.

\newpage
\setcounter{tocdepth}{2}
\tableofcontents
\newpage

\section{Introduction}
In applications of machine learning and data mining,
one frequently encounters large collections of high dimensional data
organized into a table. 
Each row in the table represents an example, and each column a feature or attribute.
These tables may have columns of different (sometimes, non-numeric) types, 
and often have many missing entries.

For example, in medicine, the table might record patient attributes or lab tests: 
each row of the table lists test or survey results for a particular patient, 
and each column corresponds to a distinct test or survey question. 
The values in the table might be numerical (3.14), Boolean (yes, no),
ordinal (never, sometimes, always), or categorical (A, B, O).
Tests not administered or questions left blank result in missing entries in the data set.
Other examples abound: in finance, the table might record known characteristics of 
companies or asset classes; in social science settings, it might record survey responses;
in marketing, it might record known customer characteristics and purchase history.

Exploratory data analysis can be difficult in this setting.
To better understand a complex data set, one would like to be able to visualize
archetypical examples, to cluster examples, to find correlated features, to fill in
(impute) missing entries, and to remove (or simply identify) spurious, anomalous,
or noisy data points.
This paper introduces a templated method to enable these analyses even on large data sets
with heterogeneous values and with many missing entries. 
Our approach will be to embed
both the rows (examples) and columns (features) of the table into the same
low dimensional vector space. 
These low dimensional vectors can then be plotted, clustered, 
and used to impute missing entries or identify anomalous ones.

If the data set consists only of numerical (real-valued) data, then a simple
and well-known technique to find this embedding is Principal
Components Analysis (PCA). PCA finds a low rank matrix that minimizes the
approximation error, in the least-squares sense, to the original data set.
A factorization of this low rank matrix embeds the original high 
dimensional features into a low dimensional space.
Extensions of PCA can handle missing data values, and can be used to 
impute missing entries.

Here, we extend PCA to approximate an arbitrary data set 
by replacing the least-squares error used in PCA 
with a loss function that is appropriate for the given data type.
Another extension beyond PCA is to add regularization on the low 
dimensional factors 
to impose or encourage some structure, such as sparsity or nonnegativity,
in the low dimensional factors.
In this paper we use the term \emph{generalized low rank model} (GLRM) 
to refer to the problem of approximating a data set as a product of two low dimensional factors
by minimizing an objective function.
The objective will consist of a loss function on the approximation error together 
with regularization of the low dimensional factors.
With these extensions of PCA, the resulting low rank representation 
of the data set still produces a low dimensional embedding of the data set, as in PCA.

Many of the low rank modeling problems we must solve will be familiar. 
We recover an optimization
formulation of nonnegative matrix factorization, 
matrix completion, sparse and robust PCA, $k$-means, $k$-SVD,
and maximum margin matrix factorization, to name just a few.%

These low rank approximation problems are not convex, and 
in general cannot be solved globally and efficiently.
There are a few exceptional problems that are known to have convex relaxations
which are tight under certain conditions, and hence are efficiently (globally) solvable
under these conditions. 
However, all of these approximation problems can be heuristically (locally) solved 
by methods that alternate
between updating the two factors in the low rank approximation.
Each step involves either a convex problem, or a nonconvex problem that is simple enough
that we can solve it exactly.
While these alternating methods need not find the globally best low rank
approximation, they are often very useful and effective for
the original data analysis problem.

\subsection{Previous work}

\paragraph{Unified views of matrix factorization.}
We are certainly not the first to note that matrix factorization algorithms 
may be viewed in a unified framework,
parametrized by a small number of modeling decisions.
The first instance we find in the literature of this unified view appeared 
in a paper by Collins, Dasgupta, and Schapire, \cite{collins2001}, %
extending PCA to use loss functions derived from any probabilistic model in the exponential family.
Gordon's Generalized$^2$ Linear$^2$ models \cite{gordon2002} extended
the framework to loss functions derived from the generalized Bregman divergence
of any convex function, which includes models such as Independent Components Analysis (ICA).
Srebro's 2004 PhD thesis \cite{srebro2004thesis} extended
the framework to other loss functions, including hinge loss
and KL-divergence loss, and to other regularizers, including the nuclear norm and max-norm.
Similarly, Chapter 8 in Tropp's 2004 PhD thesis \cite{tropp2004thesis} explored a number of new regularizers,
presenting a range of clustering problems as matrix factorization problems with constraints, 
and anticipated the $k$-SVD algorithm \cite{aharon2006}.
Singh and Gordon \cite{singh2008} offered a complete view of the state of the literature
on matrix factorization in Table~1 of their 2008 paper,
and noted that by changing the loss function and regularizer, one may 
recover algorithms including PCA, weighted PCA,
$k$-means, $k$-medians, $\ell_1$ SVD,
probabilistic latent semantic indexing (pLSI),
nonnegative matrix factorization with $\ell_2$ or KL-divergence loss,
exponential family PCA, and MMMF.
Witten et al.\ introduced the statistics community to sparsity-inducing matrix factorization
in a 2009 paper on penalized matrix decomposition,
with applications to sparse PCA and canonical correlation analysis \cite{witten2009}.
Recently, Markovsky's monograph on low rank approximation \cite{markovsky2012} 
reviewed some of this literature,
with a focus on applications in system, control, and signal processing.
The GLRMs discussed in this paper include all of these models, and many more.


\paragraph{Heterogeneous data.}
Many authors have proposed the use of low rank models as a tool for
integrating heterogeneous data. 
The earliest example of this approach is canonical correlation analysis,
developed by Hotelling \cite{hotelling1936} in 1936 to understand the relations between
two sets of variates in terms of the eigenvectors of their covariance matrix.
This approach was extended by Witten et al.\ \cite{witten2009} to encourage structured 
(\eg, sparse) factors.
In the 1970s, De Leeuw et al.\ proposed the use of low rank models to fit 
data measured in nominal, ordinal and cardinal levels \cite{deleeuw1976}.
More recently, Goldberg et al.\ \cite{goldberg2010} used a low rank model to perform
transduction (\ie, multi-label learning) in the presence of missing data by fitting
a low rank model to the features and the labels simultaneously.
Low rank models have also been used to embed image, text and video data
into a common low dimensional space \cite{gress2004},
and have recently come into vogue in the natural language processing community
as a means to embed words and documents into a low dimensional vector space 
\cite{mikolov2013a,mikolov2013b,pennington2014, srikumar2014}.

\paragraph{Algorithms.}
In general, it can be computationally hard to find the global optimum of a generalized low rank model.
For example, it is NP-hard to compute an exact solution to
$k$-means \cite{drineas2004}, nonnegative matrix factorization \cite{vavasis2009},
and weighted PCA and matrix completion \cite{gillis2011}, 
all of which are special cases of low rank models.

However, there are many (efficient) ways to go about \emph{fitting} a low rank model,
by which we mean finding a good model with a small objective value.
The resulting model may or may not be the global solution of the low rank optimization problem.
We distinguish a model fit in this way from the \emph{solution} to an optimization problem, 
which always refers to the global solution.

The matrix factorization literature presents a wide
variety of methods to fit low rank models in a variety of special cases.
For example, there are variants on alternating minimization 
(with alternating least squares as a special case)
\cite{deleeuw1976,young1976,takane1977,deleeuw1984,deleeuw2009},
alternating Newton methods \cite{gordon2002, singh2008},
(stochastic or incremental) gradient descent \cite{keshavan2009grassman,lee2010,niu2011,recht2011,bittorf2012,yun2013,recht2013},
conjugate gradients \cite{rennie2005,srebro2003},
expectation minimization (EM) (or ``soft-impute'') methods \cite{tipping1999, srebro2003, mazumder2010, hastiefastals},
multiplicative updates \cite{lee1999},
and convex relaxations to semidefinite programs \cite{srebro2004, fazel2004, recht2010, fithian2013}.

Generally, expectation minimization, which proceeds by iteratively
imputing missing entries in the matrix and solving the fully observed problem,
has been found to underperform relative to other methods \cite{singh2008}.
However, when used in conjunction with computational tricks exploiting 
a particular problem structure, such as Gram matrix caching,
these methods can still work extremely well \cite{hastiefastals}.

Semidefinite programming becomes computationally intractable 
for very large (or even just large) scale problems \cite{rennie2005}.
However, a theoretical analysis of optimality conditions 
for rank-constrainted semidefinite programs \cite{burer2003theory} 
has led to a few algorithms for semidefinite programming
based on matrix factorization \cite{burer2003alg,abernethy2009,journee2010} 
which guarantee global optimality 
and converge quickly if the global solution to the problem is exactly low rank.
Fast approximation algorithms for rank-constrained 
semidefinite programs have also been developed \cite{shalev2011}.

Recently, there has been a resurgence of interest in methods
based on alternating minimization, as numerous authors have shown
that alternating minimization (suitably initialized, and under a few technical assumptions) provably converges to the global minimum for a range of problems
including matrix completion \cite{keshavan2012, jain2013, hardt2013},
robust PCA \cite{netrapalli2014}, and
dictionary learning \cite{agarwal2013}.

Gradient descent methods are often preferred for extremely large scale problems
since these methods parallelize naturally in both shared memory
and distributed memory architectures.
See \cite{recht2013,yun2013} and references therein for some recent innovative approaches to 
speeding up stochastic gradient descent for matrix factorization by eliminating locking and 
reducing interprocess communication.

\paragraph{Contributions.}
The present paper differs from previous work in a number of ways.
We are consistently concerned with the \emph{meaning} of applying these
different loss functions and regularizers to approximate a data set.
The generality of our view allows us to introduce a number of loss functions and regularizers
that have not previously been considered.
Moreover, our perspective enables us to
extend these ideas to arbitrary data sets, rather than just matrices of real numbers.

A number of new considerations emerge when considering the problem so broadly.
First, we must face the problem of comparing approximation errors
across data of different types. For example,
we must choose a scaling to trade off the loss due to a misclassification of a categorical value 
with an error of $0.1$ (say) in predicting a real value.

Second, we require algorithms that can handle the full gamut of losses and regularizers,
which may be smooth or nonsmooth, finite or infinite valued, with arbitrary domain.
This work is the first to consider these problems in such generality,
and therefore also the first to wrestle with the algorithmic consequences.
Below, we give a number of algorithms appropriate for this setting, 
including many that have not been previously proposed in the literature.
Our algorithms are all based on alternating minimization and 
variations on alternating minimization that are more suitable for large scale data
and can take advantage of parallel computing resources.

Finally, we present some new results on some old problems. For example,
in Appendix~\ref{a-qpca}, we derive a formula for the solution to quadratically regularized PCA,
and show that quadratically regularized PCA has no local nonglobal minima;
and in \S\ref{s-optimality-certificate} we show how to certify (in some special cases) 
that a model is a global solution of a GLRM.

\subsection{Organization}
The organization of this paper is as follows. In \S\ref{s-qpca} we first recall some
properties of PCA and its common variations to familiarize the reader with our
notation. We then generalize the regularization on the low dimensional factors in \S\ref{s-rpca},
and the loss function on the approximation error in \S\ref{s-gpca}. Returning to the setting of
heterogeneous data, we extend these dimensionality reduction techniques to
abstract data types in \S\ref{s-apca} and to multi-dimensional loss functions in \S\ref{s-mpca}.
Finally, we address algorithms for fitting GLRMs in \S\ref{s-algorithms}, 
discuss a few practical considerations in choosing a GLRM for a particular problem in \S\ref{s-choosing}, 
and describe some implementations of the algorithms that we have developed in \S\ref{s-implementation}.

\section{PCA and quadratically regularized PCA} \label{s-qpca}

\paragraph{Data matrix.}
In this section, we let $A \in \reals^{m \times n}$ be a data matrix 
consisting of $m$ examples each with $n$ numerical features. 
Thus $A_{ij}\in\reals$ is the value of the $j$th feature in 
the $i$th example, the $i$th row of $A$ is the vector of $n$ feature values 
for the $i$th example, and the $j$th column of $A$ is the vector of
the $j$th feature across our set of $m$ examples.

It is common to represent other data types in a numerical matrix using
certain canonical encoding tricks.
For example, Boolean data is often encoded as 1 (for true) and -1 (for false),
ordinal data is often encoded using consecutive integers to represent the consecutive
levels of the variable, and
categorical data is often encoded by creating a column for each possible 
value of the categorical variable, and representing the data using a 1
in the column corresponding to the observed value, and -1 or 0 in all other columns.
We will see more systematic and principled ways to deal with these data types,
and others, in \S\ref{s-gpca}--\ref{s-mpca}.
For now, we assume the entries in the data matrix consist of real numbers.

\subsection{PCA}
Principal components analysis (PCA) is one of the oldest and most widely used tools 
in data analysis \cite{pearson1901,hotelling1933,jolliffe1986}. 
We review some of its well-known properties here in order to set notation and
as a warm-up to the variants presented later.

PCA seeks the best rank-$k$ approximation to the matrix $A$ in the least-squares sense,
by solving
\BEQ
\label{eq-pca-z}
\begin{array}{ll}
    \mbox{minimize} & \|A - Z\|_F^2 \\
    \mbox{subject to} & \rank(Z) \leq k,
\end{array}
\EEQ
with variable $Z\in \reals^{m \times n}$. 
Here, $\|\cdot\|_F$ is the Frobenius norm of a matrix, 
\ie, the square root of the sum of the squares of the entries.

The rank constraint can be encoded implicitly by expressing $Z$ in factored form as
$Z = XY$, with $X \in \reals^{m \times k}$, $Y \in \reals^{k \times n}$.
Then the PCA problem can be expressed as
\BEQ
\label{eq-pca}
\begin{array}{ll}
    \mbox{minimize} & \|A - XY\|_F^2
\end{array}
\EEQ
with variables $X \in \reals^{m \times k}$ and $Y \in \reals^{k \times n}$. 
(The factorization of $Z$ is of course not unique.)

Define $x_i \in \reals^{1 \times n}$ to be the $i$th \emph{row} of $X$, 
and $y_j\in \reals^m$ to be the $j$th \emph{column} of $Y$.
Thus $x_i y_j = (XY)_{ij} \in \reals$ denotes a dot or inner product.
(We will use this notation throughout the paper.)
Using this definition, we can rewrite the objective in problem~(\ref{eq-pca}) as
\[
\sum_{i=1}^m \sum_{j=1}^n(A_{ij} - x_i y_j)^2.
\]

We will give several interpretations of the low rank factorization $(X,Y)$
solving (\ref{eq-pca}) in \S\ref{s-qpca-interpretation}.  But for now, we note 
that (\ref{eq-pca}) can interpreted as a method for compressing the
$n$ features in the original data set to $k<n$ new features.
The row vector $x_i$ is associated  with example $i$; we can think of it 
as a feature vector for the example using the compressed set of $k<n$ features.
The column vector $y_j$ is associated with the original feature $j$;
it can be interpreted as mapping the $k$ new features onto the original feature $j$.

\subsection{Quadratically regularized PCA}
We can add quadratic regularization on $X$ and $Y$ to the objective.
The quadratically regularized PCA problem is
\BEQ
\label{eq-qpca}
\begin{array}{ll}
    \mbox{minimize} & \sum_{i = 1}^m\sum_{j = 1}^n(A_{ij} - x_iy_j)^2 + 
    \gamma\sum_{i = 1}^m \|x_i\|_2^2 + \gamma \sum_{j = 1}^n\|y_j\|_2^2,
\end{array}
\EEQ
with variables $X \in \reals^{m \times k}$ and $Y \in \reals^{k \times n}$,
and regularization parameter $\gamma \ge 0$. 
Problem (\ref{eq-qpca}) can be written more concisely in matrix form as
\BEQ
\begin{array}{ll}
    \mbox{minimize} & \|A - XY\|_F^2 + \gamma \|X\|_F^2 + \gamma\|Y\|_F^2.
\end{array}
\EEQ
When $\gamma = 0$, the problem reduces to the PCA problem (\ref{eq-pca}).


\subsection{Solution methods}

\paragraph{Singular value decomposition.}
It is well known that a solution to (\ref{eq-pca}) can be obtained by truncating 
the \emph{singular value decomposition} (SVD) of $A$ \cite{eckart1936}.
The (compact) SVD of $A$ is given by $A = U \Sigma V^T$, 
where $U \in \reals^{m\times r}$ and $V\in \reals^{n\times r}$ have orthonormal columns, and
$\Sigma = \diag(\sigma_1, \ldots, \sigma_r) \in \reals^{r\times r}$, 
with $\sigma_1 \geq \cdots \geq \sigma_r>0$ and $r = \rank(A)$.
The columns of $U = [u_1 \cdots u_r]$ and $V = [v_1 \cdots v_r]$ 
are called the left and right singular vectors of $A$, respectively,
and $\sigma_1, \ldots, \sigma_r$ are called the singular values of $A$.

Using the orthogonal invariance of the Frobenius norm, we can rewrite the objective
in problem~(\ref{eq-pca-z}) as
\[
\|A - XY\|_F^2 = \|\Sigma - U^T XY V\|_F^2.
\]
That is, we would like to find a matrix $U^T XY V$ of rank no more than $k$ approximating
the diagonal matrix $\Sigma$.
It is easy to see that there is no better rank $k$ approximation for $\Sigma$ than 
$\Sigma_k = \diag(\sigma_1, \ldots, \sigma_k, 0, \ldots, 0) \in \reals^{r\times r}$.
Here we have \emph{truncated} the SVD to keep only the top $k$ singular values.
We can achieve this approximation by choosing $U^T XY V = \Sigma_k$, 
or (using the orthogonality of $U$ and $V$) $XY = U \Sigma_k V^T$. For example,
define
\BEQ
\label{eq-svd-trunc}
U_k = [u_1 \cdots u_k], \quad
V_k = [v_1 \cdots v_k], 
\EEQ
and let
\BEQ
\label{eq-pca-soln}
X = U_k \Sigma_k^{1/2}, \qquad Y = \Sigma_k ^{1/2} V_k^T.
\EEQ
The solution to (\ref{eq-qpca}) is clearly not unique: if $X$, $Y$ is a
solution, then so is $XG$, $G^{-1}Y$ for any invertible matrix $G \in
\reals^{k \times k}$. 
When $\sigma_k > \sigma_{k+1}$, all solutions to the PCA problem have this form.
In particular, letting $G = tI$ and taking $t \to \infty$, 
we see that the solution set of the PCA problem is unbounded.

It is less well known that a solution to the quadratically regularized PCA problem can 
be obtained in the same way. 
(Proofs for the statements below can be found in Appendix \ref{a-qpca}.) 
Define $U_k$ and $V_k$ as above, and let 
$\tilde \Sigma_k = \diag((\sigma_1-\gamma)_+,\ldots, (\sigma_k-\gamma)_+)$, 
where $(a)_+ = \max(a,0)$.
Here we have both \emph{truncated} the SVD to keep only the top $k$ singular values,
and performed \emph{soft-thresholding} on the singular values to reduce their
values by $\gamma$.
A solution to the 
quadratically regularized PCA problem (\ref{eq-qpca}) is then
given by
\BEQ
\label{eq-pca-reg-soln}
X = U_k \tilde \Sigma_k^{1/2}, \qquad Y = \tilde \Sigma_k ^{1/2} V_k^T.
\EEQ
For $\gamma = 0$,
the solution reduces to the familiar solution to PCA (\ref{eq-pca})
obtained by truncating the SVD to the top $k$ singular values.

The set of solutions to problem (\ref{eq-qpca}) is significantly
smaller than that of problem (\ref{eq-pca}), although solutions are still not unique: 
if $X$, $Y$ is a solution, then so is $XT$, $T^{-1}Y$ 
for any orthogonal matrix $T \in \reals^{k\times k}$. 
When $\sigma_k > \sigma_{k+1}$, all solutions to (\ref{eq-qpca}) have this form.
In particular, adding quadratic regularization results in a solution set that is bounded.

The quadratically regularized PCA problem (\ref{eq-qpca}) 
(including the PCA problem as a special case)
is the only problem we will encounter for which an analytical solution exists.
The analytical tractability of PCA explains its popularity as a technique for
data analysis in the era before computers were machines.
For example, in his 1933 paper on PCA \cite{hotelling1933},
Hotelling computes the solution to his problem using power iteration to find
the eigenvalue decomposition of the matrix $A^T A = V \Sigma^2 V^T$,
and records in the appendix to his paper the itermediate results at 
each of the (three) iterations required for the method to converge.

\paragraph{Alternating minimization.}
Here we mention a second method for solving (\ref{eq-qpca}),
which extends more readily to the extensions of PCA that we discuss below.
The \emph{alternating minimization} algorithm simply alternates between minimizing the
objective over the variable $X$, holding $Y$ fixed, and then minimizing over $Y$, holding $X$ fixed.
With an initial guess for the factors $Y^0$, we repeat the iteration
\begin{eqnarray*}
X^{l} &=& \argmin_X \left(\sum_{i=1}^m \sum_{j=1}^n(A_{ij} - x_i y^{l-1}_j)^2 +
\gamma\sum_{i=1}^m\|x_i\|_2^2\right) \\
Y^{l} &=& \argmin_Y \left(\sum_{i=1}^m \sum_{j=1}^n(A_{ij} - x^l_i y_j)_{ij})^2 +
\gamma\sum_{j=1}^n\|y_j\|_2^2\right)
\end{eqnarray*}
for $l=1,\ldots$ until a stopping condition is satisfied.
(If $X$ and $Y$ are full rank, or $\gamma>0$,
the minimizers above are unique; when they are not,
we can take any minimizer.)
The objective function is nonincreasing at each iteration,
and therefore bounded.
This implies, for $\gamma > 0$, 
that the iterates $X^l$ and $Y^l$ are bounded.

This algorithm does not always work.
In particular, it has stationary points that are not
solutions of problem (\ref{eq-qpca}).
In particular, if the rows of $Y^l$ lie in a subspace spanned by
a subset of the (right) singular vectors of $A$,
then the columns of $X^{l+1}$ will lie in a subspace spanned by the 
corresponding left singular vectors of $A$, and vice versa.
Thus, if the algorithm is initialized with $Y^0$ orthogonal to any 
of the top $k$ (right) singular vectors, 
then the algorithm (implemented in exact arithmetic) 
will not converge to the global solution to the problem.

But all \emph{stable} stationary points of the iteration are solutions
(see Appendix \ref{a-qpca}).
So as a practical matter, the alternating minimization method always works, \ie,
the objective converges to the optimal value.

\paragraph{Parallelizing alternating minimization.}
Alternating minimization parallelizes easily over examples and features. 
The problem of minimizing over $X$ splits into $m$ independent minimization
problems. We can solve the simple quadratic problems
\BEQ
\label{eq-x-update-qpca}
\begin{array}{ll}
\mbox{minimize} & \sum_{j=1}^n (A_{ij} - x_i y_j)^2 + \gamma\|x_i\|_2^2
\end{array}
\EEQ
with variable $x_i$, in parallel, for $i = 1, \ldots, m$.
Similarly, the problem of minimizing over $Y$
splits into $n$ independent quadratic problems,
\BEQ
\label{eq-y-update-qpca}
\begin{array}{ll}
\mbox{minimize} & \sum_{i=1}^m (A_{ij} - x_i y_j)^2 + \gamma\|y_j\|_2^2
\end{array}
\EEQ
with variable $y_j$, which can be solved in parallel for $j = 1, \ldots, n$.

\paragraph{Caching factorizations.}\label{s-caching}
We can speed up the solution of the quadratic problems using  
a simple factorization caching technique.

For ease of exposition, we assume here that $X$ and $Y$ have full rank $k$.
The updates (\ref{eq-x-update-qpca}) and (\ref{eq-y-update-qpca}) can be expressed as
\[
X = A Y^{T}(Y Y^{T} + \gamma I)^{-1}, \qquad Y = (X^{T} X + \gamma I)^{-1} X^{T} A.
\]
We show below how to efficiently compute $X = A Y^{T}(Y Y^{T} + \gamma I)^{-1}$;
the $Y$ update admits a similar speedup using the same ideas.
We assume here that $k$ is modest, say, not more than a few hundred or a few thousand.
(Typical values used in applications are often far smaller, on the order of tens.)
The dimensions $m$ and $n$, however, can be very large.

First compute the Gram matrix $G = YY^T$ using an outer product expansion
\[
G = \sum_{j=1}^n y_j y_j^T.
\] 
This sum can be computed on-line by streaming over the index $j$,
or in parallel, split over the index $j$.
This property allows us to scale up to extremely large problems 
even if we cannot store the entire matrix $Y$ in memory.
The computation of the Gram matrix requires $2k^2n$ floating point operations (flops),
but is trivially parallelizable: with $r$ workers, 
we can expect a speedup on the order of $r$.
We next add the diagonal matrix $\gamma I$ to $G$ in $k$ flops, 
and form the Cholesky factorization of $G + \gamma I$
in $k^3/3$ flops and cache the factorization. 

In parallel over the rows of $A$, 
we compute $D = A Y^T$ ($2kn$ flops per row),
and use the factorization of $G + \gamma I$ to compute $D (G + \gamma I)^{-1}$ 
with two triangular solves ($2k^2$ flops per row).
These computations are also trivially parallelizable: with $r$ workers, 
we can expect a speedup on the order of $r$.

Hence the total time required for each update with 
$r$ workers scales as $\mathcal O (\frac{k^2 (m+n) + kmn}{r})$.
For $k$ small compared to $m$ and $n$, the time is dominated by the computation of $A Y^T$.

\subsection{Missing data and matrix completion}\label{s-mc}
Suppose we observe only entries $A_{ij}$ for $(i,j) \in \Omega
\subset \{1,\ldots,m\}\times \{1,\ldots, n\}$ from the matrix $A$,
so the other entries are unknown.
Then to find a low rank matrix that fits the data well, we solve
the problem
\BEQ
\label{eq-mc}
\begin{array}{ll}
\mbox{minimize} & \sum_{(i,j) \in \Omega} 
(A_{ij} - x_i y_j)^2 + \gamma\|X\|_F^2 + \gamma\|Y\|_F^2,
\end{array}
\EEQ
with variables $X$ and $Y$, with $\gamma > 0$.
A solution of this problem gives an estimate $\hat A_{ij} = x_i y_j$ for the value
of those entries $(i,j) \not \in \Omega$ that were not observed.
In some applications, this data imputation (\ie, guessing entries of a matrix
that are not known) is the main point.

There are two very different regimes in which solving the problem~(\ref{eq-mc})
may be useful.

\paragraph{Imputing missing entries to borrow strength.}
Consider a matrix $A$ in which very few entries are missing.
The typical approach in data analysis 
is to simply remove any rows with missing entries from the matrix 
and exclude them from subsequent analysis.
If instead we solve the problem above \emph{without} removing these affected rows,
we ``borrow strength'' from the entries that are \emph{not} missing
to improve our global understanding of the data matrix $A$.
In this regime we are imputing the (few) missing entries of $A$,
using the examples that ordinarily we would discard.

\paragraph{Low rank matrix completion.}
Now consider a matrix $A$ in which most entries are missing,
\ie, we only observe relatively few of the $mn$ elements of $A$,
so that by discarding every example with a missing feature
or every feature with a missing example,
we would discard the \emph{entire} matrix. 
Then the solution to (\ref{eq-mc}) becomes even more interesting:
we are guessing all the entries of a (presumed low rank) matrix, given just a few of them.
It is a surprising fact that this is possible: 
typical results from the matrix completion literature
show that one can recover an unknown $m\times n$ matrix $A$ of low rank $r$ from
just about $nr \log_2 n$ noisy samples $\Omega$ with an error 
that is proportional to the noise level \cite{recht2008, candes2010, recht2010, plan2009},
so long as the matrix $A$ satisfies a certain incoherence condition
and the samples $\Omega$ are chosen uniformly at random.
These works use an estimator that minimizes a nuclear norm penalty 
along with a data fitting term to
encourage low rank structure in the solution.

The argument in \S\ref{s-optimality-certificate} shows that 
problem~(\ref{eq-mc}) is equivalent to the rank-constrained nuclear-norm regularized 
convex problem
\[
\begin{array}{ll}
\mbox{minimize} & \sum_{(i,j) \in \Omega} 
(A_{ij} - Z_{ij})^2 + 2\gamma\|Z\|_*\\
\mbox{subject to} & \rank(Z) \leq k,
\end{array}
\]
where the \emph{nuclear norm} $\|Z\|_*$ (also known as the trace norm)
is defined to be the sum of the singular values of $Z$.
Thus, the solutions to problem~(\ref{eq-mc})
correspond exactly to the solutions of these proposed estimators
so long as the rank $k$ of the model is chosen to be larger than the true rank $r$ of the matrix $A$.
Nuclear norm regularization is often used to encourage solutions of rank less than $k$,
and has applications ranging from graph embedding to
linear system identification \cite{fazel2004, liu2009, mohan2010, smith2012, osnaga2014}.

Low rank matrix completion problems arise in applications like predicting
customer ratings or customer (potential) purchases.
Here the matrix consists of the ratings or numbers of purchases that $m$ customers
give (or make) for each of $n$ products.   The vast majority of the entries in this
matrix are missing, since a customer will rate (or purchase) only a small fraction 
of the total number of products available.
In this application, imputing a missing entry of the matrix as $x_iy_j$, 
for $(i,j)\not\in \Omega$, is guessing
what rating a customer would give a product, if she were to rate it.
This can used as the basis for a recommendation system, or a marketing plan.

\paragraph{Alternating minimization.}
When $\Omega \neq \{1,\ldots, m \} \times \{1, \ldots, n\}$, 
the problem (\ref{eq-mc}) has no known analytical solution, 
but it is still easy to fit a model using alternating minimization. 
Algorithms based on alternating minimization have been shown to converge quickly 
(even geometrically \cite{jain2013}) 
to a global solution satisfying a recovery guarantee when the initial values 
of $X$ and $Y$ are chosen carefully 
\cite{keshavan2009, keshavan2010, keshavan2010regularization, jain2013, hardt2013, gunasekar2013}.

On the other hand, all of these analytical results rely on using a \emph{fresh} batch of samples $\Omega$
for each iteration of alternating minimization; none uses the quadratic regularizer
above that corresponds to the nuclear norm penalized estimator;
and interestingly, Hardt \cite{hardt2013} notes that none achieves the same sample complexity guarantees 
found in the convex matrix completion literature which, unlike the alternating minimization guarantees,
match the information theoretic lower bound \cite{candes2010} up to logarithmic factors.
For these reasons, it is plausible to expect that in practice using alternating minimization to solve problem~(\ref{eq-mc})
might yield a \emph{better} solution than the ``alternating minimization'' algorithms presented 
in the literature on matrix completion when suitably initialized 
(for example, using the method proposed below in \S\ref{s-initialization}). 
However, in general the method should be considered a heuristic.

\subsection{Interpretations and applications}\label{s-qpca-interpretation}

The recovered matrices $X$ and $Y$ in the quadratically regularized PCA problems
(\ref{eq-qpca}) and (\ref{eq-mc}) admit a number of interesting interpretations.
We introduce some of these interpretations now; the terminology we use here
will recur throughout the paper.
Of course these interpretations are related to each other, and not distinct.

\paragraph{Feature compression.}
Quadratically regularized PCA (\ref{eq-qpca}) can be 
interpreted as a method for compressing the
$n$ features in the original data set to $k<n$ new features.
The row vector $x_i$ is associated  with example $i$; we can think of it 
as a feature vector for the example using the compressed set of $k<n$ features.
The column vector $y_j$ is associated with the original feature $j$;
it can be interpreted as the mapping from the original feature $j$ into
the $k$ new features.

\paragraph{Low-dimensional geometric embedding.}
We can think of each $y_j$ as associating feature $j$ with a point in a 
low ($k$-) dimensional space.
Similarly, each $x_i$ associates example $i$ with a point in the
low dimensional space.
We can use these low dimensional vectors to judge which features (or examples) 
are similar. For example, we can run a clustering algorithm on the low dimensional
vectors $y_j$ (or $x_i$) to find groups of similar features (or examples).

\paragraph{Archetypes.}
We can think of each row of $Y$ as an \emph{archetype}
which captures the behavior of one of $k$ idealized and 
maximally informative examples. These archetypes might also be 
called profiles, factors, or atoms. 
Every example $i=1,\ldots,m$ is then represented (approximately) as a 
linear combination of these archetypes, with the row vector $x_i$ giving
the coefficients.
The coefficient $x_{il}$ gives the resemblance
or \emph{loading} of example $i$ to the $l$th archetype.

\paragraph{Archetypical representations.}
We call $x_i$ the \emph{representation} of example $i$ in terms of the archetypes.
The rows of $X$ give an embedding of the examples into $\reals^k$,
where each coordinate axis corresponds to a different archetype.
If the archetypes are simple to understand or interpret, then the representation
of an example can provide better intuition about that example.

The examples can be clustered according to their representations in order
to determine a group of similar examples.
Indeed, one might choose to apply any machine learning algorithm to the 
representations $x_i$ rather than to the initial data matrix:
in contrast to the initial data, which may consist of high dimensional vectors with
noisy or missing entries, the representations $x_i$ will be low dimensional,
less noisy, and complete.

\paragraph{Feature representations.}
The columns of $Y$ embed the features into $\reals^k$. 
Here, we think of the columns of $X$ as archetypical features, and represent
each feature $j$ as a linear combination of the archetypical features.
Just as with the examples, we might choose to apply any machine learning algorithm
to the feature representations.
For example, we might find clusters of similar features that represent redundant measurements.

\paragraph{Latent variables.}
Each row of $X$ represents an example by a vector in $\reals^k$. 
The matrix $Y$ maps these representations back into $\reals^m$.
We might think of $X$ as discovering the \emph{latent variables} that 
best explain the observed data.
If the approximation error $\sum_{(i,j) \in \Omega} (A_{ij} - x_i y_j)^2$ is small,
then we view these latent variables as providing a good explanation 
or summary of the full data set.

\paragraph{Probabilistic intepretation.} \label{s-qpca-interp-prob}
We can give a probabilistic interpretation of $X$ and $Y$,
building on the probabilistic model of PCA developed by Tipping and Bishop \cite{tipping1999}.
We suppose that the matrices $\bar X$ and $\bar Y$ have entries which are generated by taking 
independent samples from a normal distribution with mean 0 and variance
$\gamma^{-1}$ for $\gamma > 0$.
The entries in the matrix $\bar X \bar Y$ 
are observed with noise $\eta_{ij} \in \reals$,
\[
A_{ij} = (\bar X \bar Y)_{ij} + \eta_{ij},
\]
where the noise $\eta$ in the $(i,j)$th entry is sampled independently from a standard normal distribution.
We observe each entry $(i,j)\in\Omega$.
Then to find the maximum a posteriori (MAP) estimator $(X,Y)$ of $(\bar X, \bar Y)$, we solve
\[
\begin{array}{ll}
\mbox{maximize} & 
\exp\left(-\frac{\gamma}{2}\|\bar X\|_F^2\right) 
\exp\left(-\frac{\gamma}{2}\|\bar Y\|_F^2\right)
\prod_{(i,j)\in\Omega} \exp\left(- (A_{ij} - x_i y_j)^2\right),
\end{array}
\]
which is equivalent, by taking logs, to (\ref{eq-qpca}).

This interpretation explains the recommendation we gave above for imputing 
missing observations $(i,j)\not\in\Omega$.
We simply use the MAP estimator $x_i y_j$ to estimate the missing entry $(\bar X \bar Y)_{ij}$.
Similarly, we can interpret $(XY)_{ij}$ for $(i,j)\in\Omega$ as a denoised version of the observation $A_{ij}$.



\paragraph{Auto-encoder.}
The matrix $X$ encodes the data; the matrix $Y$ decodes it back into the full space.
We can view PCA as providing the best linear auto-encoder for the data;
among all (bi-linear) low rank encodings ($X$) and decodings ($Y$) of the data,
PCA minimizes the squared reconstruction error.

\paragraph{Compression.}
We impose an \emph{information bottleneck} \cite{tishby2000}
on the data by using a low rank auto-encoder to fit the data.
PCA finds $X$ and $Y$ to maximize the information transmitted through this $k$-dimensional
information bottleneck.
We can interpret the solution as a compressed representation of the data,
and use it to efficiently store or transmit the information present in the original data.

\subsection{Offsets and scaling}\label{s-qpca-scaling}

For good practical performance of a generalized low rank model, it is critical
to ensure that model assumptions match the data. 
We saw above in \S\ref{s-qpca-interp-prob} that quadratically regularized PCA
corresponds to a model in which features are observed with $\mathcal N (0,1)$ errors.
If instead each column $j$ of $XY$ is observed with $\mathcal N (\mu_j, \sigma_j^2)$ errors,
our model is no longer unbiased,
and may fit very poorly, particularly if some of the column means $\mu_j$ are large.

For this reason it is standard practice 
to \emph{standardize} the data before appplying PCA or
quadratically regularized PCA:
the column means are subtracted from each column, and the columns
are normalized by their variances.
(This can be done approximately; there is no need to get the scaling and offset 
exactly right.) 
Formally, define $n_j = |\{i: (i,j) \in \Omega\}|$,
and let
\[
\mu_j = \frac{1}{n_j} \sum_{(i,j) \in \Omega} A_{ij}, \qquad \sigma_j^2 =
\frac{1}{n_j - 1} \sum_{(i,j) \in \Omega} (A_{ij} - \mu_j)^2
\]
estimate the mean and variance of each column of the data matrix.
PCA or quadratically regularized PCA 
is then applied to the matrix whose $(i,j)$ entry 
is $(A_{ij} - \mu_j)/\sigma_j$.

\section{Generalized regularization}\label{s-rpca}

It is easy to see how to extend PCA to allow arbitrary regularization
on the rows of $X$ and columns of $Y$. We form the \emph{regularized PCA problem}
\BEQ
\label{eq-rpca}
\begin{array}{ll}
\mbox{minimize} & \sum_{(i,j) \in \Omega} (A_{ij} - x_i y_j)^2 
+ \sum_{i=1}^m r_i(x_i) + \sum_{j=1}^n \tilde r_j(y_j),
\end{array}
\EEQ
with variables $x_i$ and $y_j$, with given regularizers
$r_i: \reals^k \to \reals \cup \{\infty\}$ and 
$\tilde r_j: \reals^k \to \reals \cup \{\infty\}$ for $i=1,\ldots, n$ and $j=1,\ldots,m$.
Regularized PCA (\ref{eq-rpca}) reduces to quadratically regularized PCA
(\ref{eq-qpca}) when $r_i = \gamma \|\cdot\|_2^2$, $\tilde r_j = \gamma
\|\cdot\|_2^2$. We do not restrict the regularizers to be convex.

The objective in problem~(\ref{eq-rpca}) can be expressed compactly in matrix notation as
\[
\|A-XY\|_F^2 + r(X) + \tilde r (Y),
\]
where $r(X) = \sum_{i=1}^n r(x_i)$ and $\tilde r(Y) = \sum_{j=1}^n \tilde r(y_j)$.
The regularization functions $r$ and $\tilde r$ are separable across
the rows of $X$, and the columns of $Y$, respectively.

Infinite values of $r_i$ and $\tilde r_j$ are used to enforce 
constraints on the values
of $X$ and $Y$. For example, the regularizer
\[
r_i(x) = \left\{\begin{array}{ll}
0 & x \geq 0 \\
\infty & \mathrm{otherwise},
\end{array}\right.
\]
the indicator function of the nonnegative orthant, 
imposes the constraint that $x_i$ be nonnegative.

Solutions to (\ref{eq-rpca}) need not be unique, depending on
the choice of regularizers.  If $X$ and $Y$ are a solution,
then so are $XT$ and $T^{-1}Y$, where $T$ is any nonsingular matrix
that satisfies $r(UT)=r(U)$ for all $U$ and
$\tilde r(T^{-1}V)=r(V)$ for all $V$.

By varying our choice of regularizers $r$ and $\tilde r$, 
we are able to represent a wide range of known models, as well as many new ones.
We will discuss a number of choices for regularizers below, but turn now to
methods for solving the regularized PCA problem (\ref{eq-rpca}).

\subsection{Solution methods}
In general, there is no analytical solution for (\ref{eq-rpca}).
The problem is not convex, even when $r$ and $\tilde r$ are convex.
However, when $r$ and $\tilde r$ are convex, the problem is bi-convex:
it is convex in $X$ when $Y$ is fixed, and convex in $Y$ when $X$ is fixed.

\paragraph{Alternating minimization.} There is no reason to believe that 
alternating minimization will always converge to the global minimum 
of the regularized PCA problem (\ref{eq-rpca}).
Indeed, we will see many cases below in which the problem is known to have many local minima.
However, alternating minimization can still be applied in this setting,
and it still parallelizes over the rows of $X$ and columns of $Y$. 
To minimize over $X$, we solve, in parallel,
\BEQ
\label{eq-x-update-rpca}
\begin{array}{ll}
\mbox{minimize} & \sum_{j: (i,j) \in \Omega} (A_{ij} - x_i y_j)^2 + r_i(x_i)
\end{array}
\EEQ
with variable $x_i$, for $i = 1, \ldots, m$.
Similarly, to minimize over $Y$, we solve, in parallel,
\BEQ
\label{eq-y-update-rpca}
\begin{array}{ll}
\mbox{minimize} & \sum_{i: (i,j) \in \Omega} (A_{ij} - x_i y_j)^2 + \tilde r_j(y_j)
\end{array}
\EEQ
with variable $y_j$, for $j = 1, \ldots, n$.

When the regularizers are convex, these problems are convex.
When the regularizers are not convex,
there are still many cases in which we can find analytical solutions to the nonconvex subproblems 
(\ref{eq-x-update-rpca}) and (\ref{eq-y-update-rpca}), as we will see below.
A number of concrete algorithms, in which these subproblems are solved explicitly,
are given in \S\ref{s-algorithms}.

\paragraph{Caching factorizations.}
Often, the $X$ and $Y$ updates (\ref{eq-x-update-rpca}) and (\ref{eq-y-update-rpca})
reduce to convex quadratic programs.
For example, this is the case for nonnegative matrix factorization,
sparse PCA, and quadratic mixtures 
(which we define and discuss below in \S\ref{s-greg-ex}).
The same factorization caching of the Gram matrix that was described above in the
case of PCA can be used here to speed up the solution of these updates.
Variations on this idea are described in detail in \S\ref{s-alg-qp}.

\subsection{Examples} \label{s-greg-ex}
Here and throughout the paper, we present a set of examples chosen
for pedagogical clarity, not for completeness. 
In all of the examples below, $\gamma>0$ is a parameter that controls
the strength of the regularization,
and we drop the subscripts from $r$ (or $\tilde r$) to lighten 
the notation.
Of course, it is possible to mix and match these regularizers,
\ie, to choose different $r_i$ for different $i$, and
choose different $\tilde r_j$ for different $j$.

\paragraph{Nonnegative matrix factorization (NNMF).}
Consider the regularized PCA problem (\ref{eq-rpca}) with 
$r = I_+$ and $\tilde r=I_+$, where $I_+$ is
the indicator function of the nonnegative reals.
(Here, and throughout the paper, we define the indicator function of a set $C$, 
to be $0$ when its argument is in $C$ and $\infty$ otherwise.) 
Then problem~(\ref{eq-rpca}) is NNMF: a solution gives the matrix best approximating $A$
that has a nonnegative factorization (\ie, a factorization into 
elementwise nonnegative matrices) \cite{lee1999}.
It is NP-hard to solve NNMF problems exactly \cite{vavasis2009}.
However, these problems have a rich analytical structure which can sometimes be exploited 
\cite{gillis2011thesis,bittorf2012, damle2014},
and a wide range of uses in practice \cite{lee1999, shahnaz2006, berry2007, virtanen2007, kim2007, fevotte2009}.
Hence a number of specialized algorithms and codes for fitting NNMF models are available 
\cite{lee2001nonneg, lin2007, kim2008nonneg, kim2008toward, SmallK, kim2014, kim2011}.

We can also replace the nonnegativity constraint with any interval constraint.
For example, $r$ and $\tilde r$ can be 0 if all entries of $X$ and $Y$,
respectively, are between $0$ and $1$, and infinite otherwise.

\paragraph{Sparse PCA.}
If very few of the coefficients of $X$ and $Y$ are nonzero, it can be easier
to interpret the archetypes and representations.
We can understand each archetype using only a small number of features,
and can understand each example as a combination of 
only a small number of archetypes.
To get a sparse version of PCA, we use a sparsifying penalty as
the regularization.
Many variants on this basic idea have been proposed,
together with a wide variety of algorithms
\cite{daspremont2004,zou2006,shen2008,mackey2008,witten2009,richtarik2012,vu2013}.

For example, we could enforce that no entry $A_{ij}$
depend on more than $s$ columns of $X$ or of $Y$ by setting $r$ to be the
indicator function of a $s$-sparse vector, \ie, 
\[
r (x) = \left\{\begin{array}{ll} 0 & \Card (x) \leq s \\
\infty & \mbox{otherwise},
\end{array}\right.
\]
and defining $\tilde r(y)$ similarly, 
where $\Card (x)$ denotes the cardinality (number of nonzero entries) in the vector $x$.
The updates (\ref{eq-x-update-rpca}) and (\ref{eq-y-update-rpca}) are not convex
using this regularizer,
but one can find approximate solutions using a pursuit algorithm
(see, \eg, \cite{chen1998,tropp2007}),
or exact solutions (for small $s$) using the branch and bound method 
\cite{lawler1966, boyd2003}.

As a simple example, consider $s=1$.   Here we insist that
each $x_i$ have at most one nonzero entry, which means that each example is
a multiple of \emph{one} of the rows of $Y$.
The $X$-update is easy to carry out, by evaluating the 
best quadratic fit of $x_i$ with each of the $k$ rows of $Y$.  This reduces 
to choosing the row of $Y$ that has the smallest angle to the $i$th row of $A$.

The $s$-sparse regularization can be relaxed to a convex, but still sparsifying, 
regularization using $r(x) = \|x\|_1$, $\tilde r(y) = \|y\|_1$ \cite{zou2006}.
In this case, the $X$-update reduces to solving a (small) $\ell_1$-regularized 
least-squares problem.

\paragraph{Orthogonal nonnegative matrix factorization.}
One well known property of PCA is that the principal components obtained 
(\ie, the columns of $X$ and rows of $Y$) can be chosen to be orthogonal,
so $X^T X$ and $Y Y^T$ are both diagonal.
We can impose the same condition on a nonnegative matrix factorization.
Due to nonnegativity of the matrix, two columns of $X$
cannot be orthogonal if they both have a nonzero in the same row.
Conversely, if $X$ has only one nonzero per row, then its columns
are mutually orthogonal.
So an orthogonal nonnegative matrix factorization is identical to
to a nonnegativity condition in addition to the $1$-sparse
condition described above.
Orthogonal nonnegative matrix factorization can be achieved by using the regularizer
\[
r (x) = \left\{\begin{array}{ll}
    0 & \card (x) = 1, \quad x \geq 0\\
    \infty & \mbox{otherwise},
\end{array}\right.
\]
and letting $\tilde r(y)$ be the indicator of the nonnegative orthant,
as in NNMF.

Geometrically, we can interpret this problem as modeling the data
$A$ as a union of rays. Each row of $Y$, interpreted as a point in $\reals^n$, defines a ray from the origin passing through that point.
Orthogonal nonnegative matrix factorization models each row of $X$ as 
a point along one of these rays.

Some authors \cite{ding2006} have also considered how to obtain a \emph{bi-orthogonal}
nonnegative matrix factorization, in which both $X$ and $Y^T$ have orthogonal columns.
By the same argument as above, we see this is equivalent to requiring
both $X$ and $Y^T$ to have only one positive entry per row, 
with the other entries equal to 0.

\paragraph{Max-norm matrix factorization.}
We take $r = \tilde r = \phi$ with
\[
\phi(x) = \left\{\begin{array}{ll}
0 & \|x\|^2_2 \leq \mu\\
\infty & \mathrm{otherwise}.
\end{array}\right. 
\]
This penalty enforces that
\[
\|X\|_{2,\infty}^2 \leq \mu, \qquad \|Y^T\|_{2,\infty}^2 \leq \mu,
\]
where the $(2,\infty)$ norm of a matrix $X$ with rows $x_i$ is defined as $\max_i \|x_i\|_2$.
This is equivalent to requiring the 
\emph{max-norm} (sometimes called the \emph{$\gamma_2$-norm}) of $Z = XY$,
which is defined as 
\[
\|Z\|_{\mathrm{max}} = \inf\{\|X\|_{2,\infty} \|Y^T\|_{2,\infty}: XY = Z\},
\]
to be bounded by $\mu$.
This penalty has been proposed by \cite{lee2010} as a heuristic 
for low rank matrix completion, which can perform better than Frobenius norm regularization when
the low rank factors are known to have bounded entries.

\paragraph{Quadratic clustering.} \label{s-kmeans}
Consider (\ref{eq-rpca}) with
$\tilde r = 0$. Let $r$ be the indicator function of a selection, \ie, 
\[
r (x) = \left\{\begin{array}{ll}
    0 & x = e_l~\mbox{for some}~l \in \{1,\ldots,k\} \\
    \infty & \mbox{otherwise},
\end{array}\right.
\]
where $e_l$ is the $l$-th standard basis vector.
Thus $x_i$ encodes the cluster (one of $k$) to which the data vector $(A_{i1},
\ldots, A_{im})$ is assigned. 

Alternating minimization on this problem reproduces 
the well-known $k$-means algorithm (also known as Lloyd's algorithm) \cite{lloyd1982}.
The $y$ update (\ref{eq-y-update-rpca}) is a least squares problem with the simple solution
\[
Y_{lj} = \frac{\sum_{i: (i,j) \in \Omega} A_{ij} X_{il}}{\sum_{i: (i,j) \in \Omega} X_{il}} ,
\]
\ie, each row of $Y$ is updated to be the mean of the rows of 
$A$ assigned to that archetype.
The $x$ update (\ref{eq-x-update-rpca}) is not a convex problem, but is easily solved.
The solution is given by assigning $x_i$ to the closest
archetype (often called a \emph{cluster centroid} in the context of $k$-means): 
$x_i = e_{l^\star}$ for $l^\star = \argmin_l \left( \sum_{j=1}^n 
(A_{ij} - Y_{lj})^2\right)$.


\paragraph{Quadratic mixtures.}
We can also implement partial assignment of data vectors to clusters.
Take $\tilde r = 0$, and let
$r$ be the indicator function of the set of probability vectors, \ie,
\[
r (x) = \left\{\begin{array}{ll}
    0 & \sum_{l = 1}^k x_{l} = 1, \quad x_{l} \geq 0 \\
    \infty & \mbox{otherwise}.
\end{array}\right.
\]

 
\paragraph{Subspace clustering.}
PCA approximates a data set by a single low dimensional subspace. 
We may also be interested in approximating a data set as a \emph{union}
of low dimensional subspaces.
This problem is known as \emph{subspace clustering} (see \cite{vidal2010} and references therein).
Subspace clustering may also be thought of as generalizing quadratic clustering 
to assign each data vector to a low dimensional subspace rather than to a single cluster centroid.

To frame subspace clustering as a regularized PCA problem (\ref{eq-rpca}),
partition the columns of $X$ into $k$ blocks. 
Then let $r$ be the indicator function of block sparsity
(\ie, $r(x) = 0$ if only one block of $x$ has nonzero entries, and 
otherwise $r(x) = \infty$).

It is easy to perform alternating minimization on this objective function.
This method is sometimes called the \emph{$k$-planes} algorithm 
\cite{vidal2010,tseng2000,agarwal2004}, 
which alternates over assigning examples to subspaces, and fitting the subspaces to the examples.
Once again, the $X$ update (\ref{eq-x-update-rpca}) is not a convex problem, but 
can be easily solved.
Each block of the columns of $X$ defines a subspace 
spanned by the corresponding rows of $Y$.
We compute the distance from example $i$ (the $i$th row of $A$) to each subspace
(by solving a least squares problem),
and assign example $i$ to the subspace that minimizes the least squares error
by setting $x_i$ to be the solution to the corresponding least squares problem.

Many other algorithms for this problem have also been proposed, 
such as the $k$-SVD \cite{tropp2004thesis,aharon2006}
and sparse subspace clustering \cite{elhamifar2009},
some with provable guarantees on the quality of the recovered solution \cite{soltanolkotabi2012}.


\paragraph{Supervised learning.}
Sometimes we want to understand the variation that a certain set of features can explain, and the variance that remains unexplainable.
To this end, one natural strategy would be to regress the labels in the dataset on the features; to subtract the predicted values
from the data; and to use PCA to understand the remaining variance. 
This procedure gives the same answer as the solution to a single regularized PCA problem.
Here we present the case in which the features we wish to use in the 
regression are present in the data as the first column of $A$.
To construct the regularizers, we make sure the first column of 
$A$ appears as a feature in the supervised learning problem by setting
\[
r_i(x) = \left\{\begin{array}{ll}
r_0(x_2, \ldots, x_{k+1}) & x_1 = A_{i1} \\
\infty & \mathrm{otherwise},
\end{array} \right . 
\]
where $r_0 = 0$ can be chosen as in any regularized PCA model.
The regularization on the first row of $Y$ is the regularization
used in the supervised regression, and the regularization on the 
other rows will be that used in regularized PCA.

Thus we see that regularized PCA can naturally combine supervised and unsupervised learning into a single problem.

\paragraph{Feature selection.}
We can use regularized PCA to perform feature selection.
Consider (\ref{eq-rpca}) with $r(x) = \|x\|_2^2$ and $\tilde r (y) = \|y\|_2$.
(Notice that we are \emph{not} using  $\|y\|_2^2$.)
The regularizer $\tilde r$ encourages the matrix $\tilde Y$ to be column-sparse, 
so many columns are all zero. 
If $\tilde y_j = 0$, it means that feature $j$ was uninformative, in the
sense that its values do not help much in predicting any feature in the matrix
$A$ (including feature $j$ itself). In this case we say that feature $j$ was not
selected.
For this approach to make sense, it is important that the columns of the matrix $A$ should
have mean zero. 
Alternatively, one can use the de-biasing regularizers $r'$ and $\tilde r'$ 
introduced in \S\ref{s-rpca-offset} along with the feature selection regularizer introduced here.

\paragraph{Dictionary learning.}
Dictionary learning (also sometimes called \emph{sparse coding}) 
has become a popular method to design concise representations
for very high dimensional data 
\cite{olshausen1997,lee2006efficient,mairal2009a,mairal2009b}. 
These representations have been shown to perform well when used as features
in subsequent (supervised) machine learning tasks \cite{raina2007}.
In dictionary learning, each row of $A$ is modeled as a linear combination of dictionary atoms,
represented by rows of $Y$.
The total size of the dictionary used is often very large ($k \gg \max(m,n)$),
but each example is represented using a very small number of atoms.
To fit the model, one solves the regularized PCA problem (\ref{eq-rpca}) with
$r(x) = \|x\|_1$, to induce sparsity in the number of atoms used
to represent any given example,
and with $\tilde{r}(y) = \|y\|_2^2$ or $\tilde{r}(y) = I_+(c - \|y\|_2)$ for some $c>0 \in \reals$,
in order to ensure the problem is well posed.
(Note that our notation transposes the usual notation 
in the literature on dictionary learning.)

\paragraph{Mix and match.}
It is possible to combine these regularizers to obtain a factorization with
any combination of the above properties.
As an example, one may require that both $X$ and $Y$ 
be simultaneously sparse and nonnegative
by choosing 
\[
r(x) = \|x\|_1 + I_+(x) = \ones^T x + I_+(x),
\]
and similarly for $\tilde r(y)$.
Similarly, \cite{kim2007} show how to obtain a nonnegative matrix factorization
in which one factor is sparse
by using $r(x) = \|x\|^2_1 + I_+(x)$ and $\tilde r(y) = \|y\|^2_2 + I_+(y)$;
they go on to use this factorization as a clustering technique.

\subsection{Offsets and scaling} \label{s-rpca-offset}
In our discussion of the quadratically regularized PCA problem (\ref{eq-qpca}),
we saw that it can often be quite important to standardize the data
before applying PCA.
Conversely, in regularized PCA problems such as nonnegative matrix factorization,
it makes no sense to standardize the data, since subtracting column means
introduces negative entries into the matrix.

A flexible approach is to allow an offset in the model: we solve
\BEQ
\label{eq-orpca}
\begin{array}{ll}
\mbox{minimize} & \sum_{(i,j) \in \Omega} (A_{ij} - x_i y_j - \mu_j)^2 
+ \sum_{i=1}^m r_i(x_i) + \sum_{j=1}^n \tilde r_j(y_j),
\end{array}
\EEQ
with variables $x_i$, $y_j$, and $\mu_j$.
Here, $\mu_j$ takes the role of the column mean, and in fact will be equal to the 
column mean in the trivial case $k=0$.

An offset may be included in the standard form
regularized PCA problem (\ref{eq-rpca}) by augmenting the problem slightly. 
Suppose we are given an instance of the problem (\ref{eq-rpca}), \ie, we are given $k$, $r$, and $\tilde r$.
We can fit an offset term $\mu_j$ by letting $k'=k+1$
and modifying the regularizers.
Extend the regularization $r: \reals^{k}\to\reals$ and $\tilde r: \reals^{k}\to\reals$ to new regularizers
$r': \reals^{k+1}\to\reals$ and $\tilde r': \reals^{k+1}\to\reals$ which enforce that 
the first column of $X$ is constant and the first row of $Y$ is not penalized. 
Using this scheme, the first row of the optimal $Y$ will be equal to the optimal $\mu$ in (\ref{eq-orpca}).

Explicitly, let
\[
r'(x)=\left\{\begin{array}{ll}
r(x_2, \ldots, x_{k+1}) & x_1 = 1 \\
\infty & \mathrm{otherwise},
\end{array} \right . 
\]
and $\tilde r'(y) = \tilde r(y_2, \ldots, y_{k+1})$.
(Here, we identify $r(x) = r(x_1, \ldots, x_k)$ to explicitly show the dependence 
on each coordinate of the vector $x$,
and similarly for $\tilde r$.)

It is also possible to introduce row offsets in the same way.


\section{Generalized loss functions}\label{s-gpca}

We may also generalize the \emph{loss} function in PCA to form a \emph{generalized low rank model},
\BEQ
\label{eq-gpca}
\begin{array}{ll}
\mbox{minimize} & \sum_{(i,j) \in \Omega} L_{ij}(x_i y_j, A_{ij}) 
+ \sum_{i=1}^m r_i(x_i) + \sum_{j=1}^n \tilde r_j(y_j),
\end{array}
\EEQ
where $L_{ij}:\reals \times \reals \to \reals_+$ are given loss functions 
for $i=1,\ldots,m$ and $j=1,\ldots,n$.
Problem (\ref{eq-gpca}) reduces to PCA with generalized regularization
when $L_{ij}(u,a) = (a-u)^2$.
However, the loss function $L_{ij}$ can now depend on the data $A_{ij}$
in a more complex way. 

\subsection{Solution methods}
As before, problem~(\ref{eq-gpca}) is not convex, 
even when $L_{ij}$, $r_i$ and $\tilde r_j$ are convex; 
but if all these functions are convex, then
the problem is bi-convex.

\paragraph{Alternating minimization.}
Alternating minimization can still be used to find a local minimum,
and it is still often possible to use factorization caching to speed up the solution
of the subproblems that arise in alternating minimization.
We defer a discussion of how to solve these subproblems explicitly
to \S\ref{s-algorithms}.

\paragraph{Stochastic proximal gradient method.}
For use with extremely large scale problems, we discuss fast variants of the basic
alternating minimization algorithm in \S\ref{s-algorithms}.
For example, we present an alternating directions 
stochastic proximal gradient method.
This algorithm accesses the functions $L_{ij}$, $r_i$, and $\tilde r_j$ only through
a subgradient or proximal interface, allowing it to generalize trivially
to nearly any loss function and regularizer.
We defer a more detailed discussion of this method to \S\ref{s-algorithms}.

\subsection{Examples}

\paragraph{Weighted PCA.}
A simple modification of the PCA objective is to weight the importance of fitting each
element in the matrix $A$. In the generalized low rank model, 
we let $L_{ij}(u-a) = w_{ij} (a-u)^2$, where
$w_{ij}$ is a weight, and take $r = \tilde r = 0$. 
Unlike PCA, the weighted PCA problem has no known analytical solution \cite{srebro2003}.
In fact, it is NP-hard to find an exact solution to weighted PCA \cite{gillis2011},
although it is not known whether it is always possible to find
approximate solutions of moderate accuracy efficiently.

\paragraph{Robust PCA.}
Despite its widespread use, PCA is very sensitive to outliers. 
Many authors have proposed a robust version of PCA
obtained by replacing least-squares loss with $\ell_1$ loss,
which is less sensitive to large outliers
\cite{candes2011,wright2009,xu2012}.
They propose to solve the problem
\BEQ
\label{eq-robust-pca}
\begin{array}{ll}
\mbox{minimize} & \|S\|_1 + \|Z\|_*\\
\mbox{subject to} & S+Z = A.
\end{array}
\EEQ
The authors interpret $Z$ as a robust version of the 
principal components of the data matrix $A$,
and $S$ as the sparse, possibly large noise corrupting the observations.

We can frame robust PCA as a GLRM in the following way.
If $L_{ij}(u,a) = |a - u|$, and $r(x) = \frac{\gamma}{2}\|x\|_2^2$,
$\tilde r(y) = \frac{\gamma}{2}\|y\|_2^2$, then (\ref{eq-gpca}) becomes
\[
\begin{array}{ll}
\mbox{minimize} & \|A - XY\|_1 
+ \frac{\gamma}{2}\|X\|_F^2 + \frac{\gamma}{2}\|Y\|_F^2.
\end{array}
\]
Using the arguments in \S\ref{s-optimality-certificate}, 
we can rewrite the problem by introducing a new variable $Z = XY$ as
\[
\begin{array}{ll}
    \mbox{minimize} & \|A - Z\|_1 + \gamma \|Z\|_* \\
    \mbox{subject to} & \rank(Z) \leq k.
\end{array}
\]
This results in a rank-constrained version of the estimator
proposed in the literature on robust PCA \cite{wright2009,candes2011,xu2012}:
\[ 
\begin{array}{ll}
    \mbox{minimize} & \|S\|_1 + \gamma \|Z\|_* \\
	\mbox{subject to} & S + Z = A \\
    & \rank (Z) \leq k,
\end{array}
\]
where we have introduced the new variable $S=A-Z$.

\paragraph{Huber PCA.}
The Huber function is defined as
\[
\huber(x) = \left\{ \begin{array}{ll}
\half x^2 & |x| \leq 1 \\
|x| - \half & |x| > 1.
\end{array} \right .
\]
Using Huber loss,
\[
L(u,a) = \huber(u-a),
\]
in place of $\ell_1$ loss also yields an
estimator robust to occasionaly large outliers \cite{huber1981}. 
The Huber function is less sensitive to small errors 
$|u-a|$ than the $\ell_1$
norm, but becomes linear in the error for large errors.
This choice of loss function results in a generalized low rank model formulation that is robust
both to large outliers and to small Gaussian perturbations in the data.

Previously, the problem of Gaussian noise in robust PCA has been treated by
decomposing the matrix $A = L+S+N$ into a low rank matrix $L$, 
a sparse matrix $S$, and a matrix with small Gaussian entries $N$
by minimizing the loss 
\[
\|L\|_* + \|S\|_1 + \half\|N\|_F^2
\]
over all decompositions $A = L+S+N$ of $A$ \cite{xu2012}.

In fact, this formulation is equivalent to Huber PCA
with quadratic regularization on the factors $X$ and $Y$.
The argument showing this is very similar to the one we made above for robust PCA.
The only added ingredient is the observation that 
\[
\huber(x) = \inf\{|s| + \half n^2: x = n+s\}.
\]
In other words, the Huber function is the infimal convolution
of the negative log likelihood of a gaussian random variable and 
a laplacian random variable:
it represents the most likely assignment of (additive) blame for 
the error $x$ to a gaussian error $n$ and a laplacian error $s$.

\paragraph{Robust regularized PCA.}
We can design robust versions of all the regularized PCA problems above 
by the same transformation we used to design robust PCA.
Simply replace the quadratic loss function with an $\ell_1$ or Huber loss function.
For example, $k$-mediods \cite{kaufman2009,park2009} is obtained by 
using $\ell_1$ loss in place of quadratic loss in the quadratic clustering problem.
Similarly, robust subspace clustering \cite{soltanolkotabi2013}
can be obtained by using an $\ell_1$ or Huber penalty in the subspace clustering problem.

\paragraph{Quantile PCA.}
For some applications, it can be much worse to \emph{overestimate}
the entries of $A$ than to \emph{underestimate} them, or vice versa.
One can capture this asymmetry by using the loss function
\[
L (u,a) = \alpha(a - u)_+ + (1-\alpha)(u - a)_+
\] 
and choosing $\alpha \in (0,1)$ appropriately.
This loss function is sometimes called a \emph{scalene} loss,
and can be interpreted as performing \emph{quantile regression}, 
\eg, fitting the 20th percentile \cite{koenker1978,koenker2005}.

\paragraph{Fractional PCA.}
For other applications, we may be interested in finding an approximation
of the matrix $A$ whose entries are close to the original matrix on a relative,
rather than an absolute, scale.
Here, we assume the entries $A_{ij}$ are all positive.
The loss function
\[
L (u,a) = \max\left(\frac{a-u}{u},\frac{u-a}{a}\right)
\]
can capture this objective.
A model $(X,Y)$ with objective value less than $0.10 mn$ gives
a low rank matrix $XY$ that is on average within 10\% of the original matrix.

\paragraph{Logarithmic PCA.}
Logarithmic loss functions may also useful for finding an approximation of $A$
that is close on a relative, rather than absolute, scale.
Once again, we assume all entries of $A$ are positive.
Define the logarithmic loss
\[
L(u,a) = \log^2 (u/a).
\]
This loss is \emph{not} convex, but has the nice property that it fits the 
\emph{geometric mean} of the data:
\[
\argmin_u \sum_i L(u,a_i) = (\prod_i a_i)^{1/n}.
\]
To see this, note that we are solving a least squares problem in log space.
At the solution, $\log(u)$ will be the mean of $\log(a_i)$, \ie,
\[
\log(u) = 1/n \sum_i \log(a_i) = \log \left((\prod_i a_i)^{1/n} \right).
\]

\paragraph{Exponential family PCA.}\label{s-exp-pca}

It is easy to formulate a version of PCA corresponding to any loss
in the exponential family.
Here we give some interesting loss functions generated by exponential families
when all the entries $A_{ij}$ are positive.
(See \cite{collins2001} for a general treatment of exponential family PCA.)
One popular loss function in the exponential family is the KL-divergence loss,
\[
L (u,a) = a \log\left(\frac{a}{u}\right) - a + u,
\]
which corresponds to a Poisson generative model \cite{collins2001}.

Another interesting loss function is the Itakura-Saito (IS) loss, 
\[
L (u,a) = \log\left(\frac{a}{u}\right) - 1 + \frac{a}{u},
\]
which has the property that it is scale invariant, so scaling $a$ and $u$ by the 
same factor produces the same loss \cite{sun2014}.
The IS loss corresponds to \emph{Tweedie} distributions 
(\ie, distributions for which the variance is some power of the mean) \cite{tweedie1984}.
This makes it interesting in applications, such as audio processing,
where fractional errors in recovery are perceived.

The $\beta$-divergence,
\[
L (u,a) = \frac{a^\beta}{\beta(\beta-1)} + \frac{u^\beta}{\beta} - \frac{au^{\beta-1}}{\beta-1},
\]
generalizes both of these losses. With $\beta=2$, we recover quadratic loss;
in the limit as $\beta \to 1$, we recover the KL-divergence loss;
and in the limit as $\beta \to 0$, we recover the IS loss \cite{sun2014}.

\subsection{Offsets and scaling}\label{s-gpca-scaling}
In \S\ref{s-qpca-scaling}, we saw how to use standardization to rescale the data 
in order to compensate for unequal scaling in different features.
In general, standardization destroys sparsity in the data by subtracting the (column) means
(which are in general non-zero) from each element of the data matrix $A$.
It is possible to instead rescale the \emph{loss functions} in order to compensate for unequal scaling.
Scaling the loss functions instead has the advantage that no arithmetic is performed directly
on the data $A$, so sparsity in $A$ is preserved.

A savvy user may be able to select loss functions $L_{ij}$ that are scaled 
to reflect the importance of fitting different columns.
However, it is useful to have a default automatic scaling for times 
when no savvy user can be found.
The scaling proposed here generalizes the idea of standardization to a setting with heterogeneous loss functions.

Given initial loss functions $L_{ij}$, which we assume are nonnegative,
for each feature $j$ let 
\[
\mu_j = \argmin_\mu \sum_{i: (i,j) \in \Omega} L_{ij}(\mu, A_{ij}), \qquad
\sigma^2_j = \frac{1}{n_j - 1} \sum_{i: (i,j) \in \Omega} L_{ij}(\mu_j, A_{ij}).
\]
It is easy to see that $\mu_j$ generalizes the mean of column $j$,
while $\sigma^2_j$ generalizes the column variance.
For example, when $L_{ij}(u, a) = (u-a)^2$ for every $i=1,\ldots,m$, $j=1,\ldots,n$, 
$\mu_j$ is the mean and $\sigma^2_j$ is the sample variance of the $j$th column of $A$.
When $L_{ij}(u, a) = |u-a|$ for every $i=1,\ldots,m$, $j=1,\ldots,n$, 
$\mu_j$ is the median of the $j$th column of $A$,
and $\sigma^2_j$ is the sum of the absolute values of the deviations of 
the entries of the $j$th column from the median value.

To fit a standardized GLRM, we rescale the loss functions by $\sigma^2_j$ and solve
\BEQ
\label{eq-sgpca}
\begin{array}{ll}
\mbox{minimize} & \sum_{(i,j) \in \Omega} 
L_{ij}(A_{ij}, x_i y_j + \mu_j) /\sigma^2_j
+ \sum_{i=1}^m r_i(x_i) + \sum_{j=1}^n \tilde r_j(y_j).
\end{array}
\EEQ
Note that this problem can be recast in the standard form for a generalized low rank model 
(\ref{eq-gpca}).
For the offset, we may use the same trick described in \S\ref{s-rpca-offset} to encode the offset in the regularization;
and for the scaling, we simply replace the original loss function $L_{ij}$ 
by $L_{ij}/\sigma^2_j$.

\section{Loss functions for abstract data types}\label{s-apca}
We began our study of generalized low rank modeling by considering 
the best way to approximate a matrix by another matrix of lower rank.
In this section, we apply the same procedure to approximate a data table
that may not consist of real numbers, 
by choosing a loss function that respects the data type.

We now consider $A$ to be a \emph{table}
consisting of $m$ examples (\ie, rows, samples) 
and $n$ features (\ie, columns, attributes), with each entry $A_{ij}$
drawn from a feature set $\mathcal F_j$.
The feature set $\mathcal F_j$ may be discrete or continuous.
So far, we have only considered numerical data
($\mathcal F_j = \reals$ for $j = 1, \ldots, n$), 
but now $\mathcal F_j$ can represent more abstract data types. 
For example, entries of $A$ can take on 
Boolean values ($\mathcal F_j = \{T, F\}$), 
integral values ($\mathcal F_j = 1,2,3,\ldots$),
ordinal values ($\mathcal F_j = \{\mbox{very much}, \mbox{a little}, \mbox{not at all}\}$),
or consist of a tuple of these types ($\mathcal F_j = \{(a,b): a \in \reals\}$).

We are given a loss function $L_{ij}: \reals \times \mathcal F_j \to \reals$.
The loss $L_{ij}(u, a)$ describes the approximation error incurred when 
we represent a feature value $a \in \mathcal F_j$ by the number $u \in \reals$. 
We give a number of examples of these loss functions below.

We now formulate a generalized low rank model on the database $A$ as
\BEA
\label{eq-apca}
\begin{array}{ll}
\mbox{minimize} & \sum_{(i, j) \in \Omega} L_{ij}(x_i y_j, A_{ij}) 
+ \sum_{i=1}^m r_i(x_i) + \sum_{j=1}^n \tilde r_j(y_j),
\end{array}
\EEA
with variables $X \in \reals^{n \times k}$ and $Y \in \reals^{k \times m}$, and
with loss $L_{ij}$ as above and regularizers $r_i(x_i): \reals^{1 \times k} \to
\reals$ and $\tilde r_j(y_j): \reals^{k \times 1} \to \reals$ (as before). 
When the domain of each loss function is $\reals \times \reals$, 
we recover the generalized low rank model on a matrix (\ref{eq-gpca}).

\subsection{Solution methods}
As before, this problem is not convex, but it is bi-convex if $r_i$, and $\tilde r_j$
are convex, and $L_{ij}$ is convex in its first argument. 
The problem is also separable across samples $i = 1, \ldots, m$ and
features $j = 1, \ldots, m$.
These properties makes it easy to perform alternating minimization on this objective.
Once again, we defer a discussion of how to solve these subproblems explicitly
to \S\ref{s-algorithms}.

\subsection{Examples}

\paragraph{Boolean PCA.}
Suppose $A_{ij} \in \{-1,1\}^{m \times n}$, 
and we wish to approximate this Boolean matrix.
For example, we might suppose that the entries of $A$ are generated
as noisy, 1-bit observations from an underlying low rank matrix $XY$.
Surprisingly, it is possible to accurately estimate the underlying matrix
with only a few observations $|\Omega|$ from the matrix by solving problem (\ref{eq-apca})
(under a few mild technical conditions) with an appropriate loss function \cite{davenport2012}.

We may take the loss to be 
\[
L (u,a) = (1 - au)_+,
\]
which is the hinge loss (see Figure~\ref{f-hinge-loss}), 
and solve the problem (\ref{eq-apca}) with or without regularization.
When the regularization is sum of squares ($r(x) = \lambda \|x\|_2^2$, $\tilde
r(y) = \lambda \|y\|_2^2$),
fixing $X$ and minimizing over $y_j$ is equivalent to training a 
support vector machine (SVM) on 
a data set consisting of $m$ examples with features $x_i$ and labels $A_{ij}$.
Hence alternating minimization for the problem (\ref{eq-gpca}) 
reduces to repeatedly training an SVM.
This model has been previously considered under the name 
Maximum Margin Matrix Factorization (MMMF) \cite{srebro2004,rennie2005}.

\paragraph{Logistic PCA.}
Again supposing $A_{ij} \in \{-1,1\}^{m \times n}$,
we can also use a logistic loss to measure the approximation quality.
Let 
\[
L (u,a) = \log(1+\exp(-au))
\]
(see Figure~\ref{f-logistic-loss}).
With this loss, 
fixing $X$ and minimizing over $y_j$ is equivalent to using logistic regression
to predict the labels $A_{ij}$.
This model has been previously considered under the name logistic PCA \cite{schein2003}.

\begin{figure}
\centering
\begin{minipage}{0.5\textwidth}
\centering
\begin{tikzpicture}  
\tikzset{every pin/.style={font=\small}}
  \begin{axis}[
    xlabel = $u$,
    ylabel = {$(1 - au)_+$},
    samples = 100,
    xmin = -3,
    xmax = 3,
    width = \textwidth
  ]
  \addplot[mark = none, blue, thick] 
    {max(1 - x, 0)} 
    node[pos=.25,pin={[pin distance=0cm]0:{$a = 1$}}] 
    {};
  \addplot[mark = none, red, thick] 
    {max(1 + x, 0)} 
    node[pos=.75,pin={[pin distance=0cm]180:{$a = -1$}}]  
    {};
  \end{axis}
\end{tikzpicture}
\caption{\label{f-hinge-loss}Hinge loss.}
\end{minipage}\hfill
\begin{minipage}{0.5\textwidth}
\centering
\begin{tikzpicture}
\tikzset{every pin/.style={font=\small}}
  \begin{axis}[
    xlabel = $u$,
    ylabel = {$\log(1+\exp(au))$},
    samples = 100,
    xmin = -3,
    xmax = 3,
    width = \textwidth
  ]
  \addplot[mark = none, blue, thick] 
    {ln(1+exp(-x))} 
    node[pos=.25,pin={[pin distance=0cm]0:{$a = 1$}}] 
    {};
  \addplot[mark = none, red, thick] 
    {ln(1+exp(x))} 
    node[pos=.75,pin={[pin distance=0cm]180:{$a = -1$}}]  
    {};
  \end{axis}
\end{tikzpicture}
\caption{\label{f-logistic-loss}Logistic loss.}
\end{minipage}
\end{figure}

\paragraph{Poisson PCA.}
Now suppose the data $A_{ij}$ are nonnegative integers. 
We can use any loss function that might be used in a regression framework
to predict integral data to construct a generalized low rank model for Poisson PCA.
For example, we can take 
\[
L (u,a) = \exp(u) - au + a \log a - a. 
\]
This is the exponential family loss corresponding to Poisson data.
(It differs from the KL-divergence loss from \S\ref{s-exp-pca} only in that $u$ has been replaced by $\exp(u)$,
which allows $u$ to take negative values.)

\paragraph{Ordinal PCA.}
Suppose the data $A_{ij}$ records the levels of some ordinal variable, 
encoded as $\{1, 2, \ldots, d\}$.
We wish to penalize the entries of the low rank matrix $XY$ which deviate by
many levels from the encoded ordinal value.
A convex version of this penalty is given by the ordinal hinge loss,
\BEQ
\label{eq-ord-pca}
L(u,a) = \sum_{a'=1}^{a-1}(1-u+a')_+ + \sum_{a'=a+1}^{d}(1+u-a')_+,
\EEQ
which generalizes the hinge loss to ordinal data (see Figure \ref{f-ordinal-hinge}).

\begin{figure}
\centering
\includegraphics[width=\textwidth]{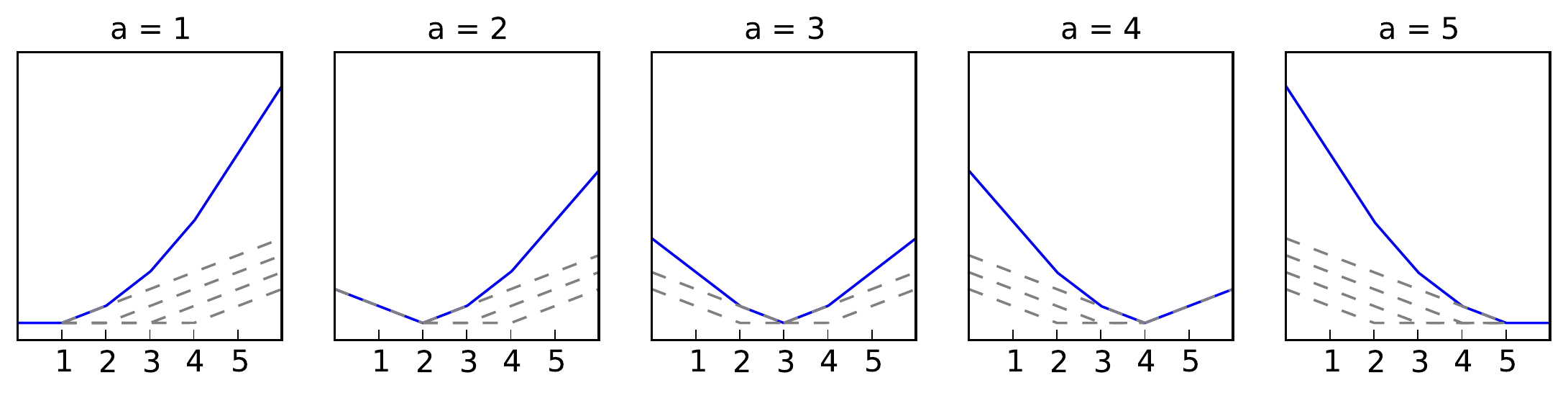}
\caption{\label{f-ordinal-hinge}Ordinal hinge loss.}
\end{figure}

This loss function may be useful for encoding Likert-scale data
indicating degrees of agreement with a question \cite{likert1932}.
For example, we might have
\[\mathcal F_j = \{\mbox{strongly disagree},\mbox{disagree},\mbox{neither agree nor disagree},\mbox{agree},\mbox{strongly agree}\}.
\]
We can encode these levels as the integers $1,\ldots,5$ and use the above loss to 
fit a model to ordinal data. 

This approach assumes that every increment of error is equally bad:
for example, that approximating ``agree'' by ``strongly disagree''
is just as bad as aproximating ``neither agree nor disagree'' by ``agree''.
In \S\ref{s-ord-mpca} we introduce a more flexible ordinal loss function 
that can learn a more flexible relationship between ordinal labels. 
For example, it could determine
that the difference between ``agree'' and ``strongly disagree'' is smaller
than the difference between ``neither agree nor disagree'' and ``agree''.

\paragraph{Interval PCA.}
Suppose that the data $A_{ij} \in \reals^2$ are tuples 
denoting the endpoints of an interval, 
and we wish to find a low rank matrix whose entries lie inside these intervals.
We can capture this objective using, for example, the deadzone-linear loss
\[
L(u,a) = \max((a_1-u)_+,(u-a_2)_+).
\]

\subsection{Missing data and data imputation}\label{s-apca-impute}
We can use the solution $(X,Y)$ to a low rank model 
to impute values corresponding to missing data $(i,j) \not \in \Omega$.
This process is sometimes also called \emph{inference}.
Above, we saw that for quadratically regularized PCA,
the MAP estimator for the missing entry $A_{ij}$ is equal to $x_i y_j$.
This is still true for many of the loss functions above, such as
the Huber function or $\ell_1$ loss,
for which it makes sense for the data to take on any real value.

However, to approximate abstract data types we must consider a more nuanced view.
While we can still think of the solution $(X,Y)$ to the generalized low rank model
(\ref{eq-gpca}) in Boolean PCA as approximating the Boolean matrix $A$,
the solution is not a Boolean matrix.
Instead we say that we have \emph{encoded}
the original Boolean matrix as a real-valued low rank matrix $XY$, or that
we have \emph{embedded} the original Boolean matrix into the space of real-valued matrices.

To fill in missing entries in the original matrix $A$, 
we compute the value $\hat A_{ij}$ that minimizes
the loss for $x_iy_j$:
\[
\hat A_{ij} = \argmin_a L_{ij}(x_i y_j,a).
\]
This implicitly constrains $\hat A_{ij}$ to lie in the domain 
$\mathcal F_j$ of $L_{ij}$.
When $L_{ij}: \reals \times \reals \to \reals$, 
as is the case for the losses in \S\ref{s-gpca} above
(including $\ell_2$, $\ell_1$, and Huber loss), then $\hat A_{ij} = x_i y_j$.
But when the data is of an abstract type, 
the minimum $\argmin_a L_{ij}(u,a)$ will not in general be equal to $u$.

For example, when the data is Boolean, $L_{ij}: \{0,1\} \times \reals \to \reals$,
we compute the Boolean matrix $\hat A$
implied by our low rank model by solving 
\[
\hat A_{ij} = \argmin_{a\in\{0,1\}} (a (XY)_{ij} - 1)_+
\]
for MMMF, or 
\[
\hat A_{ij} = \argmin_{a\in\{0,1\}} \log(1+\exp(-a (XY)_{ij}))
\] 
for logistic PCA.
These problems both have the simple solution 
\[
\hat A_{ij} = \sign(x_i y_j).
\]

When $\mathcal F_j$ is finite, 
inference \emph{partitions} the real numbers into regions
\[
\mathcal R_a = \{x\in\reals: L_{ij}(u,x)=\min_a L_{ij}(u,a)\}
\] 
corresponding to different values $a \in \mathcal F_j$.
When $L_{ij}$ is convex, these regions are intervals.


We can use the estimate $\hat A_{ij}$ even when $(i,j)\in \Omega$ \emph{was} observed.
If the original observations have been corrupted by noise, we can view $\hat A_{ij}$ 
as a denoised version of the original data. 
This is an unusual kind of denoising: both the noisy ($A_{ij}$) and denoised ($\hat A_{ij}$)
versions of the data lie in the \emph{abstract} space $\mathcal F_j$.

\subsection{Interpretations and applications}
We have already discussed some interpretations of $X$ and $Y$ in the PCA setting. Now
we reconsider those interpretations in the context of approximating these abstract data types. 

\paragraph{Archetypes.}
As before, we can think of each row of $Y$ as an \emph{archetype}
which captures the behavior of an idealized example.
However, the rows of $Y$ are real numbers. To represent each archetype $l=1,\ldots,k$ 
in the abstract space as $\mathcal Y_l$ with $(\mathcal Y_l)_j \in \mathcal F_j$, we solve
\[
\mathcal (Y_l)_j = \argmin_{a \in \mathcal F_j} L_j(y_{lj},a).
\]
(Here we assume that the loss $L_{ij} = L_j$ is independent of the example $i$.)

\paragraph{Archetypical representations.}
As before, we call $x_i$ the \emph{representation} of example $i$ in terms of the archetypes.
The rows of $X$ give an embedding of the examples into $\reals^k$,
where each coordinate axis corresponds to a different archetype.
If the archetypes are simple to understand or interpret, then the representation
of an example can provide better intuition about that example.

In contrast to the initial data, which may consist of arbitrarily complex data types, 
the representations $x_i$ will be low dimensional vectors, and can easily be 
plotted, clustered, or used in nearly any kind of machine learning
algorithm. Using the generalized low rank model, we have converted an abstract feature space 
into a vector space.

\paragraph{Feature representations.}
The columns of $Y$ embed the features into $\reals^k$. 
Here we think of the columns of $X$ as archetypical features, and represent
each feature $j$ as a linear combination of the archetypical features.
Just as with the examples, we might choose to apply any machine learning algorithm
to the feature representations.

This procedure allows us to compare non-numeric features using their representation
in $\reals^l$.
For example, if the features $\mathcal F$ are Likert variables giving the 
extent to which respondents on a questionnaire agree with statements $1,\ldots,n$,
we might be able to say that questions $i$ and $j$ are similar if $\|y_i - y_j\|$ is small; 
or that question $i$ is a more polarizing form of question $j$ if 
$y_i = \alpha y_j$, with $\alpha>1$.

Even more interesting, it allows us to compare features of different types. We 
could say that the real-valued feature $i$ is similar to Likert-valued question $j$
if $\|y_i - y_j\|$ is small. 

\paragraph{Latent variables.}
Each row of $X$ represents an example by a vector in $\reals^k$. 
The matrix $Y$ maps these representations back into the original feature space
(now nonlinearly) as described in the discussion on data imputation in \S\ref{s-apca-impute}.
We might think of $X$ as discovering the \emph{latent variables} that 
best explain the observed data, with the added benefit that these latent variables
lie in the vector space $\reals^k$. 
If the approximation error $\sum_{(i,j) \in \Omega} L_{ij}(x_i y_j,A_{ij})$ is small,
then we view these latent variables as providing a good explanation 
or summary of the full data set.

\paragraph{Probabilistic intepretation.}
We can give a probabilistic interpretation of $X$ and $Y$, generalizing the 
hierarchical Bayesian model presented by Fithian and Mazumder in \cite{fithian2013}.
We suppose that the matrices $\bar X$ and $\bar Y$ are generated according to a probability distribution
with probability proportional to $\exp(-r(\bar X))$ and $\exp(-\tilde r(\bar Y))$, respectively.
Our observations $A$ of the entries in the matrix $\bar Z = \bar X \bar Y$ 
are given by
\[
A_{ij} = \psi_{ij}((\bar X \bar Y)_{ij}),
\]
where the random variable $\psi_{ij}(u)$ takes value $a$
with probability proportional to
\[
\exp\left(-L_{ij}(u, a)\right).
\]
We observe each entry $(i,j)\in\Omega$.
Then to find the maximum a posteriori (MAP) estimator $(X,Y)$ of $(\bar X, \bar Y)$, we solve
\[
\begin{array}{ll}
\mbox{maximize} &
\exp\left(-\sum_{(i,j)\in\Omega} L_{ij}(x_i y_j, A_{ij})\right) 
\exp(-r(X)) \exp(-\tilde r(Y)),
\end{array}
\]
which is equivalent, by taking logs, to problem~(\ref{eq-apca}).

This interpretation gives us a simple way to interpret our procedure for imputing 
missing observations $(i,j)\not\in\Omega$.
We are simply computing the MAP estimator $\hat A_{ij}$.

\paragraph{Auto-encoder.}
The matrix $X$ encodes the data; the matrix $Y$ decodes it back into the full space.
We can view (\ref{eq-apca}) as providing the best linear auto-encoder for the data.
Among all linear encodings ($X$) and decodings ($Y$) of the data,
the abstract generalized low rank model (\ref{eq-apca}) minimizes the 
reconstruction error measured according to the loss functions $L_{ij}$.

\paragraph{Compression.}
We impose an information bottleneck by using a low rank auto-encoder to fit the data.
The bottleneck is imposed by both the dimensionality reduction and the regularization,
giving both soft and hard constraints on the information content allowed. 
The solution $(X,Y)$ to problem~(\ref{eq-apca}) maximizes the information transmitted through this $k$-dimensional
bottleneck, measured according to the loss functions $L_{ij}$.
This $X$ and $Y$ give a compressed and real-valued representation that may be used to more 
efficiently store or transmit the information present in the data.

\subsection{Offsets and scaling}\label{s-apca-scaling}
Just as in the previous section, better practical performance can often be achieved by allowing an offset in the model as described in \S\ref{s-rpca-offset},
and automatic scaling of loss functions as described in 
\S\ref{s-gpca-scaling}.
As we noted in \S\ref{s-gpca-scaling}, scaling the loss functions (instead of standardizing the data)
has the advantage that no arithmetic is performed directly on the data $A$.
When the data $A$ consists of abstract types, 
it is quite important that no arithmetic is performed on the data,
so that we need not take the average of, 
say, ``very much'' and ``a little'',
or subtract it from ``not at all''.

\subsection{Numerical examples}
In this section we give results of some small experiments illustrating 
the use of different loss functions
adapted to abstract data types, and comparing their performance to quadratically
regularized PCA. To fit these GLRMs, we use alternating minimization and solve
the subproblems with subgradient descent. This approach is explained more fully in
\S\ref{s-algorithms}.
Running the alternating subgradient method multiple times on the same GLRM from
different initial conditions yields different models, all with very
similar (but not identical) objective values.

\paragraph{Boolean PCA.} \label{ss-bool}

For this experiment, we generate Boolean data $A \in \{-1, +1\}^{n \times m}$ as
\[
    A = \sign{\left(X^{\mathrm{true}} Y^{\mathrm{true}}\right)},
\]
where $X^{\mathrm{true}} \in \reals^{n \times k_{\mathrm{true}}}$ and 
$Y^{\mathrm{true}} \in \reals^{k_{\mathrm{true}} \times m}$ 
have independent, standard normal entries. 
We consider a problem instance with $m = 50$, $n = 50$, and
$k_{\mathrm{true}} = k = 10$.

We fit two GLRMs to this data to compare their performance.
Boolean PCA uses hinge loss $L(u, a) = \max{(1-au, 0)}$
and quadratic regularization $r(u) = \tilde r(u) = .1 \|u\|_2^2$,
and produces the model $(X^{\mathrm{bool}}, Y^{\mathrm{bool}})$.
Quadratically regularized PCA uses squared loss $L(u, a) = (u-a)^2$
and the same quadratic regularization, 
and produces the model $(X^{\mathrm{real}}, Y^{\mathrm{real}})$. 

Figure~\ref{f-boolean-boolean} shows the results of fitting Boolean PCA to this data.
The first column shows the original ground-truth data $A$;
the second shows the imputed data given the model, $\hat A^{\mathrm{bool}}$, generated by
rounding the entries of $X^{\mathrm{bool}}Y^{\mathrm{bool}}$ to the closest number in ${0,1}$
(as explained in \S\ref{s-apca-impute});
the third shows the error $A - \hat A^{\mathrm{bool}}$.
Figure~\ref{f-boolean-boolean} shows the results of running quadratically regularized PCA
on the same data, and shows 
$A$, $\hat A^{\mathrm{real}}$, and $A - \hat A^{\mathrm{real}}$.

As expected, Boolean PCA performs substantially better 
than quadratically regularized PCA on this data set.
On average over 100 draws from the ground truth data distribution,
the misclassification error (percentage of misclassified entries)
\[
    \epsilon(X,Y; A) = \frac{\#\{(i,j) \mid A_{ij} \ne \sign{(XY)_{ij}} \}}{mn}
\]
is much lower using hinge loss
($\epsilon(X^{\mathrm{bool}}, Y^{\mathrm{bool}}; A) = 0.0016$)
than squared loss
($\epsilon(X^{\mathrm{real}}, Y^{\mathrm{real}}; A) = 0.0051$).
The average RMS errors
\[
    \mbox{RMS}(X,Y; A) =  \left(\frac{1}{mn} \sum_{i=1}^m \sum_{j=1}^n (A_{ij} - (XY)_{ij})^2\right)^{1/2}
\]
using hinge loss ($\mbox{RMS}(X^{\mathrm{bool}}, Y^{\mathrm{bool}}; A) = 0.0816$) 
and squared loss ($\mbox{RMS}(X^{\mathrm{real}}, Y^{\mathrm{real}}; A) = 0.159$)
also indicate an advantage for Boolean PCA.

\begin{figure}[htb!]
\begin{center}
\includegraphics[height = 0.2\textheight]{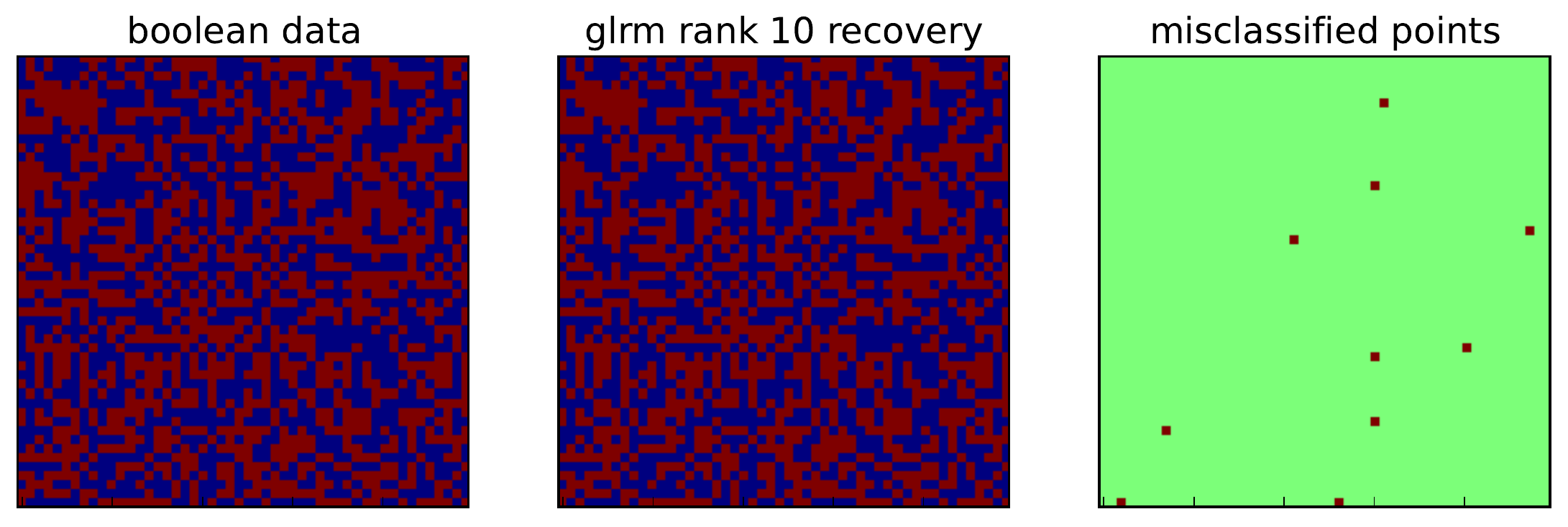}
\end{center}
\caption{\label{f-boolean-boolean}Boolean PCA on Boolean data.}
\end{figure}

\begin{figure}[htb!]
\begin{center}
\includegraphics[height = 0.2\textheight]{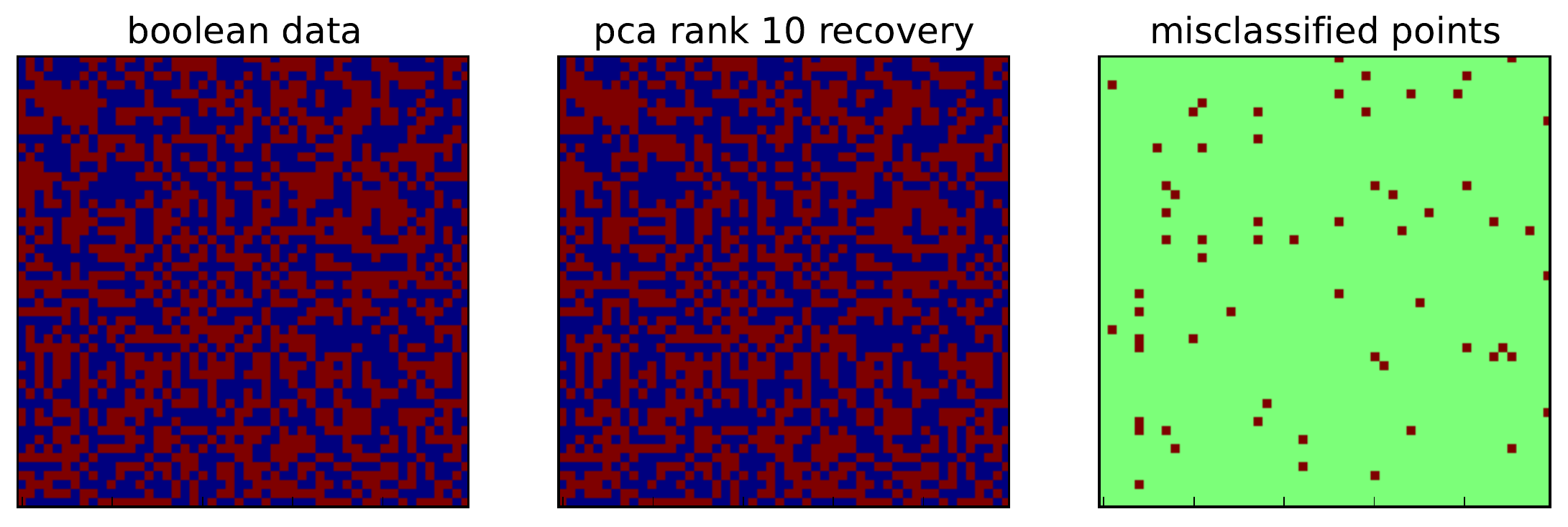}
\end{center}
\caption{\label{f-boolean-pca}Quadratically regularized PCA on Boolean data.}
\end{figure}

\paragraph{Censored PCA.}

In this example, we consider the performance of Boolean PCA
when only a subset of positive entries in 
the Boolean matrix $A \in \{-1, 1\}^{m \times n}$ have been observed, 
\ie, the data has been \emph{censored}.
For example, a retailer might know only a subset of the products each customer purchased;
or a medical clinic might know only a subset of the diseases a patient has contracted,
or of the drugs the patient has taken.
Imputation can be used in this setting to (attempt to) 
distinguish true negatives $A_{ij} = -1$
from unobserved positives $A_{ij} = +1$, $(i,j) \not \in \Omega$.

We generate a low rank matrix $B = XY \in [0,1]^{m \times n}$ with 
$X\in\reals^{m\times k}$, $Y\in\reals^{k \times n}$,
where the entries of $X$ and $Y$ are drawn from a uniform distribution on $[0,1]$,
$m=n=300$ and $k=3$.
Our data matrix $A$ is chosen by letting $A_{ij} = 1$ with probability 
proportional to $B_{ij}$, and $-1$ otherwise;
the constant of proportionality is chosen so that half of the entries in $A$ are positive.
We fit a rank 5 GLRM to an observation set $\Omega$ consisting of 
$10\%$ of the positive entries in the matrix, drawn uniformly at random, 
using hinge loss and quadratic regularization.
That is, we fit the low rank model
\[
\begin{array}{ll}
\mbox{minimize} & \sum_{(i,j) \in \Omega} \max(1 - x_i y_j A_{ij}, 0) 
+ \gamma \sum_{i=1}^m \|x_i\|_2^2 + \gamma \sum_{j=1}^n \|y_j\|_2^2
\end{array}
\]
and vary the regularization parameter $\gamma$.

We consider three error metrics to measure the performance 
of the fitted model $(X,Y)$:
normalized training error, 
\[
\frac{1}{|\Omega|}\sum_{(i,j) \in \Omega} \max(1 - A_{ij} x_i y_j, 0),
\]
normalized test error, 
\[
\frac{1}{|\Omega^C|}\sum_{(i,j) \in \Omega^C} \max(1 - A_{ij} x_i y_j, 0),
\]
and \emph{precision at 10} (p@10), which is computed 
as the fraction of the top ten predicted values not in the observation set,
$\{x_i y_j: (i,j) \in \Omega^C\}$, for which $A_{ij}=1$.
(Here, $\Omega^C = \{1,\ldots, m\} \times \{1,\ldots, n\} \setminus \Omega$.)
Precision at 10 measures the usefulness of the model: if we predicted that 
the top 10 unseen elements $(i,j)$ had values $+1$, how many would we get right? 

Figure~\ref{f-censored} shows the regularization path as $\gamma$
ranges from 0 to 40, averaged over 50 samples from the distribution generating the data.
Here, we see that while the training error decreases as $\gamma$ decreases, the test
error reaches a minimum around $\gamma = 5$.
Interestingly, the precision at 10 improves as the regularization increases;
since precision at 10 is computed using only relative rather than absolute values of the
model, it is insensitive to the shrinkage of the parameters
introduced by the regularization.
The grey line shows the probability of identifying a positive entry by guessing randomly; 
precision at 10, which exceeds $80\%$ when $\gamma \gtrsim 30$, is significantly higher.
This performance is particularly impressive given that the observations $\Omega$
are generated by sampling from rather than rounding the auxiliary matrix B.

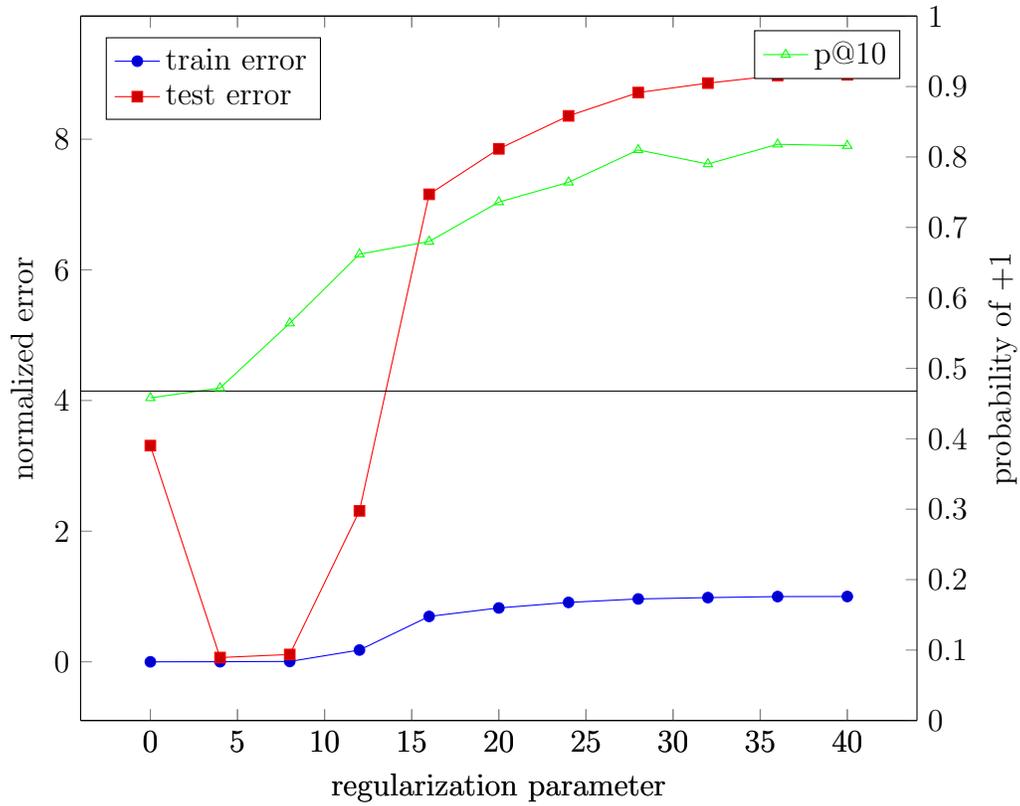
\begin{figure}
\begin{centering}
\begin{tikzpicture}[]
\begin{axis}[view = {0}{90}, axis y line*=left, width = 5in, xlabel = regularization parameter, ylabel = normalized error, legend pos=north west]
\addplot+ coordinates {
(40.0, 0.9999250275022336)
(36.0, 0.998645592703602)
(32.0, 0.9830282610127545)
(28.0, 0.9624431669025898)
(24.0, 0.9093238824715273)
(20.0, 0.8245805204249774)
(16.0, 0.6949344902207811)
(12.0, 0.1806216908618794)
(8.0, 0.005602388796491626)
(4.0, 0.0014610770292638547)
(0.0, 2.2410618538422013e-5)
};
\addlegendentry{train error}
\addplot+ coordinates {
(40.0, 8.985700255329329)
(36.0, 8.9758731137834)
(32.0, 8.858315286350836)
(28.0, 8.713651933666064)
(24.0, 8.357011604287244)
(20.0, 7.852363931187378)
(16.0, 7.1564876377135045)
(12.0, 2.310889850642374)
(8.0, 0.11232973341434623)
(4.0, 0.06555686970026971)
(0.0, 3.309612957804895)
};
\addlegendentry{test error}
\end{axis}
\begin{axis}[view = {0}{90}, axis y line*=right, width = 5in, xlabel = regularization parameter, ylabel = probability of +1, ymin = 0, ymax = 1]
\addplot[mark=triangle,color=green] coordinates {
(40.0, 0.816)
(36.0, 0.818)
(32.0, 0.79)
(28.0, 0.81)
(24.0, 0.764)
(20.0, 0.7360000000000001)
(16.0, 0.68)
(12.0, 0.6619999999999999)
(8.0, 0.5640000000000001)
(4.0, 0.47200000000000003)
(0.0, 0.458)
};
\addlegendentry{p@10}
\draw[thin] (axis cs:\pgfkeysvalueof{/pgfplots/xmin},.4677) -- (axis cs:\pgfkeysvalueof{/pgfplots/xmax},.4677);
\addlegendentry{proportion of +1s in $\Omega^C$}
\end{axis}
\end{tikzpicture}
\caption{\label{f-censored}Error metrics for Boolean GLRM on censored data. The grey line shows the probability of identifying a positive entry by guessing randomly.}
\end{centering}
\end{figure}

\paragraph{Mixed data types.} \label{ss-mix}

In this experiment, we fit a GLRM to a data table with numerical,
Boolean, and ordinal columns generated as follows.
Let $\mathcal N_1$, $\mathcal N_2$, and $\mathcal N_3$ partition the column indices
$1,\ldots,n$.
Choose $X^{\mathrm{true}} \in \reals^{m \times k_{\mathrm{true}}}$,
$Y^{\mathrm{true}} \in \reals^{k_{\mathrm{true}} \times n}$ to have independent,
standard normal entries.
Assign entries of $A$ as follows:
\[
    A_{ij} = \left\{
    \begin{array}{ll}
         x_iy_j & j \in \mathcal N_1 \\
         \sign{\left(x_iy_j\right)} & j \in \mathcal N_2 \\
         \round(3x_iy_j +1) & j \in \mathcal N_3,
    \end{array}\right.
\]
where the function $\round$ maps $a$ to the nearest integer
in the set $\{1, \ldots, 7 \}$.
Thus, $\mathcal N_1$ corresponds to real-valued data;
$\mathcal N_2$ corresponds to Boolean data;
and $\mathcal N_3$ corresponds to ordinal data. 
We consider a problem instance in which 
$m = 100$, $n_1 = 40$, $n_2 = 30$, $n_3 = 30$, and $k_{\mathrm{true}} = k =
10$. 

We fit a heterogeneous loss GLRM to this data with loss function
\[
    L_{ij}(u,a) = \left\{
        \begin{array}{ll}
            L_{\mathrm{real}}(u,a) & j \in \mathcal N_1\\
            L_{\mathrm{bool}}(u,a) & j \in \mathcal N_2\\
            L_{\mathrm{ord}}(u,a) & j \in \mathcal N_3,
        \end{array} \right. 
\]
where $L_{\mathrm{real}}(u,a) = (u-a)^2$, $L_{\mathrm{bool}}(u,a) =
(1 - au)_+$, and $L_{\mathrm{ord}}(u,a)$ is defined in (\ref{eq-ord-pca}),
and with quadratic regularization $r(u) = \tilde r(u) = .1 \|u\|_2^2$.
We fit the GLRM to produce the model $(X^{\mathrm{mix}},Y^{\mathrm{mix}})$.
For comparison, we also fit quadratically regularized PCA to the same data,
using $L_{ij}(u,a) = (u-a)^2$ for all $j$ 
and quadratic regularization $r(u) = \tilde r(u) = .1 \|u\|_2^2$,
to produce the model $(X^{\mathrm{real}}, Y^{\mathrm{real}})$.



Figure~\ref{f-mixed-glrm} shows the results of fitting the heterogeneous loss GLRM
to the data.
The first column shows the original ground-truth data $A$;
the second shows the imputed data given the model, $\hat A^{\mathrm{mix}}$, generated by
rounding the entries of $X^{\mathrm{mix}}Y^{\mathrm{mix}}$ to the closest number in ${0,1}$
(as explained in \S\ref{s-apca-impute});  
the third shows the error $A - \hat A^{\mathrm{mix}}$.
Figure~\ref{f-mixed-pca} corresponds to quadratically regularized PCA, and shows 
$A$, $\hat A^{\mathrm{real}}$, and $A - \hat A^{\mathrm{real}}$.

To evaluate error for Boolean and ordinal data, we use the misclassification
error $\epsilon$ defined above. For notational convenience, we let 
$Y_{\mathcal N_l}$ ($A_{\mathcal N_l}$) denote
$Y$ ($A$) restricted to the columns $\mathcal N_l$ in order to pick out 
real-valued columns ($l=1$), Boolean columns ($l=2$), and ordinal columns ($l=3$).

Table \ref{results_mix} compare the average error (difference between imputed entries and ground truth) over 100
draws from the ground truth distribution for models using heterogeneous loss ($X^{\mathrm{mix}}$,
$Y^{\mathrm{mix}}$) and quadratically regularized loss ($X^{\mathrm{real}}$,
$Y^{\mathrm{real}}$). Columns are labeled by error metric.
We use misclassification error $\epsilon$ for Boolean and
ordinal data and MSE for numerical data.
\begin{table}[h]
\begin{center}

     \begin{tabular}{ | l | c | c | c | }
    \hline
                        &$\mbox{MSE}(X, Y_{\mathcal N_1}; A_{\mathcal N_1})$ & 
                        $\epsilon(X, Y_{\mathcal N_2}; A_{\mathcal N_2})$ & 
                        $\epsilon(X, Y_{\mathcal N_3}; A_{\mathcal N_3})$ \\ \hline
    $X^{\mathrm{mix}}$, $Y^{\mathrm{mix}}$  & $0.0224$ & $0.0074$ & $0.0531$ \\ \hline
    $X^{\mathrm{real}}$ , $Y^{\mathrm{real}}$ & $0.0076$ & $0.0213$ & $0.0618$
    \\ \hline
  \end{tabular}
  \caption{Average error for numerical, Boolean, and ordinal features using GLRM
  with heterogenous loss and quadratically regularized loss.}
   \label{results_mix}
   \end{center}
\end{table}

%
%
\begin{figure}
\begin{centering}
\includegraphics[height = 0.2\textheight]{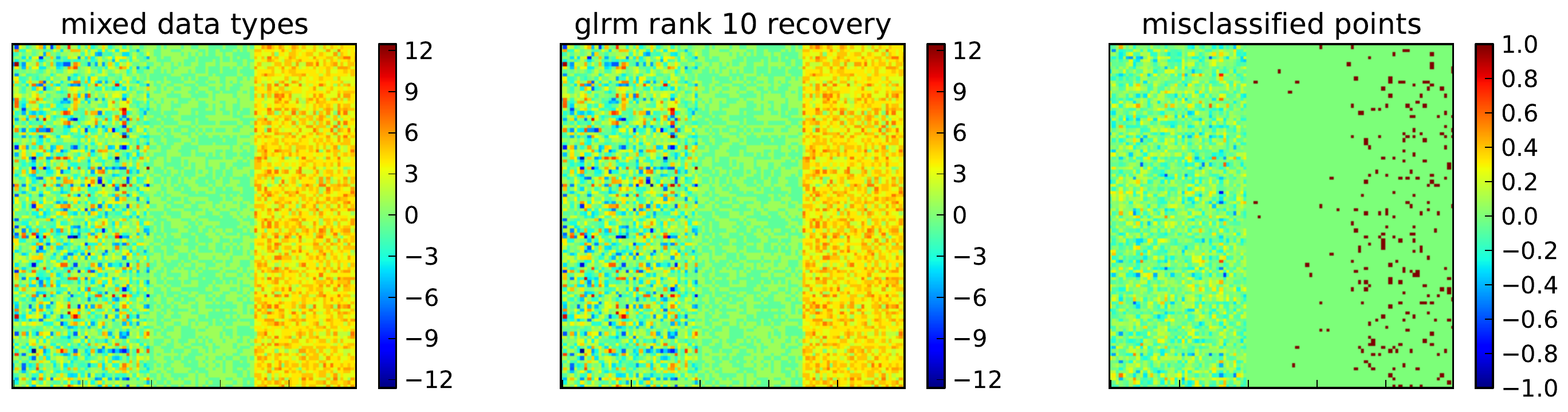}
\end{centering}
\caption{\label{f-mixed-glrm}Heterogeneous loss GLRM on mixed data.}
\end{figure}

\begin{figure}
\begin{centering}
\includegraphics[height = 0.2\textheight]{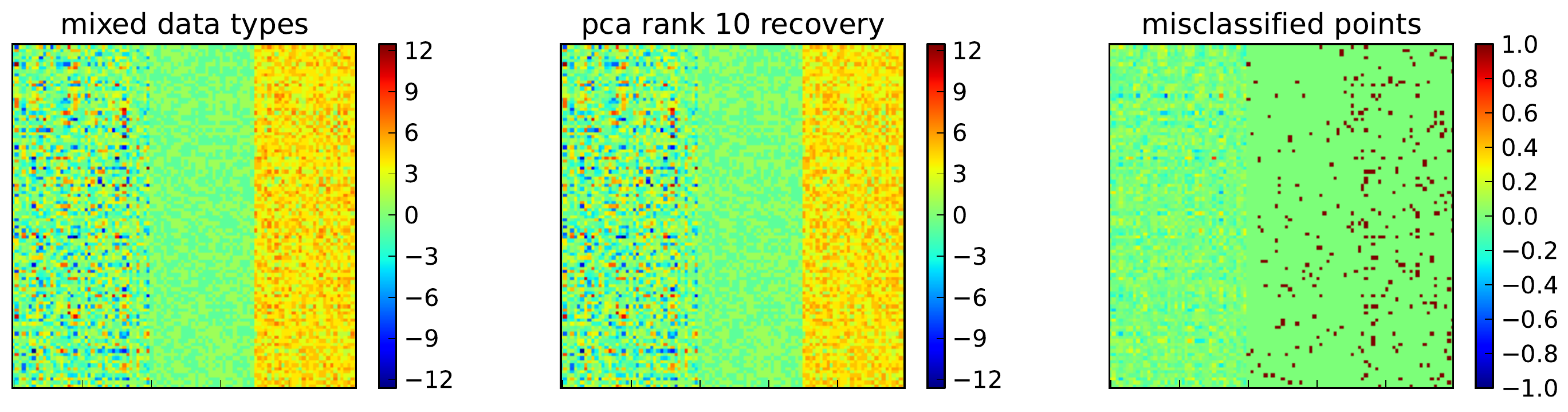}
\end{centering}
\caption{\label{f-mixed-pca}Quadratically regularized PCA on mixed data.}
\end{figure}

\paragraph{Missing data.}
Here, we explore the effect of missing entries on the accuracy of the recovered model.
We generate data $A$ as detailed above, 
but then censor one large block of entries in the table
(constituting 3.75\% of numerical, 50\% of Boolean, and 50\% of ordinal data),
removing them from the observed set $\Omega$.

Figure~\ref{f-missing-glrm} shows the results of fitting the heterogeneous loss GLRM
described above 
on the censored data.
The first column shows the original ground-truth data $A$;
the second shows the block of data that has been removed from the
observation set $\Omega$;
the third shows the imputed data given the model, $\hat A^{\mathrm{mix}}$, generated by
rounding the entries of $X^{\mathrm{mix}}Y^{\mathrm{mix}}$ to the closest number in $\{0,1\}$
(as explained in \S\ref{s-apca-impute});
the fourth shows the error $A - \hat A^{\mathrm{mix}}$.
Figure~\ref{f-missing-pca} corresponds to running quadratically regularized PCA on the same data, 
and shows $A$, $\hat A^{\mathrm{real}}$, and $A - \hat A^{\mathrm{real}}$.
While quadradically regularized PCA and the heterogeneous loss GLRM performed 
similarly when no data was missing, the heterogeneous loss GLRM performs much
better than quadradically regularized PCA when a large block of data is
censored.

We compare the average error (difference between imputed entries and ground truth) over 100
draws from the ground truth distribution in Table \ref{results_missing}. 
As above, we use misclassification error $\epsilon$ for Boolean and
ordinal data and MSE for numerical data.

\begin{table}[h]
\begin{center}

     \begin{tabular}{ | l | c | c | c | }
    \hline
                        &$\mbox{MSE}(X, Y_{\mathcal N_1}; A_{\mathcal N_1})$ & 
                        $\epsilon(X, Y_{\mathcal N_2}; A_{\mathcal N_2})$ & 
                        $\epsilon(X, Y_{\mathcal N_3}; A_{\mathcal N_3})$ \\ \hline
    $X^{\mathrm{mix}}$, $Y^{\mathrm{mix}}$  & $0.392$ & $0.2968$ & $0.3396$ \\ \hline
    $X^{\mathrm{real}}$ , $Y^{\mathrm{real}}$ & $0.561$ & $0.4029$ & $0.9418$
    \\ \hline
  \end{tabular}
  \caption{Average error over imputed data using a GLRM with heterogenous
  loss and regularized quadratic loss.}
   \label{results_missing}
   \end{center}
\end{table}

%


\begin{figure}
\begin{centering}
\includegraphics[height = 0.15\textheight]{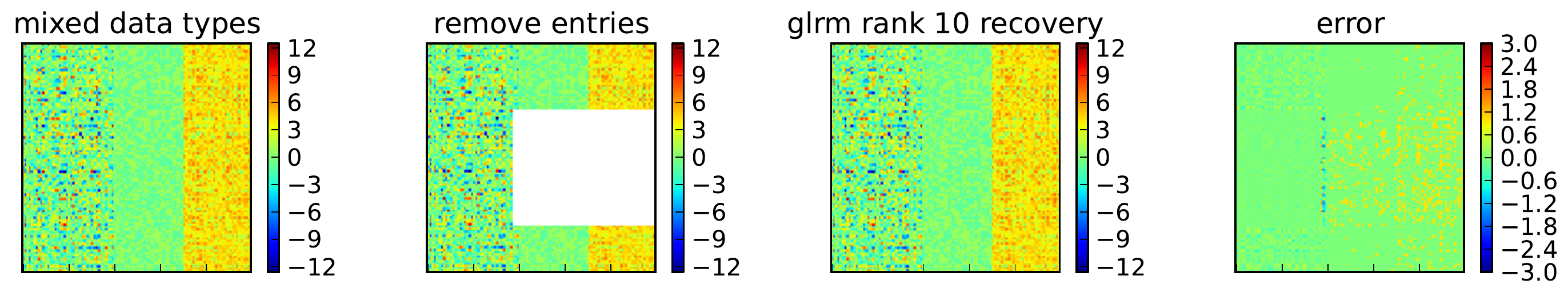}\\
\end{centering}
\caption{\label{f-missing-glrm}Heterogeneous loss GLRM on missing data.}
\end{figure}

\begin{figure}
\begin{centering}
\includegraphics[height = 0.15\textheight]{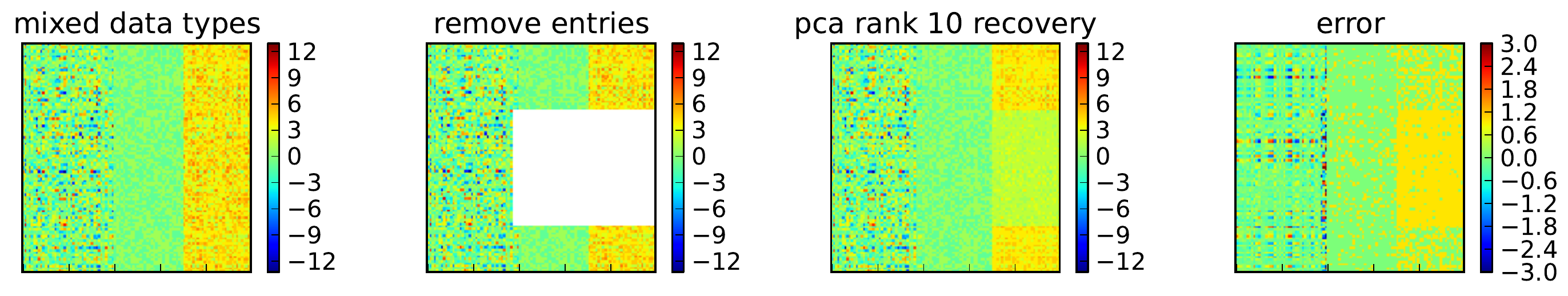}\\
\end{centering}
\caption{\label{f-missing-pca}Quadratically regularized PCA on missing data.}
\end{figure}

\section{Multi-dimensional loss functions}\label{s-mpca}

In this section, we generalize the procedure to allow the loss functions to depend on \emph{blocks} of
the matrix $XY$, which allows us to represent abstract data types more naturally.
For example, we can now represent categorical values 
, permutations, distributions, and rankings. 

We are given a loss function $L_{ij}: \reals^{1\times d_j} \times \mathcal F_j \to
\reals$, where $d_j$ is the \emph{embedding dimension} of feature $j$, 
and $d = \sum_j d_j$ is the total dimension of the embedded features. 
The loss $L_{ij}(u, a)$ describes the
approximation error incurred when we represent a feature value $a \in \mathcal
F_j$ by the vector $u \in \reals^{d_j}$. 

Let $x_i \in \reals^{1 \times k}$ be the $i$th row of $X$ (as before), and let
$Y_j \in \reals^{k \times d_j}$ be the $j$th block matrix of $Y$ so
the columns of $Y_j$ correspond to the columns of embedded feature $j$. 
We now formulate a multi-dimensional generalized low rank model on the database $A$,
\BEQ
\label{eq-mpca}
\begin{array}{ll}
\mbox{minimize} & \sum_{(i, j) \in \Omega} L_{ij}(x_i Y_j, A_{ij}) 
+ \sum_{i=1}^m r_i(x_i) + \sum_{j=1}^n \tilde r_j(Y_j),
\end{array}
\EEQ
with variables $X \in \reals^{n \times k}$ and $Y \in \reals^{k \times d}$, and
with loss $L_{ij}$ as above and regularizers $r_i(x_i): \reals^{1 \times k} \to
\reals$ (as before) and $\tilde r_j(Y_j): \reals^{k \times d_j} \to \reals$. 
Note that the first argument of $L_{ij}$ is a \emph{row} vector with $d_j$ entries,
and the first argument of $r_j$ is a matrix with $d_j$ columns.
When every entry $A_{ij}$ is real-valued (\ie, $d_j = 1$), then we recover the
generalized low rank model (\ref{eq-gpca}) seen in the previous section.

\subsection{Examples}

\paragraph{Categorical PCA.}
Suppose that $a\in \mathcal F$ is a categorical variable, 
taking on one of $d$ values or labels.
Identify the labels with the integers $\{1,\ldots,d\}$.
In (\ref{eq-mpca}), set
\[
L(u,a) = (1-u_a)_+ + \sum_{a' \in \mathcal F, ~a' \ne a} (1 + u_{a'})_+,
\]
and use the quadratic regularizer $r_i = \gamma \|\cdot\|_2^2$, 
$\tilde r = \gamma \|\cdot\|_2^2$.

Fixing $X$ and optimizing over $Y$ is equivalent to training one SVM 
per label to separate that label from all the others:
the $j$th column of $Y$ gives the weight vector corresponding to the $j$th SVM.
(This is sometimes called one-vs-all multiclass classification \cite{rifkin2004}.)
Optimizing over $X$ identifies the low-dimensional feature vectors for each example that allow these SVMs to most accurately predict the labels.

The difference between categorical PCA and Boolean PCA 
is in how missing labels are imputed. 
To impute a label for entry $(i,j)$ with feature vector $x_i$
according to the procedure described above in \ref{s-apca-impute}, 
we project the representation $Y_j$ 
onto the line spanned by $x_i$ to form $u = x_i Y_j$.
Given $u$, the imputed label is simply $\argmax_l u_l$. 
This model has the interesting property that if column $l'$ of $Y_j$ lies 
in the interior of the convex hull of the columns of $Y_j$, 
then $u_{l'}$ will lie in the interior of the interval $[\min_l u_l, \max_l u_l]$ \cite{boyd2004}.
Hence the model will never impute label $l'$ for any example. 

We need not restrict ourselves to the loss function given above.
In fact, any loss function that can be used to train a classifier for categorical variables 
(also called a multi-class classifier) 
can be used to fit a categorical PCA model,
so long as the loss function depends only on the inner products between 
the parameters of the model and the features corresponding to each example.
The loss function becomes the loss function $L$ used in (\ref{eq-mpca});
the optimal parameters of the model give the optimal matrix $Y$,
while the implied features will populate the optimal matrix $X$.
For example, it is possible to use loss functions derived from
error-correcting output codes \cite{dietterich1995};
the Directed Acyclic Graph SVM \cite{platt1999};
the Crammer-Singer multi-class loss \cite{crammer2002};
or the multi-category SVM \cite{lee2004}.

Of these loss functions, only the one-vs-all loss is separable across the classes $a \in \mathcal F$.
(By separable, we mean that the objective value can be written as a sum over the classes.)
Hence fitting a categorical features with any other loss functions is \emph{not}
the same as fitting $d$ Boolean features.
For example, in the Crammer-Singer loss
\[
L(u,a) = (1 - u_a + \max_{a' \in \mathcal F, ~a' \not = a} u_a')_+,
\]
the classes are combined according to their maximum, rather than their sum.
While one-vs-all classification performs about as well as more sophisticated loss functions
on small data sets \cite{rifkin2004}, these more sophisticated nonseparable loss 
tend to perform much better as the number of classes (and examples) increases \cite{gupta2014}.

Some interesting nonconvex loss functions have also been suggested 
for this problem.
For example, consider a generalization of Hamming distance to this setting,
\[
L(u,a) = \delta_{u_a, 1} + \sum_{a' \ne a} \delta_{u_{a'}, 0},
\]
where $\delta_{\alpha, \beta} = 0$ if $\alpha = \beta$ and 1 otherwise.
In this case, alternating minimization with regularization that enforces a 
clustered structure in the low rank model 
(see the discussion of quadratic clustering in \S\ref{s-kmeans})
reproduces the $k$-modes algorithm \cite{huang1999}.

\paragraph{Ordinal PCA.}\label{s-ord-mpca}
We saw in \S\ref{s-apca}
one way to fit a GLRM to ordinal data.
Here, we use a larger embedding dimension for ordinal features.
The multi-dimensional embedding will be particularly useful
when the best mapping of the ordinal variable onto a linear scale is not uniform;
\eg, if level 1 of the ordinal variable is much more similar to level 2 than level 2 is
to level 3. Using a larger embedding dimension allows us to infer the relations between
the levels from the data itself.
Here we again identify the labels $a \in \mathcal F$ with the integers $\{1,\ldots,d\}$.

One approach we can use for (multi-dimensional) ordinal PCA
is to solve (\ref{eq-mpca}) with the loss function
\BEQ
\label{eq-ord-mpca}
L(u,a) = \sum_{a'=1}^{d -1} (1 - I_{a > a'} u_{a'})_+,
\EEQ
and with quadratic regularization.
Fixing $X$ and optimizing over $Y$ is equivalent to training an SVM 
to separate labels $a \le l$ from $a > l$ for each $l \in \mathcal F$. 
This approach produces a set of hyperplanes 
(given by the columns of $Y$) separating each level $l$ from the next.
The hyperplanes need not be parallel to each other.
Fixing $Y$ and optimizing over $X$ finds the low dimensional features vector
for each example that places the example between the appropriate hyperplanes.
(See Figure \ref{f-mpca-ordinal} for an illustration of an optimal fit 
of this loss function, with $k=2$, to a simple synthetic data set.)

\begin{figure}
\centering
\includegraphics[width=.5\textwidth]{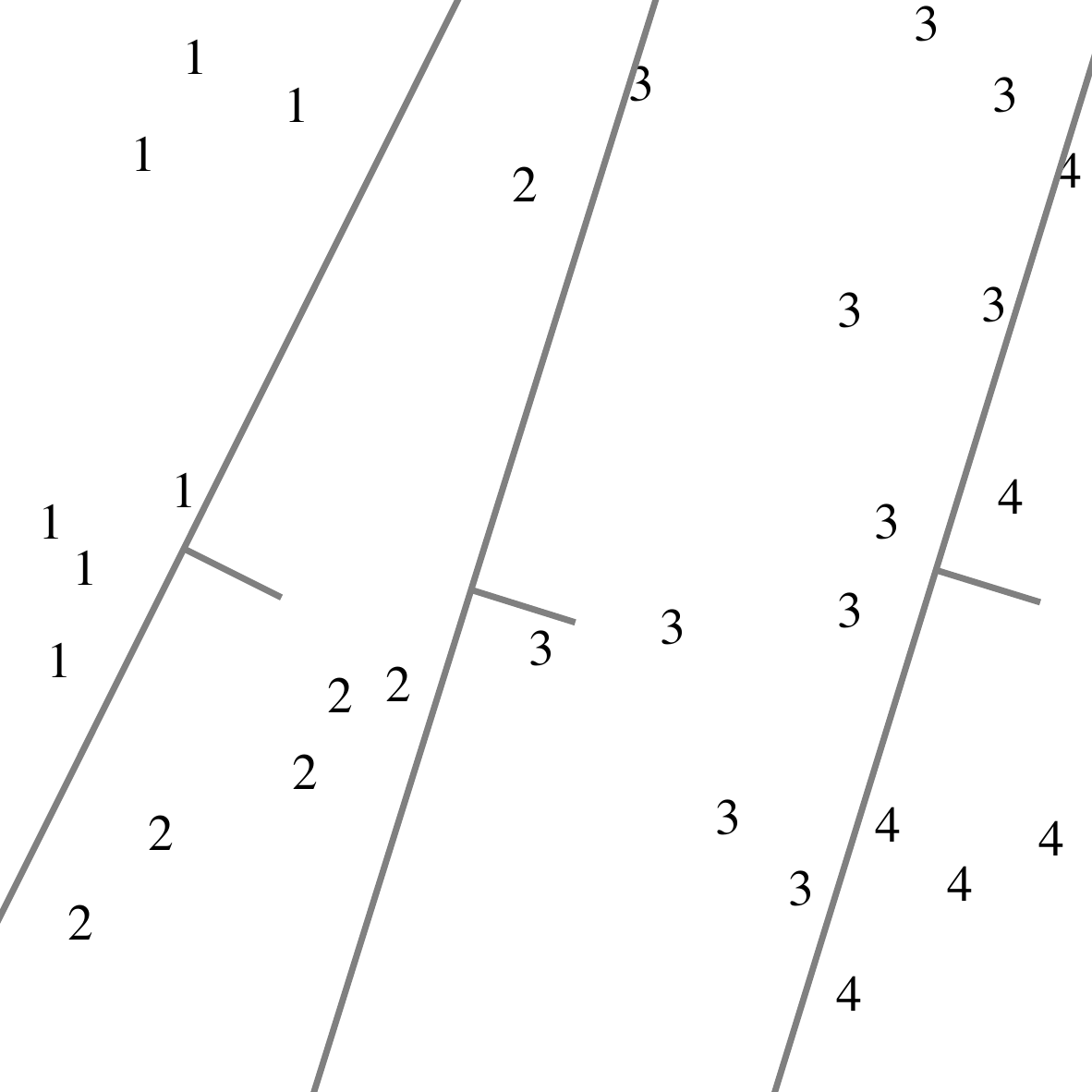}
\caption{\label{f-mpca-ordinal}Multi-dimensional ordinal loss.}
\end{figure}

\paragraph{Permutation PCA.}
Suppose that $a$ is a permutation of the numbers $1,\ldots,d$.
Define the permutation loss
\[
L(u,a) = \sum_{i = 1}^{d-1} (1 - u_{a_i} + u_{a_{i+1}})_+.
\]
This loss is zero if $u_{a_i} > u_{a_{i+1}} + 1$ for $i=1,\ldots, d-1$,
and increases linearly when these inequalities are violated.
Define $\sort(u)$ to return a permutation $\hat a$ of the indices $1,\ldots,d$
so that $u_{\hat a_i} \ge u_{\hat a_{i+1}}$ for $i=1,\ldots, d-1$.
It is easy to check that $\argmin_a L(u,a) = \sort(u)$.
Hence using the permutation loss function in generalized PCA (\ref{eq-mpca})
finds a low rank approximation of a given table of permutations.

\paragraph{Ranking PCA.}
Many variants on the permutation PCA problem are possible.
For example, in ranking PCA, we interpret the permutation as 
a ranking of the choices $1,\ldots,d$, and penalize deviations of many levels 
more strongly than deviations of only one level by choosing the loss
\[
L(u,a) = \sum_{i = 1}^{d-1} \sum_{j = i+1}^{d} 
(1 - u_{a_i} + u_{a_{j}})_+.
\]

From here, it is easy to generalize to a setting in which the rankings are
only partially observed. Suppose that we observe pairwise comparisons
$a \subseteq \{1,\ldots,d\} \times \{1,\ldots,d\}$, where
$(i,j) \in a $ means that choice $i$ was ranked above choice $j$. 
Then a loss function penalizing devations from these observed rankings is 
\[
L(u,a) = \sum_{(i,j) \in a} 
(1 - u_{a_i} + u_{a_{j}})_+.
\]

Many other modifications to ranking loss functions 
have been proposed in the literature that interpolate
between the the two first loss functions proposed above,
or which prioritize correctly predicting the top ranked choices.
These losses include
the area under the curve loss \cite{steck2007},
ordered weighted average of pairwise classification losses \cite{usunier2009},
the weighted approximate-rank pairwise loss \cite{weston2010},
the $k$-order statistic loss \cite{weston2013},
and the accuracy at the top loss \cite{boyd2012aatp}.

\subsection{Offsets and scaling}
 Just as in the previous section, better practical performance can often be achieved by allowing an offset in the model as described in \S\ref{s-rpca-offset},
and scaling loss functions as described in 
\S\ref{s-gpca-scaling}.

\subsection{Numerical examples}\label{s-census}

We fit a low rank model to the 2013 American Community Survey (ACS) to illustrate
how to fit a low rank model to heterogeneous data.

The ACS is a survey administered to 1\% of the population of the United States each year
to gather their responses to a variety of demographic and economic questions.
Our data sample consists of $m=3132796$ responses gathered from residents of the US, 
excluding Puerto Rico, in the year 2013, on the 23 questions listed in Table \ref{t-ACS-q}.

\begin{table}
\begin{center}
\begin{tabular}{l | l | l}
Variable & Description & Type \\
\hline
HHTYPE & household type & categorical \\
STATEICP & state & categorical \\
OWNERSHP & own home & Boolean \\
COMMUSE & commercial use & Boolean \\
ACREHOUS & house on $\geq 10$ acres & Boolean \\
HHINCOME & household income & real \\
COSTELEC & monthly electricity bill & real \\
COSTWATR & monthly water bill & real \\
COSTGAS & monthly gas bill & real \\
FOODSTMP & on food stamps & Boolean \\
HCOVANY & have health insurance & Boolean \\
SCHOOL & currently in school & Boolean \\
EDUC & highest level of education & ordinal \\
GRADEATT & highest grade level attained & ordinal \\
EMPSTAT & employment status & categorical \\
LABFORCE & in labor force & Boolean \\
CLASSWKR & class of worker & Boolean \\
WKSWORK2 & weeks worked per year & ordinal \\
UHRSWORK & usual hours worked per week & real \\
LOOKING & looking for work & Boolean \\
MIGRATE1 & migration status & categorical \\
\end{tabular}
\caption{\label{t-ACS-q}ACS variables.}
\end{center}
\end{table}

We fit a rank 10 model to this data using 
Huber loss for real valued data,
hinge loss for Boolean data,
ordinal hinge loss for ordinal data,
one-vs-all categorical loss for categorical data,
and regularization parameter $\gamma = .1$.
We allow an offset in the model and scale the loss functions and regularization
as described in \S\ref{s-gpca-scaling}.



In Table~\ref{t-similar}, we select a few features $j$ from the model, along with their
associated vectors $y_j$, and find the two features most similar to them by finding 
the two features $j'$ which minimize $\cos(y_j, y_{j'})$.
The model automatically groups states which intuitively share demographic features:
for example, three wealthy states adjoining (but excluding) a major metropolitan area
--- Virginia, Maryland, and Connecticut --- are grouped together.
The low rank structure also identifies the results (high water prices) of the prolonged 
drought afflicting California, and corroborates the intuition that work leads only
to more work: hours worked per week, weeks worked per year, and education level are 
highly correlated.

\begin{table}
\begin{center}
\begin{tabular}{l|l}
Feature & Most similar features\\
\hline
Alaska & Montana, North Dakota \\
California & Illinois, cost of water \\
Colorado & Oregon, Idaho \\
Ohio & Indiana, Michigan \\
Pennsylvania & Massachusetts, New Jersey \\
Virginia & Maryland, Connecticut \\
Hours worked & weeks worked, education \\
\end{tabular}
\caption{\label{t-similar}Most similar features in demography space.}
\end{center}
\end{table}

\section{Fitting low rank models} \label{s-algorithms}

In this section, we discuss a number of algorithms that may be used to fit
generalized low rank models.
As noted earlier, 
it can be computationally hard to find the global optimum of a generalized low rank model.
For example, it is NP-hard to compute an exact solution to
$k$-means \cite{drineas2004}, nonnegative matrix factorization \cite{vavasis2009},
and weighted PCA and matrix completion \cite{gillis2011}
all of which are special cases of low rank models.

In \S\ref{s-am}, we will examine a number of local optimization methods
based on alternating minimization.
Algorithms implementing lazy variants of alternating minimization,
such as the alternating gradient, proximal gradient, or stochastic gradient algorithms,
are faster per iteration than alternating minimization, 
although they may require more iterations for convergence.
In numerical experiments, we notice that lazy variants often converge to points
with a lower objective value: it seems that these lazy variants are less
likely to be trapped at a saddle point than is alternating minimization.
\S\ref{s-convergence} explores the convergence of these algorithms in practice.

We then consider a few special cases in which we can show that 
alternating minimization converges to the global optimum in 
some sense: for example, we will see convergence with high probability,
approximately, and in retrospect.
\S\ref{s-initialization} discusses a few strategies for initializing
these local optimization methods, with provable guarantees in
special cases.
\S\ref{s-optimality-certificate} shows that for problems with
convex loss functions and quadratic regularization, 
it is sometimes possible to certify
global optimality of the resulting model.

\subsection{Alternating minimization}\label{s-am}
We showed earlier how to use alternating minimization to find an (approximate) solution
to a generalized low rank model. 
Algorithm (\ref{alg-am}) shows how to explicitly extend alternating minimization
to a generalized low rank model (\ref{eq-gpca}) with observations $\Omega$.

\begin{algorithm}
\caption{\label{alg-am}}
\begin{algorithmic}
\State \textbf{given} $X^0$, $Y^0$
\For{$k=1,2,\ldots$}
	\For{$i=1,\ldots,M$}
	    \State $x_i^k = \argmin_x \left( \sum_{j: (i,j) \in \Omega} L_{ij}(x y^{k-1}_j, A_{ij}) + r(x)\right)$
	\EndFor
	\For{$j=1,\ldots,N$}
	    \State $y_j^k = \argmin_y  \left(\sum_{i: (i,j) \in \Omega} L_{ij}(x^k_i y, A_{ij}) + \tilde r(y)\right)$
	\EndFor
\EndFor
\end{algorithmic}
\end{algorithm}

\paragraph{Parallelization.}
Alternating minimization parallelizes naturally over examples and features.
In Algorithm \ref{alg-am},
the loops over $i=1,\ldots,N$ and over $j=1,\ldots,M$ may both be executed in parallel.

\subsection{Early stopping}

It is not very useful to spend a lot of effort optimizing over $X$
before we have a good estimate for $Y$.
If an iterative algorithm is used to compute the minimum over $X$, 
it may make sense to stop the optimization over $X$ early
before going on to update $Y$.
In general, we may consider replacing the minimization over
$x$ and $y$ above by any update rule that moves towards the minimum.
This templated algorithm is presented as Algorithm \ref{alg-update}.
Empirically, we find that this approach often finds a better local minimum than
performing a full optimization over each factor in every iteration,
in addition to saving computational effort on each iteration.

\begin{algorithm}
\caption{\label{alg-update}}
\begin{algorithmic}
\State \textbf{given} $X^0$, $Y^0$
\For{$t=1,2,\ldots$}
  \For{$i=1,\ldots,m$}
      \State $x_i^t = \update_{L,r}(x_i^{t-1},Y^{t-1},A)$
  \EndFor
  \For{$j=1,\ldots,n$}
      \State $y_j^t = \update_{L,\tilde r}(y_j^{(t-1)T},X^{(t)T},A^T)$ 
  \EndFor
\EndFor
\end{algorithmic}
\end{algorithm}

We describe below a number of different update rules $\update_{L,r}$
by writing the $X$ update.
The $Y$ update can be implemented similarly. (In fact, it can be implemented
by substituting $\tilde r$ for $r$, switching the roles of $X$ and $Y$,
and transposing all matrix arguments.)
All of the approaches outlined below can still be executed in parallel 
over examples (for the $X$ update) and features (for the $Y$ update).

\paragraph{Gradient method.}
For example, we might take just one gradient step on the objective.
This method can be used as long as $L$, $r$, and $\tilde r$ do not take infinite values.
(If any of these functions $f$ is not differentiable, 
replace $\nabla f$ below by any subgradient of $f$ \cite{borwein2010,boyd2003subgradient}.)

We implement $\update_{L,r}$ as follows.
Let 
\[
g = \sum_{j: (i,j) \in \Omega} \nabla L_{ij}(x_i y_j, A_{ij}) y_j + \nabla r(x_i).
\]
Then set
\[
x_i^t = x_i^{t-1} - \alpha_t g,
\] 
for some step size $\alpha_t$.
For example, a common step size rule is $\alpha_t = 1/t$, 
which guarantees convergence to the globally
optimal $X$ if $Y$ is fixed \cite{borwein2010, boyd2003subgradient}.

\paragraph{Proximal gradient method.} \label{proxgradmethod}
If a function takes on the value $\infty$, it need
not have a subgradient at that point,
which limits the gradient update to cases where
the regularizer and loss are (finite) real-valued.
When the regularizer (but not the loss) takes on infinite values
(say, to represent a hard constraint),
we can use a proximal gradient method instead.

The proximal operator of a function $f$ \cite{parikh2013} is
\[
\prox_f(z) = \argmin _x (f(x) + \frac{1}{2}\|x-z\|_2^2).
\]
If $f$ is the indicator function of a set $\mathcal C$, 
the proximal operator of $f$ is just (Euclidean)
projection onto $\mathcal C$.

A proximal gradient update $\update_{L,r}$ is implemented as follows.
Let 
\[
g = \sum_{j: (i,j) \in \Omega} \nabla L_{ij}(x^{t-1}_i y^{t-1}_j, A_{ij}) y^{t-1}_j.
\]
Then set
\[
x_i^t = \prox_{\alpha_t r} (x_i^{t-1} - \alpha_t g),
\] 
for some step size $\alpha_t$.
The step size rule $\alpha_t = 1/t$ guarantees convergence to the globally
optimal $X$ if $Y$ is fixed, while
using a fixed, but sufficiently small, step size $\alpha$ guarantees
convergence to a small $O(\alpha)$ neighborhood around the optimum \cite{bertsekas2011}.
The technical condition required on the step size
is that $\alpha_t < 1/L$, where $L$ is the 
Lipshitz constant of the gradient of the objective function.
Bolte et al.\ have shown that the iterates $x_i^t$ and $y_j^t$ 
produced by the proximal gradient update rule
(which they call proximal alternating linearized minimization, or PALM)
globally converge to a critical point of the objective function 
under very mild conditions on $L$, $r$, and $\tilde r$ \cite{bolte2013}.

\paragraph{Prox-prox method.}
Letting $f_t(X) = \sum_{(i,j) \in \Omega} L_{ij}(x_i y^{t}_j, A_{ij})$, 
define the proximal-proximal (prox-prox) update
\[
X^{t+1} = \prox_{\alpha_t r}(\prox_{\alpha_t {f_t}}(X^t)).
\] 

The prox-prox update is simply a proximal gradient step on the objective
when $f$ is replaced by the \emph{Moreau envelope} of $f$,
\[
M_f(X) = \inf_{X'} \left( f(X')+ \|X-X'\|_F^2\right).
\]
(See \cite{parikh2013} for details.)
The Moreau envelope has the same minimizers as the original objective.
Thus, just as the proximal gradient method repeatedly applied to $X$
converges to global minimum of the objective if $Y$ is fixed,
the prox-prox method repeatedly applied to $X$ also
converges to global minimum of the objective if $Y$ is fixed
under the same conditions on the step size $\alpha_t$.
for any constant stepsize $\alpha \leq \|G\|_2^2$.
(Here, $\|G\|_2 = \sup_{\|x\|_2 \leq 1} \|Gx\|_2$ is the operator norm of $G$.)

This update can also be seen as a single iteration of ADMM when the dual variable in ADMM 
is initialized to 0; see \cite{boyd2011}. 
In the case of quadratic objectives, we will see below that the prox-prox update 
can be applied very efficiently,
making iterated prox-prox, or ADMM, effective means of computing the solution to the subproblems
arising in alternating minimization.

\paragraph{Choosing a step size.}
In numerical experiments, we find that using a slightly more nuanced rule
allowing \emph{different} step sizes for different rows and columns
can allow fast progress towards convergence
while ensuring that the value of the objective never increases.
The safeguards on step sizes we propose are quite important in practice:
without these checks, we observe \emph{divergence} when the initial step
sizes are chosen too large.

Motivated by the convergence proof in \cite{bertsekas2011},
for each \emph{row}, we seek a step size on the order of $1/\|g_i\|_2$,
where $g_i$ is the gradient of the objective function with respect to $x_i$.
We start by choosing an initial step size scale $\alpha_i$ for each row
of the same order as the average gradient of 
the loss functions for that row.
In the numerical experiments reported here, we choose $\alpha_i = 1$ for $i=1,\ldots,m$.
Since $g_i$ grows with the number of observations 
$n_i = |\{j: (i,j) \in \Omega\}|$ in row $i$, we achieve the desired scaling by setting
$\alpha_i^0 = \alpha_i / n_i$.
We take a gradient step on each row $x_i$ using the step size $\alpha_i$.
Our procedure for choosing $\alpha_j^0$ is the same.

We then check whether the objective value for the row,
\[
\sum_{j: (i,j) \in \Omega} L_{j}(x_i y_j, A_{ij}) + \gamma\|x_i\|_2^2,
\] 
has increased or decreased.
If it has increased, then we trust our first order approximation to the objective function
less far, and reduce the step size; 
if it has decreased, we gain confidence, and increase the step size. 
In the numerical experiments reported below, 
we decrease the step size by $30\%$ when the objective increases,
and increase the step size by $5\%$ when the objective decreases.
This check stabilizes the algorithm and prevents divergence even when the initial
scale has been chosen poorly.

We then do the same with respect to each column $y_j$: we take a
gradient step,
check if the objective value for the column has increased or decreased,
and adjust the step size. 

The time per iteration is thus $O(k(m + n + |\Omega|))$:
computing the gradient of the $i$th loss function with respect to $x_i$ takes time $O(k n_i)$;
computing the proximal operator of the square loss takes time $O(k)$;
summing these over all the rows $i=1,\ldots,m$ gives time $O(k(m + |\Omega|))$;
and adding the same costs for the column updates gives time $O(k(m + n + |\Omega|))$.
The checks on the objective value take time $O(k)$ per observation
(to compute the inner product $x_i y_j$ and value of the loss function for each observation)
and time $O(1)$ per row and column to compute the value of the regularizer.
Hence the total time per iteration is $O(k(m + n + |\Omega|))$.

By partitioning the job of updating different rows and different columns onto different processors, 
we can achieve an iteration time of $O(k(m + n + |\Omega|)/p)$ using $p$ processors.

\paragraph{Stochastic gradients.}
Instead of computing the full gradient of $L$ with respect to $x_i$ above,
we can replace the gradient $g$ in either the gradient or proximal gradient method
by any \emph{stochastic gradient} $g$, which is a vector that satisfies
\[
\Expect g =  \sum_{j: (i,j) \in \Omega} \nabla L_{ij}(x_i y_j, A_{ij}) y_j.
\]
A stochastic gradient can be computed by 
sampling $j$ uniformly at random from among observed features of $i$,
and setting $g = |\{j: (i,j) \in \Omega\}| \nabla L_{ij}(x_i y_j, A_{ij}) y_j$.
More samples from $\{j: (i,j) \in \Omega\}$ can be used to compute a less noisy stochastic gradient.

\subsection{Quadratic objectives}\label{s-alg-qp}
Here we describe how to efficiently implement the prox-prox update rule 
for quadratic objectives and arbitrary regularizers,
extending the factorization caching technique introduced in \S\ref{s-caching}.
We assume here that the objective is given by
\[
\|A-XY\|_F^2 + r(X) + \tilde r(Y).
\]
We will concentrate here on the $X$ update; 
as always, the $Y$ update is exactly analogous.

As in the case of quadratic regularization, 
we first form the Gram matrix $G=Y Y^T$.
Then the proximal gradient update is fast to evaluate:
\[
\prox_{\alpha_k r}(X - \alpha_k(XG - 2AY^T)).
\]

But we can take advantage of the ease of inverting the Gram matrix $G$ to design
a faster algorithm using the prox-prox update.
For quadratic objectives with Gram matrix $G = Y^T Y$,
the prox-prox update takes the simple form
\[
\prox_{\alpha_k r}((G+\frac{1}{\alpha_k}I)^{-1}(AY^T+\frac{1}{\alpha_k}X)).
\]
As in \S\ref{s-caching}, we can compute $(G+\frac{1}{\alpha_k}I)^{-1}(AY^T+\frac{1}{\alpha_k}X)$
in parallel by first caching the factorization of $(G+\frac{1}{\alpha_k}I)^{-1}$.
Hence it is advantageous to repeat this update many times before updating $Y$, 
since most of the computational effort is in forming $G$ and $AY^T$.

For example, in the case of nonnegative least squares, this update is just
\[
\Pi_+((G+\frac{1}{\alpha_k}I)^{-1}(AY^T+\frac{1}{\alpha_k}X)),
\]
where $\Pi_+$ projects its argument onto the nonnegative orthant.

\subsection{Convergence}\label{s-convergence}

Alternating minimization need not converge to the same model 
(or the same objective value)
when initialized at different starting points.
Through examples, we explore this idea here.
These examples are fit using the serial Julia implementation (presented in \S\ref{s-implementation})
of the alternating proximal gradient updates method.

\paragraph{Global convergence for quadratically regularized PCA.}
Figure~\ref{f-qpca-convergence} shows the convergence of the alternating proximal gradient update method 
on a quadratically regularized PCA problem 
with randomly generated, fully observed data $A = X^{\mathrm{true}} Y^{\mathrm{true}}$, 
where entries of $X^{\mathrm{true}}$ and $Y^{\mathrm{true}}$ are
drawn from a standard normal distribution. 
We pick five different random initializations of $X$ and $Y$ with standard normal entries to 
generate five different convergence trajectories.
Quadratically regularized PCA is a simple problem with an analytical solution (see \S\ref{s-qpca}),
and with no local minima (see Appendix~\ref{a-qpca}).
Hence it should come as no surprise that the trajectories all converge to the same,
globally optimal value.

\paragraph{Local convergence for nonnegative matrix factorization.}
Figure~\ref{f-nnmf-convergence} shows convergence of the same algorithm on 
a nonnegative matrix factorization model, with data generated in the same way as in Figure~\ref{f-qpca-convergence}.
(Note that $A$ has some negative entries, so the minimal objective value is strictly greater than zero.)
Here, we plot the convergence of the objective value, rather than the suboptimality, since we cannot 
provably compute the global minimum of the objective function.
We see that the algorithm converges to a different optimal value (and point) depending
on the initialization of $X$ and $Y$.
Three trajectories converge to the same optimal value (though one does so much faster than the others), 
one to a value that is somewhat better, and one to a value that is substantially worse.

\begin{figure}[htb!]
\begin{centering}
\begin{tikzpicture}[]
\begin{semilogyaxis}[view = {0}{90}, width = 5in, xmin = 0, xmax = 3, xtick={0,1,2,3}, ylabel = objective suboptimality, xlabel = time (s)]
\addplot[color=blue,no markers] coordinates {
(0.0, 162306.28878996446)
(0.2823331356048584, 88515.68451398159)
(0.6157271862030029, 84588.08409960709)
(0.902780294418335, 70886.53599541285)
(1.1303253173828125, 32240.45384143944)
(1.4286322593688965, 2111.5339258811555)
(1.6582612991333008, 130.42049882611704)
(1.937556266784668, 11.902208799352167)
(2.228029251098633, 3.388755576606542)
(2.4776461124420166, 2.7654083835847985)
(2.7177441120147705, 2.7151239590201897)
(2.9989211559295654, 2.7064999619332184)
};
\addplot[color=blue,no markers] coordinates {
(0.0, 162703.21353041055)
(0.24230003356933594, 85917.58345787862)
(0.48273396492004395, 76439.87823056939)
(0.68589186668396, 55434.26023839697)
(0.9246060848236084, 43845.332126883884)
(1.15061616897583, 40546.9133090243)
(1.4445431232452393, 34142.440174292606)
(1.6670122146606445, 19447.842732049507)
(1.9573931694030762, 3640.6461681475885)
(2.2152082920074463, 277.3552946653996)
(2.459270477294922, 23.787700842874543)
(2.687990665435791, 4.592662596539853)
(2.959815740585327, 3.1405922002718825)
};
\addplot[color=blue,no markers] coordinates {
(0.0, 155735.4217170638)
(0.3296689987182617, 87695.78196223431)
(0.588951826095581, 82968.03290522602)
(0.8999238014221191, 69322.28816119903)
(1.1653656959533691, 51199.0395935667)
(1.4196608066558838, 44873.53022183818)
(1.6670517921447754, 35802.39203938894)
(1.9260575771331787, 14852.79417538596)
(2.240365505218506, 1074.808353906841)
(2.498340368270874, 62.62688703946657)
(2.770785331726074, 6.573974391469349)
(3.0034172534942627, 2.9523836964742713)
};
\addplot[color=blue,no markers] coordinates {
(0.0, 166159.02999810255)
(0.21102499961853027, 87208.78526220514)
(0.45676088333129883, 77842.08539254131)
(0.6431748867034912, 59716.39723476941)
(0.8840019702911377, 48265.93949348301)
(1.0804390907287598, 41231.54159718496)
(1.3116631507873535, 22075.441753847856)
(1.5023109912872314, 2230.6429326032257)
(1.7062218189239502, 109.83159590236457)
(1.9477019309997559, 9.891811500970121)
(2.1414389610290527, 3.3609396373858544)
(2.3825738430023193, 2.9107721002333165)
};
\addplot[color=blue,no markers] coordinates {
(0.0, 181001.4571394322)
(0.25560903549194336, 89177.30554432835)
(0.4474220275878906, 82361.55299870433)
(0.6821129322052002, 67017.58150956912)
(0.8890230655670166, 36456.47664636432)
(1.1188480854034424, 7034.1491195651715)
(1.3343689441680908, 309.9931359295911)
(1.587702989578247, 25.074796773495535)
(1.776566982269287, 5.6102933043745224)
(1.9732730388641357, 3.9706505286936107)
(2.199392080307007, 3.8190755676519217)
(2.41235089302063, 3.7987829537278373)
};
\end{semilogyaxis}
\end{tikzpicture}
\caption{\label{f-qpca-convergence}Convergence of alternating proximal gradient updates on quadratically regularized PCA for $n=m=200$, $k=2$.}
\end{centering}
\end{figure}
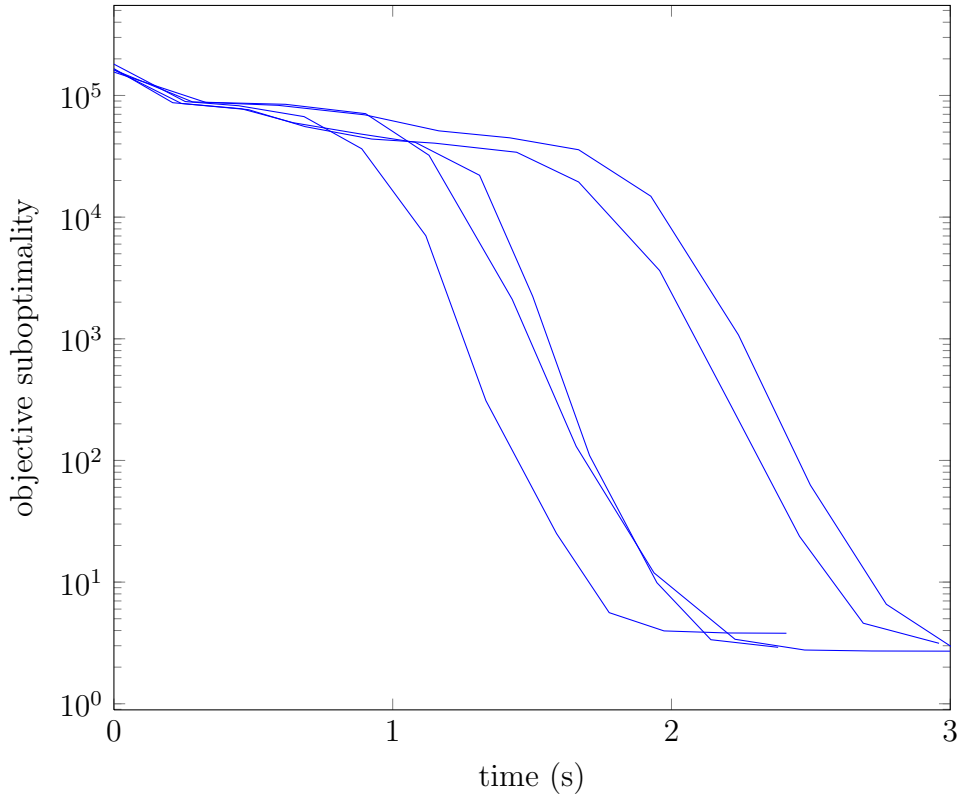

\begin{figure}[htb!]
\begin{centering}
\begin{tikzpicture}[]
\begin{axis}[view = {0}{90}, width = 5in, xmin = 0, xmax = 3, xtick={0,1,2,3}, ylabel = objective value, xlabel = time (s)]\addplot[color=blue,no markers] coordinates {
(0.177001953125, 17774.851098805386)
(0.22897601127624512, 17566.02411020949)
(0.336108922958374, 17379.470166831245)
(0.4035360813140869, 17163.250002883924)
(0.4490702152252197, 16936.62568781356)
(0.5354881286621094, 16738.57088278383)
(0.5710272789001465, 16585.26372526139)
(0.6136364936828613, 16478.336730115105)
(0.721158504486084, 16407.23343378354)
(0.7814385890960693, 16359.43262247778)
(0.8486506938934326, 16317.91366687969)
(0.9527637958526611, 16276.874571575665)
(1.021460771560669, 16233.843373905835)
(1.1346766948699951, 16186.104150967847)
(1.1995818614959717, 16132.936440588379)
(1.2672419548034668, 16073.254173504129)
(1.382185935974121, 16005.876888952358)
(1.4495458602905273, 15940.39618195999)
(1.4968218803405762, 15878.130567844502)
(1.5894479751586914, 15826.761074625396)
(1.656451940536499, 15781.469435885998)
(1.7219388484954834, 15740.27845125789)
(1.8313188552856445, 15702.40236164603)
(1.895042896270752, 15669.204437603)
(1.999344825744629, 15641.342959956506)
(2.062774896621704, 15619.276950048023)
(2.1098108291625977, 15601.957021988288)
(2.1922929286956787, 15588.581441080261)
(2.2277588844299316, 15578.163109137105)
(2.290457010269165, 15569.281883286847)
};
\addplot[color=blue,no markers] coordinates {
(0.0986940860748291, 17865.753511084054)
(0.1447889804840088, 17423.084470381753)
(0.2353379726409912, 16866.34661510219)
(0.2949528694152832, 16184.57059335689)
(0.34473681449890137, 15575.684004544142)
(0.42554283142089844, 15099.811976403196)
(0.460979700088501, 14610.339914094748)
(0.49863767623901367, 14007.102864582517)
(0.578284740447998, 13433.807949450978)
(0.6168358325958252, 13099.144785979386)
(0.6639456748962402, 12973.729303839296)
(0.7408847808837891, 12930.620257890829)
(0.7767188549041748, 12918.162210028055)
(0.8525447845458984, 12913.781172219533)
};
\addplot[color=blue,no markers] coordinates {
(0.039396047592163086, 17952.91633659547)
(0.07667303085327148, 17664.185371402225)
(0.1548628807067871, 17317.165655704575)
(0.1915607452392578, 16803.185123223673)
(0.227294921875, 16167.180766123483)
(0.30399513244628906, 15573.46454441867)
(0.3416931629180908, 15137.027101541404)
(0.4229881763458252, 14797.051701696519)
(0.4705312252044678, 14525.972655271675)
(0.5283610820770264, 14332.225604091756)
(0.6604430675506592, 14193.725327151495)
(0.7135109901428223, 14083.171363173411)
(0.763084888458252, 13987.132126509805)
(0.848660945892334, 13897.442440338475)
(0.8850898742675781, 13810.76123462777)
(0.9898698329925537, 13723.304527857668)
(1.038254737854004, 13635.546030828373)
(1.0762057304382324, 13553.313855531864)
(1.1671257019042969, 13480.702394630378)
(1.2122547626495361, 13423.70913633765)
(1.258260726928711, 13384.361296849205)
(1.356076955795288, 13353.389782157297)
(1.3946897983551025, 13326.954183591692)
(1.4353117942810059, 13304.30633782066)
(1.537336826324463, 13284.069042319987)
(1.6010937690734863, 13265.824149117594)
(1.7030668258666992, 13249.145175941994)
(1.740738868713379, 13232.626483444561)
(1.7763350009918213, 13216.38300301449)
(1.8521559238433838, 13200.545445423799)
(1.908823013305664, 13184.901052215238)
(1.960710048675537, 13169.473447251316)
(2.043921947479248, 13154.396639513054)
(2.098723888397217, 13139.760096895541)
(2.1629838943481445, 13125.58749741501)
(2.268411874771118, 13111.946810562473)
(2.3039920330047607, 13098.861175144633)
(2.3907179832458496, 13086.743922464368)
(2.426158905029297, 13075.51218163933)
(2.4730780124664307, 13065.039012843332)
(2.564666986465454, 13055.393315758649)
};
\addplot[color=blue,no markers] coordinates {
(0.054522037506103516, 18034.90848268617)
(0.11849594116210938, 17888.509080727872)
(0.21422791481018066, 17752.93540679853)
(0.27085089683532715, 17524.33643969415)
(0.3587329387664795, 17138.970159072072)
(0.41502881050109863, 16607.406515960676)
(0.46277689933776855, 16030.170961959131)
(0.5560858249664307, 15503.462267312167)
(0.6108589172363281, 15033.800941192361)
(0.6556880474090576, 14606.440682661387)
(0.7639269828796387, 14262.916896449751)
(0.8183999061584473, 13987.235850025549)
(0.8866069316864014, 13802.259369970358)
(0.9735610485076904, 13668.027821749778)
(1.0262839794158936, 13555.268722958137)
(1.1195499897003174, 13462.628826095903)
(1.1619839668273926, 13403.143191993233)
(1.2226178646087646, 13363.26099584446)
(1.3121068477630615, 13333.057760510335)
(1.3499937057495117, 13308.069664337954)
(1.4190287590026855, 13286.052387967375)
(1.5275325775146484, 13266.392958380658)
(1.5924205780029297, 13248.602194013436)
(1.684483528137207, 13231.50922007214)
(1.7480535507202148, 13214.61718850457)
(1.8115825653076172, 13198.03020515489)
(1.907212495803833, 13181.498110127935)
(1.95233154296875, 13165.129179210608)
(1.9943156242370605, 13149.153763471882)
(2.0716936588287354, 13133.661601406387)
(2.1232316493988037, 13118.704047723475)
(2.180736541748047, 13104.360457009581)
(2.2584924697875977, 13090.770809683116)
(2.294142484664917, 13078.382218459403)
(2.3695504665374756, 13066.912549545068)
(2.4057254791259766, 13056.318386582238)
(2.441619634628296, 13046.88867117435)
};
\addplot[color=blue,no markers] coordinates {
(0.1246180534362793, 17953.37252682276)
(0.1953449249267578, 17548.683029542106)
(0.25040292739868164, 17019.441271668773)
(0.35127997398376465, 16266.746308244235)
(0.4113650321960449, 15516.209854011919)
(0.5144181251525879, 14903.449906466825)
(0.5749311447143555, 14358.902878848847)
(0.6294541358947754, 13874.987378627275)
(0.7373471260070801, 13533.116956709415)
(0.7918341159820557, 13334.981763373964)
(0.8525021076202393, 13238.565365169241)
(0.9635031223297119, 13197.217901975931)
(1.024360179901123, 13175.84073277369)
(1.1004912853240967, 13164.583255532427)
(1.2004432678222656, 13159.462343532541)
};
\end{axis}

\end{tikzpicture}
\caption{\label{f-nnmf-convergence}Convergence of alternating proximal gradient updates on NNMF for $n=m=200$, $k=2$.}
\end{centering}
\end{figure}
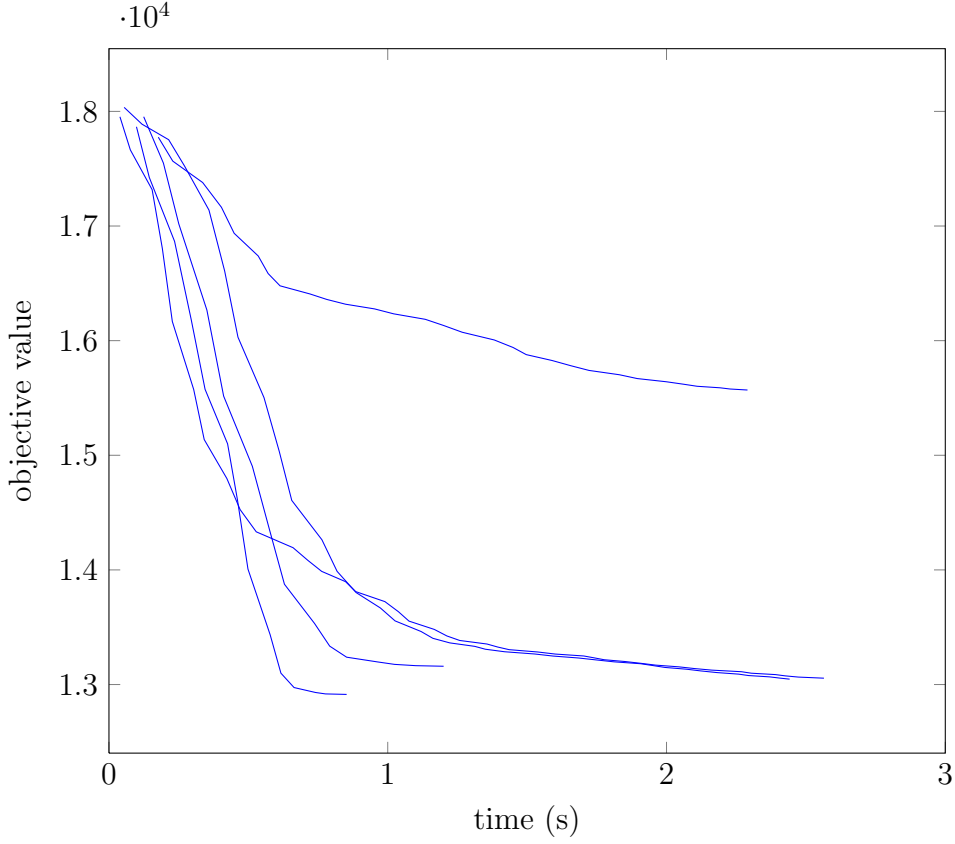

\subsection{Initialization}\label{s-initialization}
Alternating minimization need not converge to the same solution 
(or the same objective value)
when initialized at different starting points.
Above, we saw that alternating minimization can converge to models
with optimal values that differ significantly.

Here, we discuss two approaches to initialization
that result in provably good solutions, for special cases of the generalized low rank problem.
We then discuss how to apply these initialization schemes to more general models.

\paragraph{SVD.}
A literature that is by now extensive shows that the SVD provides 
a provably good initialization for the quadratic matrix completion problem (\ref{eq-mc})
\cite{keshavan2009, keshavan2010, keshavan2010regularization, jain2013, hardt2013, gunasekar2013}.
Algorithms based on alternating minimization have been shown to converge quickly 
(even geometrically \cite{jain2013}) 
to a global solution satisfying a recovery guarantee when the initial values 
of $X$ and $Y$ are chosen carefully; see \S\ref{s-mc} for more details.

Here, we extend the SVD initialization previously proposed for matrix completion
to one that works well for all PCA-like problems:
problems with convex loss functions that have been scaled as in \S\ref{s-gpca-scaling};
with data $A$ that consists of real values, Booleans, categoricals, and ordinals;
and with quadratic (or no) regularization.

But we will need a matrix on which to perform the SVD. 
What matrix corresponds to our data table?
Here, we give a simple proposal for how to construct such a matrix,
motivated by \cite{keshavan2010,jain2013,chatterjee2014}.
Our key insight is that the SVD is the solution to our problem when
the entries in the table have mean zero and variance one
(and all the loss functions are quadratic).
Our initialization will construct a matrix with mean zero and variance one
from the data table, take its SVD,
and invert the construction to produce the correct initialization.

Our first step is to expand the categorical columns taking on $d$ values 
into $d$ Boolean columns, and to re-interpret ordinal and Boolean columns as numbers.
The scaling we propose below is insensitive to the values of the numbers in the expansion
of the Booleans: for example, using (false, true)$=(0, 1)$ or (false, true)$=(-1, 1)$
produces the same initialization.
The scaling \emph{is} sensitive to the differences between ordinal values: while 
encoding (never, sometimes, always) as $(1, 2, 3)$ or as $(-5, 0, 5)$ will make no difference,
encoding these ordinals as $(0,1,10)$ \emph{will} result in a different initialization.

Now we assume
that the \emph{rows} of the data table are independent and identically distributed,
so they each have equal means and variances. 
Our mission is to standardize the \emph{columns}.
The observed entries in column $j$ have mean $\mu_j$ and variance $\sigma_j^2$,
\BEAS
\mu_j = \argmin_\mu \sum_{i: (i,j) \in \Omega} L_j(\mu, A_{ij}) \\
\sigma^2_j = \frac{1}{n_j - 1} \sum_{i: (i,j) \in \Omega} L_j(\mu_j, A_{ij}),
\EEAS
so the matrix whose $(i,j)$th entry is $(A_{ij} - \mu_j) / \sigma_j$ for $(i,j) \in \Omega$
has columns whose observed entries have mean 0 and variance 1.

Each missing entry can be safely replaced with 0
in the scaled version of the data without changing the column mean.
But the column \emph{variance} will decrease to $m_j/m$.
If instead we define
\[
\tilde A_{ij} = \left\{\ba{ll}
\frac{m}{\sigma_j m_j}(A_{ij} - \mu_j) & (i,j) \in \Omega \\
0 & \mbox{otherwise},
\ea \right.
\]
then the column will have mean 0 and variance 1.

Take the SVD $U \Sigma V^T$ of $\tilde A$, 
and let $\tilde U \in \reals^{m \times k}$, 
$\tilde \Sigma \in \reals^{k \times k}$,
and $\tilde V \in \reals^{n \times k}$ 
denote these matrices truncated to the top $k$ singular vectors and values.
We initialize $X = \tilde U \tilde \Sigma^{1/2}$, and 
$Y = \tilde \Sigma^{1/2} \tilde V^T \diag(\sigma)$.
The offset row in the model is initialized with the means, \ie, 
the $k$th column of $X$ is filled with 1's, and the $k$th row of $Y$ is filled with the means,
so $Y_{kj} = \mu_j$.


Finally, note that we need not compute the full SVD of $\tilde A$, but 
instead can simply compute the top $k$ singular triples.
For example, the randomized top $k$ SVD algorithm proposed in \cite{halko2011}
computes the top $k$ singular triples
of $\tilde A$ in time linear in $|\Omega|$, $m$, and $n$ (and quadratic in $k$).

Figure \ref{f-svd-convergence} compares the convergence of this SVD-based initialization
with random initialization on a low rank model for census data
described in detail in \S\ref{s-census}.
We initialize the algorithm at six different points: 
from five different random normal initializations
(entries of $X^0$ and $Y^0$ drawn iid from $\mathcal N(0,1)$),
and from the SVD of $\tilde A$.
The SVD initialization produces a better initial value for the objective function,
and also allows the algorithm to converge to a substantially lower final objective value
than can be found from \emph{any} of the five random starting points.
This behaviour indicates that the ``good'' local minimum discovered by the SVD
initialization is located in a basin of attraction that has low probability
with respect to the measure induced by random normal initialization.


\begin{figure}
\begin{center}
\begin{tikzpicture}[]
\begin{axis}[
		view = {0}{90}, 
		width=5in,
		xlabel=iteration,
		ylabel=objective value
	]
	\addplot[color=blue,no markers] coordinates {
(2.0, 873729.888194348)
(3.0, 829533.3136935838)
(4.0, 790819.871871953)
(5.0, 654909.7103520385)
(6.0, 487049.84848666174)
(7.0, 450916.8658144163)
(8.0, 417492.33511719946)
(9.0, 396128.8977446231)
(10.0, 389802.51199362864)
(11.0, 358839.4157658114)
(12.0, 355701.85878376185)
(13.0, 349679.71614544804)
(14.0, 345853.95336203405)
(15.0, 341893.3347655441)
(16.0, 340202.91638392943)
(17.0, 339262.41528693)
(18.0, 338040.3233988243)
(19.0, 336743.82688709884)
(20.0, 336700.2378720764)
(21.0, 336391.00880781905)
(22.0, 335218.34361301665)
(23.0, 334202.5661925412)
(24.0, 333765.98348767404)
(25.0, 333449.05999378714)
(26.0, 333247.5652059346)
(27.0, 332861.1630904935)
(28.0, 332788.35349440126)
(29.0, 332309.35957483895)
(30.0, 332213.8923786297)
(31.0, 331933.1938150223)
(32.0, 331765.4500437346)
(33.0, 331680.7253728598)
(34.0, 331680.7253728598)
};
\addlegendentry{random}
\addplot[color=blue,no markers] coordinates {
(2.0, 597253.8461498419)
(3.0, 419400.39729697385)
(4.0, 398470.35164651455)
(5.0, 383805.2268968231)
(6.0, 370170.69442871405)
(7.0, 367747.5550477448)
(8.0, 357459.74774598074)
(9.0, 354682.7131104882)
(10.0, 351208.60385651357)
(11.0, 348822.9519831141)
(12.0, 345042.4592163142)
(13.0, 342677.76105293806)
(14.0, 340629.9879194479)
(15.0, 338927.5121115017)
(16.0, 337973.33552107506)
(17.0, 337708.2788673361)
(18.0, 337195.7191872192)
(19.0, 336537.8478791696)
(20.0, 334769.35103221197)
(21.0, 332726.38057623943)
(22.0, 332623.884907817)
(23.0, 332323.09878918267)
(24.0, 330905.37106069585)
(25.0, 329937.98593887285)
(26.0, 329857.9456838864)
(27.0, 329410.4648749505)
(28.0, 328845.28896487685)
(29.0, 328732.5002903112)
(30.0, 328679.0248016629)
(31.0, 328532.1333770791)
(32.0, 328517.2453608084)
(33.0, 328372.8119698268)
(34.0, 327997.47346258216)
(35.0, 327741.9602940207)
(36.0, 327741.9602940207)
};
\addlegendentry{random}
\addplot[color=blue,no markers] coordinates {
(2.0, 576411.5994330856)
(3.0, 495058.6675279004)
(4.0, 431190.17726056685)
(5.0, 416914.0096587776)
(6.0, 405873.60922042676)
(7.0, 384062.3423589405)
(8.0, 356589.5437706274)
(9.0, 343075.73672896175)
(10.0, 341281.1165672342)
(11.0, 336332.92717524106)
(12.0, 332805.72634765104)
(13.0, 330162.1962212119)
(14.0, 328590.69682586315)
(15.0, 327073.17762617243)
(16.0, 326670.2913582652)
(17.0, 326012.1519455335)
(18.0, 325445.22035664227)
(19.0, 324966.24657695775)
(20.0, 324303.05407749827)
(21.0, 324246.14853914385)
(22.0, 324113.93110497)
(23.0, 323528.15745377267)
(24.0, 323214.1500636815)
(25.0, 323016.8898957058)
(26.0, 322658.2142292026)
(27.0, 322414.6945093508)
(28.0, 322082.36276016134)
(29.0, 321928.28429484623)
(30.0, 321780.2665071678)
(31.0, 321737.8206636559)
(32.0, 321465.0197656342)
(33.0, 321180.60638211784)
(34.0, 320970.14808292966)
(35.0, 320849.51340309303)
(36.0, 320849.51340309303)
};
\addlegendentry{random}
\addplot[color=blue,no markers] coordinates {
(2.0, 594710.3810784162)
(3.0, 510209.17144939327)
(4.0, 452631.9752830471)
(5.0, 448286.98160747276)
(6.0, 367602.80044831673)
(7.0, 350734.55335020606)
(8.0, 343795.35554982215)
(9.0, 343158.02004275593)
(10.0, 340122.64760387136)
(11.0, 339206.1853153142)
(12.0, 338931.17206849885)
(13.0, 335432.47435988707)
(14.0, 333386.6311611552)
(15.0, 326693.2621513337)
(16.0, 326512.17008019774)
(17.0, 325528.5603597335)
(18.0, 324846.27061073727)
(19.0, 322312.41073414544)
(20.0, 320614.9988541134)
(21.0, 318964.1674170207)
(22.0, 318850.585547303)
(23.0, 317646.0789413322)
(24.0, 317049.0693890328)
(25.0, 316836.37699606334)
(26.0, 316726.7563271608)
(27.0, 316171.95337234903)
(28.0, 315935.9862121451)
(29.0, 315812.0532946038)
(30.0, 315619.7553829763)
(31.0, 315439.6533439822)
(32.0, 315325.50360445393)
(33.0, 315297.8559614159)
(34.0, 315163.1887286475)
(35.0, 315163.1887286475)
};
\addlegendentry{random}
\addplot[color=blue,no markers] coordinates {
(2.0, 498772.8011211088)
(3.0, 419331.66706965974)
(4.0, 401464.1952900847)
(5.0, 389268.20930328936)
(6.0, 358260.09543399676)
(7.0, 352952.18099127605)
(8.0, 347132.5684779385)
(9.0, 346146.14687200915)
(10.0, 340172.4965664208)
(11.0, 337221.47873228043)
(12.0, 333480.2447429349)
(13.0, 329993.9600412735)
(14.0, 328796.97459730576)
(15.0, 328221.83549463976)
(16.0, 327222.84913555265)
(17.0, 326004.902628331)
(18.0, 324332.4345587176)
(19.0, 323135.8136569273)
(20.0, 322576.896227519)
(21.0, 321455.21066799876)
(22.0, 320704.8193688907)
(23.0, 320627.4888218349)
(24.0, 320218.9181408711)
(25.0, 320042.63528939907)
(26.0, 319675.1364988295)
(27.0, 319444.2823565937)
(28.0, 319210.22984718834)
(29.0, 319208.14492769947)
(30.0, 319068.46607704414)
(31.0, 318866.29466919426)
(32.0, 318828.73435844644)
(33.0, 318700.9421358182)
(34.0, 318629.01681062696)
(35.0, 318629.01681062696)
};
\addlegendentry{random}
\addplot[color=red,mark=x] coordinates {
(2.0, 244275.09163092592)
(3.0, 242148.4737869167)
(4.0, 240408.96539358806)
(5.0, 238864.38967507993)
(6.0, 237445.46912163793)
(7.0, 236273.47373403676)
(8.0, 235213.01759613975)
(9.0, 234152.17456335263)
(10.0, 233115.78153824137)
(11.0, 231999.14212798656)
(12.0, 231015.93484650613)
(13.0, 229606.3890291173)
(14.0, 228402.22705669003)
(15.0, 227311.93579723846)
(16.0, 225805.34488822435)
(17.0, 224382.32921205467)
(18.0, 222997.10726116505)
(19.0, 220725.53431488734)
(20.0, 219417.11086095416)
(21.0, 219266.47182962383)
(22.0, 218481.1715962479)
(23.0, 217790.22479193445)
(24.0, 217183.6676200531)
(25.0, 216850.74380838996)
(26.0, 216653.57031366669)
(27.0, 216405.56387848503)
(28.0, 215892.89091867718)
(29.0, 215650.00337761638)
(30.0, 215606.18108260122)
(31.0, 215169.7200061441)
(32.0, 214963.4051105464)
(33.0, 214770.6653782049)
(34.0, 214389.03818068697)
(35.0, 214299.26451170802)
(36.0, 214005.64027861634)
(37.0, 213878.63320030703)
(38.0, 213749.04801865047)
(39.0, 213601.39338842558)
(40.0, 213547.83779308852)
(41.0, 213273.1137467881)
(42.0, 212960.13253841278)
(43.0, 212798.43768531622)
(44.0, 212684.87478676662)
(45.0, 212540.00310099075)
(46.0, 212540.00310099075)
};
\addlegendentry{SVD}
\end{axis}

\end{tikzpicture}
\caption{\label{f-svd-convergence}Convergence from five different random initializations, and from the SVD initialization.}
\end{center}
\end{figure}
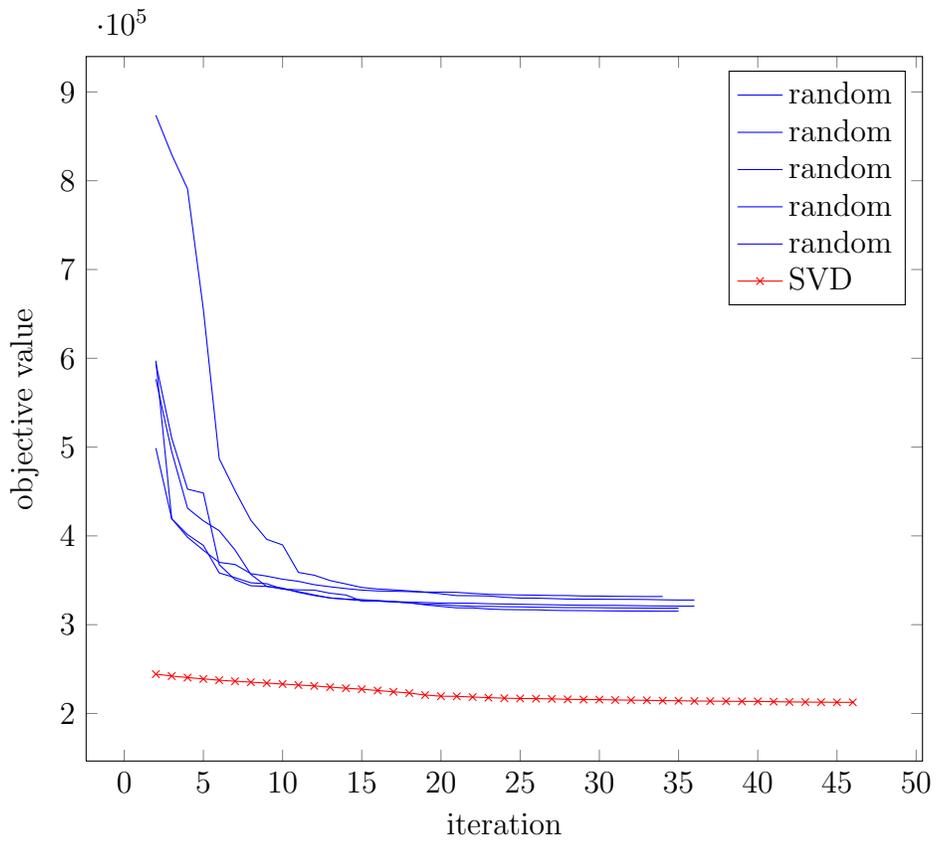

\paragraph{$k$-means$++$.} 
The $k$-means$++$ algorithm is an initialization scheme designed for quadratic clustering problems \cite{arthur2007}.
It consists of choosing an initial cluster centroid at random from the points, and 
then choosing the remaining $k-1$ centroids from the points $x$ that have not yet been chosen
with probability proportional to $D(x)^2$, where $D(x)$ is the minimum distance of $x$ 
to any previously chosen centroid. 

Quadratic clustering is known to be NP-hard, even with only two clusters ($k = 2$) \cite{drineas2004}.
However, $k$-means$++$ followed by alternating minimization
gives a solution with expected approximation ratio within $O(\log k)$ of the optimal value \cite{arthur2007}. 
(Here, the expectation is over the randomization in the initialization algorithm.)
In contrast, an arbitrary initialization of the cluster centers
for $k$-means can result in a solution whose value is arbitrarily worse than the true optimum.

A similar idea can be used for other low rank models. 
If the model rewards a solution that is spread out, as is the case in 
quadratic clustering or subspace clustering, 
it may be better to initialize the algorithm by choosing
elements with probability proportional to a distance measure, as in $k$-means$++$.
In the $k$-means$++$ procedure, one can use the loss function $L(u)$  as the distance metric $D$.

\subsection{Global optimality}\label{s-optimality-certificate}

All generalized low rank models are non-convex, but some are more non-convex 
than others. In particular, for some problems, the only important
source of non-convexity is the low rank constraint. For these problems,
it is sometimes possible to certify global optimality of a model by considering
an equivalent rank-constrained convex problem.

The arguments in this section are similar to ones found in \cite{recht2010},
in which Recht et al. propose using a factored (nonconvex) formulation of the 
(convex) nuclear norm regularized estimator in order to efficiently solve
the large-scale SDP arising in a matrix completion problem.
However, the algorithm in \cite{recht2010} relies on a subroutine for
finding a local minimum of an augmented Lagrangian which has the same 
biconvex form as problem~(\ref{eq-mc}).
Finding a local minimum of this problem (rather than a saddle point) may be hard.
In this section, we avoid the issue of finding a local minimum of the nonconvex problem; 
we consider instead whether it is possible to verify global optimality 
when presented with some putative solution.

\paragraph{The factored problem is equivalent to the rank constrained problem.}

Consider the \emph{factored} problem
\BEQ\label{eq-factored}
\begin{array}{ll}
\mbox{minimize} & L(XY) + \frac{\gamma}{2} \|X\|_F^2 + \frac{\gamma}{2} \|Y\|_F^2,
\end{array}
\EEQ
with variables $X \in \reals^{m \times k}$, $Y \in \reals^{k \times n}$,
where $L:\reals^{m \times n} \to \reals$ is any convex loss function.
Compare this to the \emph{rank-constrained} problem
\BEQ\label{eq-rank-constrained}
\ba{ll}
\mbox{minimize} & L(Z) + \gamma \|Z\|_* \\
\mbox{subject to} & \rank(Z) \leq k.
\ea
\EEQ
with variable $Z \in \reals^{m \times n}$.
Here, we use $\|\cdot\|_*$ to denote the nuclear norm,
the sum of the singular values of a matrix.

\begin{theorem}\label{thm-convex-equivalence}
$(X^\star, Y^\star)$ is a solution to the factored problem~\ref{eq-factored}
if and only if 
$Z^\star = X^\star Y^\star$ is a solution to the rank-constrained problem~\ref{eq-rank-constrained},
and $\|X^\star\|_F^2 = \|Y^\star\|_F^2 = \frac{1}{2} \|Z^\star\|_*$.
\end{theorem}

We will need the following lemmas to understand the relation between
the rank-constrained problem and the factored problem.

\begin{lemma}\label{t-nuclear-ub}
Let $XY = U \Sigma V^T$ be the SVD of $XY$, where $\Sigma = \diag(\sigma)$. Then
\BEQ
\|\sigma\|_1 \leq \frac{1}{2}( ||X||_F^2 + ||Y||_F^2 ).
\EEQ
\end{lemma}

\begin{proof}
We may derive this fact as follows:
\BEAS
\|\sigma\|_1 &=& \tr( U^T XY V) \\
&\leq& \|U^T X\|_F \|Y V\|_F \\
&\leq& \|X\|_F \|Y\|_F \\
&\leq& \frac{1}{2}( ||X||_F^2 + ||Y||_F^2 ), \\
\EEAS
where the first inequality above uses the Cauchy-Schwartz inequality,
the second relies on the orthogonal invariance of the Frobenius norm,
and the third follows from the basic inequality $ab \leq \frac{1}{2} (a^2 + b^2)$
for any real numbers $a$ and $b$.
\end{proof}

\begin{lemma}\label{t-nuclear}
For any matrix $Z$, $\|Z\|_* = \inf_{XY = Z} \frac{1}{2}( ||X||_F^2 + ||Y||_F^2 )$.
\end{lemma}

\begin{proof}
Writing $Z = U D V^T$ and recalling the definition of the nuclear norm
$\|Z\|_* = \|\sigma\|_1$, 
we see that Lemma \ref{t-nuclear-ub} implies 
\[
\|Z\|_* \leq \inf_{XY = Z} \frac{1}{2}( ||X||_F^2 + ||Y||_F^2 ).
\]
But taking $X = U \Sigma^{1/2}$ and $Y = \Sigma^{1/2} V^T$, 
we have 
\[
\frac{1}{2}( ||X||_F^2 + ||Y||_F^2 ) = 
\frac{1}{2}( \| \Sigma^{1/2} \|_F ^2 +  \| \Sigma^{1/2} \|_F ^2 ) = \|\sigma\|_1,
\]
(using once again the orthogonal invariance of the Frobenius norm),
so the bound is satisfied with equality.
\end{proof}

Note that the infimum is achieved by
$X = U \Sigma^{1/2} T$ and $Y = T^T \Sigma^{1/2} V^T$
for any orthonormal matrix $T$.

Theorem \ref{thm-convex-equivalence} follows as a corollary, 
since $L(Z) = L(XY)$ so long as $Z = XY$.

\paragraph{The rank constrained problem is sometimes equivalent to an unconstrained problem.}

Note that problem~(\ref{eq-rank-constrained}) is still a hard problem to solve:
it is a rank-constrained semidefinite program.
On the other hand, the same problem without the rank constraint is convex
and tractable (though not easy to solve at scale).
In particular, it is possible to write down an optimality condition for 
the problem
\BEQ\label{eq-rank-unconstrained}
\ba{ll}
\mbox{minimize} & L(Z) + \gamma \|Z\|_*
\ea
\EEQ
that certifies that a matrix $Z$ is globally optimal.
This problem is a relaxation of problem~(\ref{eq-rank-constrained}), and so has an optimal value
that is at least as small.
Furthermore, if any solution to problem~(\ref{eq-rank-unconstrained}) has rank no more than $k$,
then it is feasible for problem~(\ref{eq-rank-constrained}),
so the optimal values of problem~(\ref{eq-rank-unconstrained}) and 
problem~(\ref{eq-rank-constrained})
must be the same.
Hence any solution of problem~(\ref{eq-rank-unconstrained}) with rank no more than $k$
also solves problem~\ref{eq-rank-constrained}.

Recall that the matrix $Z$ is a solution the problem $\mathcal{U}$ if and only if 
\[
0 \in \partial (L(Z) + \gamma \|Z\|_*),
\]
where $\partial f(Z)$ is the \emph{subgradient} of the function $f$ at $Z$.
The subgradient is a \emph{set-valued} function.

The subgradient of the nuclear norm at a matrix $Z = U \Sigma V^T$ is 
any matrix of the form $U V^T + W$ where $U^T W = 0$, $W V = 0$, and $\|W\|_2 \leq 1$.
Equivalently, define the set-valued function $\sign$ on scalar arguments $x$ as 
\[
\sign(x) = \left\{
\ba{ll}
\{1\} & x > 0 \\
{[}-1, 1{]} & x = 0 \\
\{-1\} & x < 0,
\ea
\right . ,
\]
and define $(\sign(x))_i = \sign(x_i)$ for vectors $x \in \reals^n$.
Then we can write the subgradient of the nuclear norm at $Z$ as 
\[
\partial \|Z\|_* = U \diag(\sign(\sigma)) V^T,
\]
where now we use the full SVD of $Z$ with $U \in \reals^{m \times \min(m, n)}$,
$V \in \reals^{n \times \min(m, n)}$, and $\sigma \in \reals^{\min(m,n)}$.

Hence $Z = U \Sigma V^T$ is a solution to problem~(\ref{eq-rank-unconstrained})
if and only if 
\[
0 \in \partial L(Z) + \gamma (U V^T + W),
\]
or more simply, if
\BEQ\label{eq-convex-certificate}
\|(1/\gamma) G + U V^T\|_2 \leq 1
\EEQ
for some $G \in \partial L(Z)$.
In particular, if a matrix $Z$ with rank no more than $k$
satisfies~(\ref{eq-convex-certificate}), then $Z$ also solves
the rank-constrained problem (\ref{eq-rank-constrained}). 

This result allows us to (sometimes) certify global optimality of a particular model.
Given a model $(X, Y)$, we compute the SVD of the product $XY = U \Sigma V^T$
and an element $G \in \partial L(Z)$.
If $\|(1/\gamma) G + U V^T\|_2 \leq 1$, then $(X, Y)$ is globally optimal.
(If the objective is differentiable then we simply pick $G = \nabla L(Z)$;
otherwise some choices of $G \in \partial L(Z)$ may produce invalid certificates
even if $(X, Y)$ is globally optimal.) 

\section{Choosing low rank models}\label{s-choosing}


\subsection{Regularization paths}
Suppose that we wish to understand the entire \emph{regularization path} for a GLRM;
that is, we would like to know the solution $(X(\gamma),Y(\gamma))$ to the problem
\[
\begin{array}{ll}
\mbox{minimize} & \sum_{(i,j) \in \Omega} L_{ij}(x_i y_j, A_{ij}) 
+ \gamma \sum_{i=1}^m r_i(x_i) + \gamma \sum_{j=1}^n \tilde r_j(y_j)
\end{array}
\]
as a function of $\gamma$.
Frequently, the regularization path may be computed almost as quickly as the solution for 
a single value of $\gamma$.
We can achieve this by initially fitting the model with a very high value for $\gamma$, which
is often a very easy problem. 
(For example, when $r$ and $\tilde r$ are norms, the solution is $(X,Y)=(0,0)$
for sufficiently large $\gamma$.)
Then we may fit models corresponding to smaller and smaller values of $\gamma$ by initializing
the alternating minimization algorithm from our previous solution.
This procedure is sometimes called a \emph{homotopy method}.

For example, Figure~\ref{f-regpath} shows the regularization path for 
quadratically regularized Huber PCA on a synthetic data set. 
We generate a dataset $A = XY + S$ with $X\in\reals^{m\times k}$, $Y\in\reals^{k \times n}$,
and $S\in\reals^{m \times n}$,
with $m=n=300$ and $k=3$.
The entries of $X$ and $Y$ are drawn from a standard normal distribution,
while the entries of the sparse noise matrix $S$ are drawn from a uniform distribution
on $[0,1]$ with probability $0.05$, and are $0$ otherwise.
We fit a rank 5 GLRM to an observation set $\Omega$ consisting of 
$10\%$ of the entries in the matrix, drawn uniformly at random from $\{1,\ldots,i\}\times
\{1,\ldots,j\}$,
using Huber loss and quadratic regularization, and vary the regularization parameter.
That is, we fit the model
\[
\begin{array}{ll}
\mbox{minimize} & \sum_{(i,j) \in \Omega} \huber(x_i y_j, A_{ij}) 
+ \gamma \sum_{i=1}^m \|x_i\|_2^2 + \gamma \sum_{j=1}^n \|y_j\|_2^2
\end{array}
\]
and vary the regularization parameter $\gamma$.
The figure plots both the normalized training error, 
\[
\frac{1}{|\Omega|}\sum_{(i,j) \in \Omega} \huber(x_i y_j, A_{ij}),
\]
and the normalized test error, 
\[
\frac{1}{nm - |\Omega|}\sum_{(i,j) \not \in \Omega} \huber(x_i y_j, A_{ij}),
\]
of the fitted model $(X,Y)$, for $\gamma$ ranging from 0 to 3.
Here, we see that while the training error decreases and $\gamma$ decreases, the test
error reaches a minimum around $\gamma = .5$.
Interestingly, it takes only three times longer (about 3 seconds) to generate the entire regularization path
than it does to fit the model for a single value of the regularization parameter (about 1 second).

\begin{figure}[htb!]
\begin{centering}
\begin{tikzpicture}[]
	\begin{axis}[
		view = {0}{90}, 
		width=5in,
		xlabel=$\gamma$,
		ylabel=normalized error
	]
		
	\addplot+ coordinates {
		(3.0, 0.268134158973883)
		(2.9, 0.268134158973883)
		(2.8, 0.268134158973883)
		(2.7, 0.268134158973883)
		(2.6, 0.268134158973883)
		(2.5, 0.268134158973883)
		(2.4000000000000004, 0.268134158973883)
		(2.3000000000000003, 0.268134158973883)
		(2.1999999999999997, 0.268134158973883)
		(2.0999999999999996, 0.268134158973883)
		(2.0, 0.268134158973883)
		(1.9, 0.268134158973883)
		(1.7999999999999998, 0.268134158973883)
		(1.7, 0.268134158973883)
		(1.6, 0.268134158973883)
		(1.5, 0.268134158973883)
		(1.4, 0.268134158973883)
		(1.3, 0.268134158973883)
		(1.2000000000000002, 0.26026322502268284)
		(1.0999999999999999, 0.26026322502268284)
		(1.0, 0.2543554108470851)
		(0.8999999999999999, 0.2543554108470851)
		(0.8, 0.2458180301410701)
		(0.7, 0.23692671977507312)
		(0.6000000000000001, 0.23732936092681242)
		(0.5, 0.23732936092681242)
		(0.4, 0.23732936092681242)
		(0.30000000000000004, 0.3011279841378484)
		(0.2, 0.3187365802787273)
		(0.1, 0.3343240062352943)
		(0.0, 0.35603664941298185)
	};
	\addlegendentry{test error}

	\addplot+ coordinates {
		(3.0, 0.13628214585822127)
		(2.9, 0.13628214585822127)
		(2.8, 0.13628214585822127)
		(2.7, 0.13628214585822127)
		(2.6, 0.13628214585822127)
		(2.5, 0.13628214585822127)
		(2.4000000000000004, 0.13628214585822127)
		(2.3000000000000003, 0.13628214585822127)
		(2.1999999999999997, 0.13628214585822127)
		(2.0999999999999996, 0.13628214585822127)
		(2.0, 0.13628214585822127)
		(1.9, 0.13628214585822127)
		(1.7999999999999998, 0.13628214585822127)
		(1.7, 0.13628214585822127)
		(1.6, 0.13628214585822127)
		(1.5, 0.13628214585822127)
		(1.4, 0.13628214585822127)
		(1.3, 0.13628214585822127)
		(1.2000000000000002, 0.13028904132817043)
		(1.0999999999999999, 0.13028904132817043)
		(1.0, 0.12637108557782675)
		(0.8999999999999999, 0.12637108557782675)
		(0.8, 0.1214163988150702)
		(0.7, 0.11615412250438717)
		(0.6000000000000001, 0.10920985518454013)
		(0.5, 0.10920985518454013)
		(0.4, 0.10920985518454013)
		(0.30000000000000004, 0.089262844411835)
		(0.2, 0.08361863581383462)
		(0.1, 0.08078737619662292)
		(0.0, 0.07894851325763326)
	};
	\addlegendentry{train error}
	\end{axis}
\end{tikzpicture}
\caption{\label{f-regpath}Regularization path.}
\end{centering}
\end{figure}
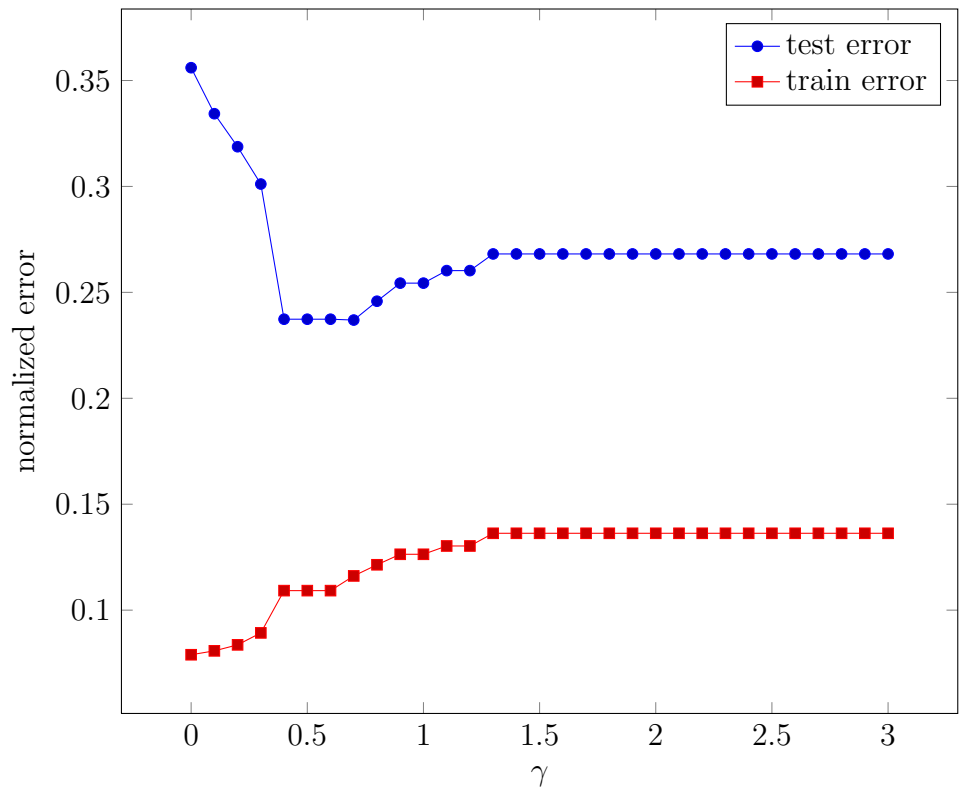

\subsection{Choosing model parameters}

To form a generalized low rank model, one needs to specify the loss functions $L_j$, 
regularizers $r$ and $\tilde r$, and a rank $k$.
The loss function should usually be chosen by a domain expert to reflect the
intuitive notion of what it means to ``fit the data well''.
On the other hand,
the regularizers and rank are often chosen based on statistical considerations,
so that the model generalizes well to unseen (missing) data.

There are three major considerations to balance in choosing the regularization and rank of the model. 
In the following discussion, we suppose that 
the regularizers $r = \gamma r^0$ and $\tilde r = \gamma \tilde r^0$ 
have been chosen up to a scaling $\gamma$.

\paragraph{Compression.} 
A low rank model $(X,Y)$ with rank $k$ and no sparsity represents the data table $A$
with only $(m+n)k$ numbers, achieving a compression ratio of $(m+n)k/(mn)$.
If the factors $X$ or $Y$ are sparse, then we have used fewer than $(m+n)k$ numbers
to represent the data $A$, achieving a higher compression ratio.
We may want to pick parameters of the model ($k$ and $\gamma$) 
in order to achieve a good error $\sum_{(i,j)\in \Omega}L_j(A_{ij} - x_i y_j)$
for a given compression ratio. 
For each possible combination of model parameters, 
we can fit a low rank model with those parameters, observing both 
the error and the compression ratio.
We can then choose the best model parameters (highest compression rate) 
achieving the error we require, or the best model parameters (lowest error rate)
achieving the compression we require. 

More formally, one can construct an information criterion 
for low rank models by analogy with the 
Aikake Information Criterion (AIC) or the 
Bayesian Information Criterion (BIC).
For use in the AIC, 
the number of degrees of freedom in a low rank model can be computed as
the difference between the number of nonzeros in the model and the dimensionality
of the symmetry group of the problem. 
For example, if the model $(X, Y)$ is dense, and 
the regularizer is invariant under 
orthogonal transformations (\eg, $r(x) = \|x\|_2^2$),
then the number of degrees of freedom is $(m+n)k - k^2$ \cite{tipping1999}.
Minka \cite{minka2001} proposes a method based on the BIC
to automatically choose the dimensionality in PCA, and observes
that it performs better than cross validation in identifying the true rank of the model
when the number of observations is small ($m$, $n \lesssim 100$).

\paragraph{Denoising.} 

Suppose we observe every entry in a true data matrix contaminated by noise, \eg, 
$A_{ij} = A^{\mathrm{true}}_{ij} + \epsilon_{ij}$, with $\epsilon_{ij}$ some random variable. 
We may wish to choose model parameters to identify the truth and remove the noise:
we would like to find $k$ and $\gamma$ to minimize 
$\sum_{(i,j)\in \Omega}L_j(A^{\mathrm{true}}_{ij} - x_i y_j)$.

A number of commonly used rules-of-thumb have been proposed in the case of PCA
to distinguish the signal (the true rank $k$ of the data) from the noise, 
some of which can be generalized to other low rank models.
These include using scree plots, often known as the ``elbow method'' \cite{cattell1966}; 
the eigenvalue method;
Horn's parallel analysis \cite{horn1965, dinno2009}; 
and other related methods \cite{zwick1986, preacher2003}.
A recent, more sophisticated method adapts 
the idea of dropout training \cite{srivastava2014dropout} 
to regularize low-rank matrix estimation \cite{josse2014stable}.

Some of these methods can easily be adapted to the GLRM context.
The ``elbow method'' increases $k$ until the objective value decreases less than linearly;
the eigenvalue method increases $k$ until the objective value decreases by less than some threshold;
Horn's parallel analysis increases $k$ until the objective value compares unfavorably
to one generated by fitting a model to data drawn from a synthetic noise distribution.

Cross validation is also simple to apply, and is discussed further below as a means of
predicting missing entries.
However, applying cross validation to the denoising problem is somewhat tricky, since
leaving out too few entries results in overfitting to the noise,
while leaving out too many results in underfitting to the signal.
The optimal number of entries to leave out may depend on the aspect
ratio of the data, as well as on the type of noise present in the data \cite{perry2009}, 
and is not well understood except in the case of Gaussian noise \cite{owen2009}.
We explore the problem of choosing a holdout size numerically below.

\paragraph{Predicting missing entries.} 
Suppose we observe some entries in the matrix and wish to predict the others. 
A GLRM with a higher rank will always be able to fit the (noisy) data better than one 
of lower rank.
However, a model with many parameters may also overfit to the noise.
Similarly, a GLRM with no regularization ($\gamma=0$) will always produce a model
with a lower empirical loss $\sum_{(i,j) \in \Omega} L_j(x_iy_j, A_{ij})$.
Hence, we cannot pick a rank $k$ or regularization $\gamma$ simply by considering the
objective value obtained by fitting the low rank model.

But by resampling from the data, we can 
simulate the performance of the model on out of sample (missing) data 
to identify GLRMs that neither over nor underfit.
Here, we discuss a few methods for choosing model parameters by 
cross-validation; that is, by resampling from the data to evaluate
the model's performance.
Cross validation is commonly used in regression models to choose parameters such as the regularization parameter $\gamma$, as in Figure \ref{f-regpath}.
In GLRMs, cross validation can also be used to choose the rank $k$. 
Indeed, using a lower rank $k$ can be considered another form of model regularization.

We can distinguish between three sources of noise or variability in the data,
which give rise to three different resampling procedures. 
\bit
\item The \emph{rows} or \emph{columns} of the data are chosen at random, \ie, 
drawn iid from some population.
In this case it makes sense to resample the \emph{rows} or \emph{columns}.
\item The rows or columns may be fixed, but the indices of the observed entries in the matrix are chosen at random. In this case, it makes sense to resample from the observed \emph{entries} in the matrix.
\item The indices of the observed entries are fixed, but the values are observed with some measurement error. In this case, it makes sense to resample the \emph{errors} in the model.
\eit

Each of these leads to a different reasonable kind of resampling scheme. 
The first two give rise to resampling schemes based on cross validation (\ie, resampling the rows, columns,
or individual entries of the matrix) which we discuss further below.
The third gives rise to resampling schemes based on the bootstrap or jackknife procedures,
which resample from the errors or residuals after fitting the model.
A number of methods using the third kind of resampling have been proposed in order to perform inference
(\ie, generate confidence intervals) for PCA; see Josse et al.\ \cite{josse2014confidence} and references therein.

As an example, let's explore the effect of varying $|\Omega|/mn$, $\gamma$, and $k$.
We generate random data as follows.
Let $X\in\reals^{m\times k^\mathrm{true}}$, $Y\in\reals^{k^\mathrm{true} \times n}$,
and $S\in\reals^{m \times n}$,
with $m=n=300$ and $k^\mathrm{true}=3$,.
Draw the entries of $X$ and $Y$ from a standard normal distribution,
and draw the entries of the sparse outlier matrix $S$ are drawn from a uniform distribution
on $[0,3]$ with probability $0.05$, and are $0$ otherwise.
Form $A = XY + S$.
Select an observation set $\Omega$ by picking
entries in the matrix uniformly at random from $\{1,\ldots, n\}\times
\{1,\ldots,m\}$. 
We fit a rank $k$ GLRM with Huber loss and quadratic regularization $\gamma \|\cdot\|_2^2$,
varying $|\Omega|/mn$, $\gamma$, and $k$, and compute the test error.
We average our results over 5 draws from the distribution generating the data.

In Figure \ref{f-cv2}, we see that the true rank $k=3$ performs best on cross-validated error for any number of observations $|\Omega|$. 
(Here, we show performance for $\gamma=0$. The plot for other values of the regularization parameter is qualitatively the same.) 
Interestingly, it is easiest to identify the true rank with a small number of observations: higher numbers of observations make it more difficult to overfit to the data even when allowing higher ranks.

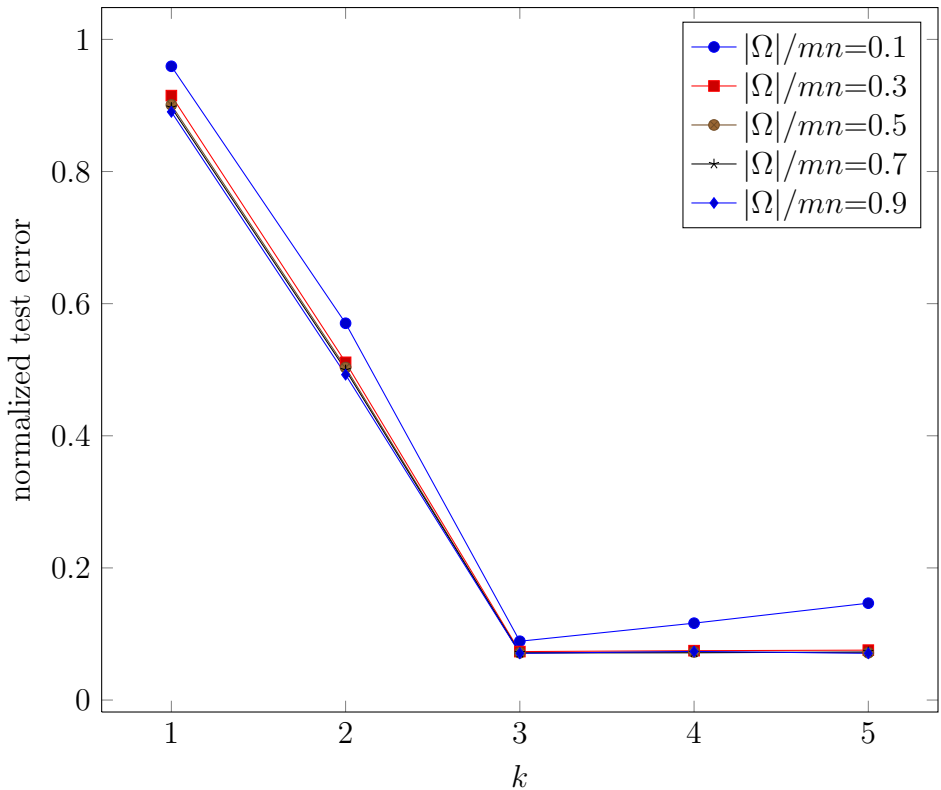
\begin{figure}[htb!]
\begin{centering}
\begin{tikzpicture}[]
    \begin{axis}[
        view = {0}{90}, 
        width=5in,
        xlabel=$k$,
        ylabel=normalized test error,
        xtick={1,2,3,4,5}
    ]
     
\addplot+ coordinates {
    (1,0.9593126702209849)
    (2,0.5703283720746729)
    (3,0.08902437274803894)
    (4,0.11636780692119926)
    (5,0.14661385052572817)
};
\addlegendentry{$|\Omega|/mn$=0.1}
\addplot+ coordinates {
    (1,0.914901508797407)
    (2,0.511192846722518)
    (3,0.07336947342541625)
    (4,0.0747247414548605)
    (5,0.07553750964759079)
};
\addlegendentry{$|\Omega|/mn$=0.3}
\addplot+ coordinates {
    (1,0.9008840003534594)
    (2,0.5032567087473229)
    (3,0.07149694257265456)
    (4,0.07277169388558886)
    (5,0.07154883662633502)
};
\addlegendentry{$|\Omega|/mn$=0.5}
\addplot+ coordinates {
    (1,0.8969909263051777)
    (2,0.49963536974406564)
    (3,0.07133607593704534)
    (4,0.07150690280052976)
    (5,0.07279099725267732)
};
\addlegendentry{$|\Omega|/mn$=0.7}
\addplot+ coordinates {
    (1,0.8902879760901993)
    (2,0.4924940852679792)
    (3,0.070745659554346)
    (4,0.07357063394113014)
    (5,0.0707679273010429)
};
\addlegendentry{$|\Omega|/mn$=0.9}
    \end{axis}
\end{tikzpicture}
\caption{\label{f-cv2}Test error as a function of $k$, for $\gamma=0$.}
\end{centering}
\end{figure}

In Figure \ref{f-cv3}, we consider the interdependence of our choice of $\gamma$ and $k$. Regularization is most important when few matrix elements have been observed: the curve for each $k$ is nearly flat when more than about 10\% of the entries have been observed, so we show here a plot for $|\Omega|=.1mn$. Here, we see that the true rank $k=3$ performs best on cross-validated error for any value of the regularization parameter. Ranks that are too high ($k>3$) benefit from increased regularization $\gamma$, whereas higher regularization hurts the performance of models with $k$ lower than the true rank. That is, regularizing the rank (small $k$) can substitute for explicit regularization of the factors (large $\gamma$).

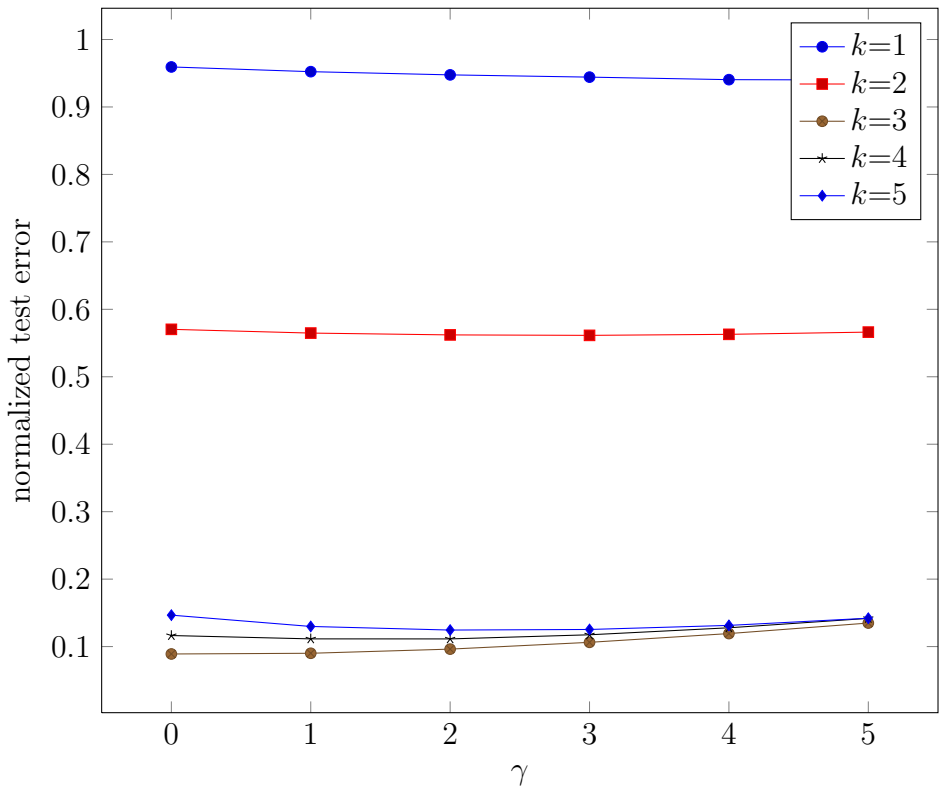
\begin{figure}[htb!]
\begin{centering}
\begin{tikzpicture}[]
    \begin{axis}[
        view = {0}{90}, 
        width=5in,
        xlabel=$\gamma$,
        ylabel=normalized test error,
        xtick={0,1,2,3,4,5}
    ]
\addplot+ coordinates {
    (0,0.9593126702209849)
    (1,0.9523237731977117)
    (2,0.9475697068902894)
    (3,0.9442279990340736)
    (4,0.9404767903269382)
    (5,0.9399891111119189)
};
\addlegendentry{$k$=1}
\addplot+ coordinates {
    (0,0.5703283720746729)
    (1,0.5647730480796055)
    (2,0.5620686687741911)
    (3,0.5613831996863705)
    (4,0.5629144165667216)
    (5,0.5662068407477036)
};
\addlegendentry{$k$=2}
\addplot+ coordinates {
    (0,0.08902437274803894)
    (1,0.0900730804384017)
    (2,0.09623360895607005)
    (3,0.10628244755440559)
    (4,0.11927338894188697)
    (5,0.13488128232200527)
};
\addlegendentry{$k$=3}
\addplot+ coordinates {
    (0,0.11636780692119926)
    (1,0.11141681061195807)
    (2,0.11135440827194938)
    (3,0.11747986343435486)
    (4,0.12784999446542716)
    (5,0.14180515164733842)
};
\addlegendentry{$k$=4}
\addplot+ coordinates {
    (0,0.14661385052572817)
    (1,0.12983132899845698)
    (2,0.12449270674496474)
    (3,0.12536422121969473)
    (4,0.13140090611203598)
    (5,0.14206256580965673)
};
\addlegendentry{$k$=5}   

    \end{axis}
\end{tikzpicture}
\caption{\label{f-cv3}Test error as a function of $\gamma$ when 10\% of entries are observed.}
\end{centering}
\end{figure}

Finally, in Figure \ref{f-cv1} we consider how the fit of the model depends on the number of observations. 
If we correctly guess the rank $k=3$, we find that the fit is insensitive to the number of observations. 
If our rank is either too high or too low, the fit improves with more observations.

\begin{figure}[htb!]
\begin{centering}
\begin{tikzpicture}[]
    \begin{axis}[
        view = {0}{90}, 
        width=5in,
        xlabel=$|\Omega|/mn$,
        ylabel=normalized test error,
        xtick={.1,.3,.5,.7,.9}
    ]
        
\addplot+ coordinates {
    (0.1,0.9593126702209849)
    (0.3,0.914901508797407)
    (0.5,0.9008840003534594)
    (0.7,0.8969909263051777)
    (0.9,0.8902879760901993)
};
\addlegendentry{k=1}
\addplot+ coordinates {
    (0.1,0.5703283720746729)
    (0.3,0.511192846722518)
    (0.5,0.5032567087473229)
    (0.7,0.49963536974406564)
    (0.9,0.4924940852679792)
};
\addlegendentry{k=2}
\addplot+ coordinates {
    (0.1,0.08902437274803894)
    (0.3,0.07336947342541625)
    (0.5,0.07149694257265456)
    (0.7,0.07133607593704534)
    (0.9,0.070745659554346)
};
\addlegendentry{k=3}
\addplot+ coordinates {
    (0.1,0.11636780692119926)
    (0.3,0.0747247414548605)
    (0.5,0.07277169388558886)
    (0.7,0.07150690280052976)
    (0.9,0.07357063394113014)
};
\addlegendentry{k=4}
\addplot+ coordinates {
    (0.1,0.14661385052572817)
    (0.3,0.07553750964759079)
    (0.5,0.07154883662633502)
    (0.7,0.07279099725267732)
    (0.9,0.0707679273010429)
};
\addlegendentry{k=5}
    \end{axis}
\end{tikzpicture}
\caption{\label{f-cv1}Test error as a function of observations $|\Omega|/mn$, for $\gamma=0$.}
\end{centering}
\end{figure}
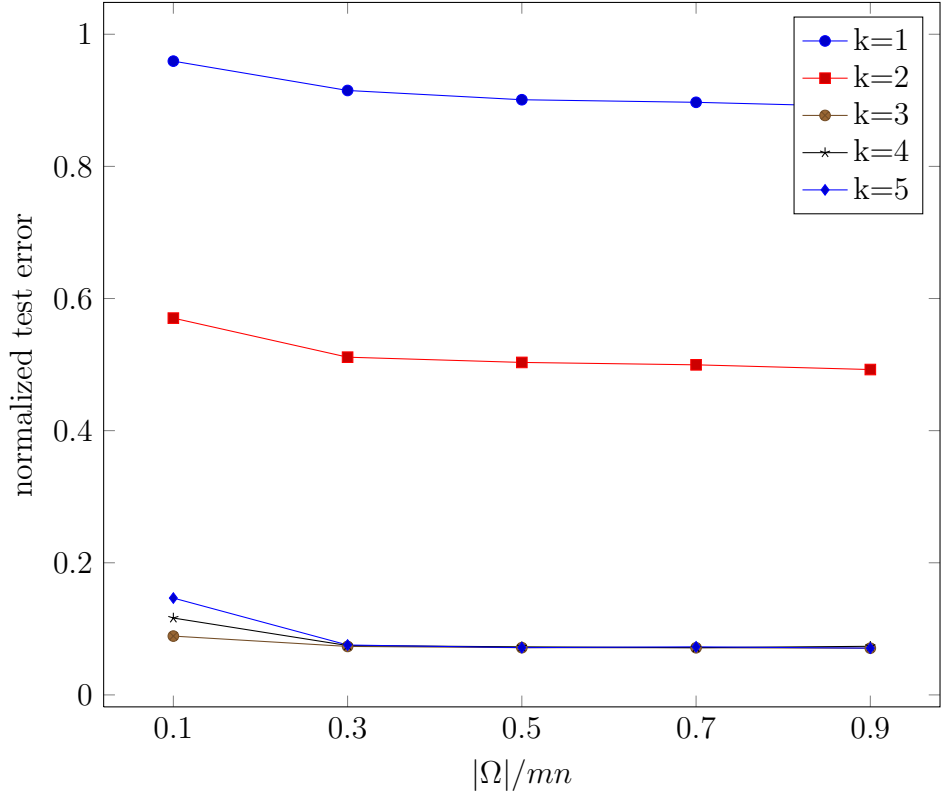

\subsection{On-line optimization}
Suppose that new examples or features are being added to our data set continuously,
and we wish to perform \emph{on-line optimization}, which means that we should have a good 
estimate at any time for the representations of 
those examples $x_i$ or features $y_j$ which we have seen.
This model is equivalent to adding new rows or columns to the data table $A$ as the algorithm continues.
In this setting, alternating minimization performs quite well, and has a very natural interpretation.
Given an estimate for $Y$,
when a new example is observed in row $i$, we may solve
\[
\begin{array}{ll}
	\mathrm{minimize} & \sum_{j: (i,j) \in \Omega} L_{ij}(A_{ij}, x y_j) + r(x)
\end{array}
\]
with variable $x$ to compute a representation for row $i$. 
This computation is exactly the same as one step of alternating minimization.
Here, we are finding the best feature representation for the new example in terms
of the (already well understood) archetypes $Y$.
If the number of other examples previously seen is large, 
the addition of a single new example should not change the optimal $Y$ by very much;
hence if $(X,Y)$ was previously the global minimum of (\ref{eq-gpca}),
this estimate of the feature representation for the new example will be very close
to its optimal representation (\ie, the one that minimizes problem (\ref{eq-gpca})).
A similar interpretation holds when new columns are added to $A$.

\section{Implementations}\label{s-implementation}

The authors have developed and released three open source codes for modelling 
and fitting generalized low rank models:
a basic serial implementation written in Python,
a serial and shared-memory parallel implementation written in Julia, 
and a distributed implementation written in Scala using the Spark framework.
The Julia and Spark implementations use the alternating proximal gradient method 
described in \S\ref{s-algorithms} to fit GLRMs,
while the Python implementation uses alternating minimization and a \verb|cvxpy|
\cite{cvxpy} backend for each subproblem.
In this section we briefly discuss these implementations, and report some timing results.
For a full description and up-to-date information about available functionality,
we encourage the reader to consult the on-line documentation for each
of these packages.

There are also many implementations available for fitting special cases of GLRMs. 
For example, an implementation capable of fitting any GLRM for which the subproblems in
an alternating minimization method are quadratic programs was recently developed 
in Spark by Debasish Das \cite{das2014}.

\subsection{Python implementation}
\verb|GLRM.py| is a Python implementation for fitting GLRMs that can be found, 
together with documentation, at
\begin{center}
\url{https://github.com/cehorn/glrm}.
\end{center}
We encourage the interested reader to consult the on-line documentation for the
most up-to-date functionality and a collection of examples.

\paragraph{Usage.}

The user initializes a GLRM by specifying 
\bit
\item the data table \verb|A| ($A$), stored as a Python list of 2-D arrays,
    where each 2-D array in \verb|A| contains all data associated
    with a particular loss function,
\item the list of loss functions \verb|L| ($L_j$, $j = 1, \ldots, n$), that
    correspond to the data as specified by \verb|A|,
\item regularizers \verb|regX| ($r$) and \verb|regY| ($\tilde r$),
\item the rank \verb|k| ($k$),
\item an optional list \verb|missing_list| with the same length as 
    \verb|A| so that each entry of
    \verb|missing_list| is a list of missing entries 
    corresponding to the data from \verb|A|, and
\item an optional convergence object \verb|converge| that characterizes the
    stopping criterion for the alternating minimization procedure.
\eit

The following example illustrates how to use \verb|GLRM.py| to fit a GLRM with
Boolean (\verb|A_bool|) and numerical (\verb|A_real|) data, with quadratic
regularization and a few missing entries.
\begin{verbatim}
    from glrm import GLRM                               # import the model
    from glrm.loss import QuadraticLoss, HingeLoss      # import losses
    from glrm.reg import QuadraticReg                   # import regularizer
    
    A = [A_bool, A_real]                                # data stored as a list
    L = [Hinge_Loss, QuadraticLoss]                     # loss function as a list
    regX, regY = QuadraticReg(0.1), QuadraticReg(0.1)   # penalty weight is 0.1
    missing_list = [[], [(0,0), (0,1)]]                 # indexed by submatrix

    model = GLRM(A, L, regX, regY, k, missing_list)     # initialize GLRM
    model.fit()                                         # fit GLRM 
\end{verbatim}
The \verb|fit()| method automatically 
adds an offset to the GLRM and scales the loss functions 
as described in \S\ref{s-gpca-scaling}.

\verb|GLRM.py| fits GLRMS by alternating minimization.
The code instantiates \verb|cvxpy| problems \cite{cvxpy}
corresponding to the $X$- and $Y$-update steps, 
then iterates by alternately solving each problem until
convergence criteria are met.  

The following loss functions and regularizers are supported by \verb|GLRM.py|:
\bit
\item quadratic loss \verb|QuadraticLoss|,
\item Huber loss \verb|HuberLoss|,
\item hinge loss \verb|HingeLoss|,
\item ordinal loss \verb|OrdinalLoss|,
\item no regularization \verb|ZeroReg|,
\item $\ell_1$ regularization \verb|LinearReg|,
\item quadratic regularization \verb|QuadraticReg|, and
\item nonnegative constraint \verb|NonnegativeReg|.
\eit
Users may implement their own loss functions (regularizers) using the abstract
class \verb|Loss| (\verb|Reg|). 

\subsection{Julia implementation}

\verb|LowRankModels| is a code written in Julia \cite{bezanson2012} 
for modelling and fitting GLRMs.
The implementation is available on-line at
\begin{center}
\url{https://github.com/madeleineudell/LowRankModels.jl}.
\end{center}
We discuss some aspects of the usage and features of the code here. 
For a full description and up-to-date information about available functionality,
we encourage the reader to consult the on-line documentation.

\paragraph{Usage.}

To form a GLRM using \verb|LowRankModels|, the user specifies
\bit
\item the data \verb|A| ($A$), which can be any array or array-like data structure (\eg, a Julia \verb|DataFrame|);
\item the observed entries \verb|obs| ($\Omega$), a list of tuples of the indices of the observed entries in the matrix,
which may be omitted if all the entries in the matrix have been observed;
\item the list of loss functions \verb|losses| ($L_j$, $j=1,\ldots,n$), one for each column of \verb|A|;
\item the regularizers \verb|rx| ($r$) and \verb|ry| ($\tilde r$); and
\item the rank \verb`k` ($k$).
\eit
For example, the following code forms and fits a $k$-means model with $k=5$ on the matrix $A \in \reals^{m \times n}$.
\begin{verbatim}
  losses = fill(quadratic(),n)   # quadratic loss
  rx = unitonesparse()           # x is 1-sparse unit vector
  ry = zeroreg()                 # y is not regularized
  glrm = GLRM(A,losses,rx,ry,k)  # form GLRM
  X,Y,ch = fit!(glrm)            # fit GLRM
\end{verbatim}

\verb|LowRankModels| uses the proximal gradient method described in \S\ref{proxgradmethod} to fit GLRMs. 
The optimal model is returned in the factors \verb|X| and \verb|Y|, while \verb`ch` gives the convergence history.
The exclamation mark suffix follows the convention in Julia
denoting that the function mutates at least one of its arguments.
In this case, it caches the best fit $X$ and $Y$ as \verb`glrm.X` and \verb`glrm.Y`
\cite{chen2014}.

Losses and regularizers must be of type \verb|Loss| and \verb|Regularizer|, respectively,
and may be chosen from a list of supported losses and regularizers, which include
\bit
\item quadratic loss \verb`quadratic`,
\item hinge loss \verb`hinge`,
\item $\ell_1$ loss \verb`l1`,
\item Huber loss \verb`huber`,
\item ordinal hinge loss \verb`ordinal_hinge`,
\item quadratic regularization \verb`quadreg`,
\item no regularization \verb`zeroreg`,
\item nonnegative constraint \verb`nonnegative`, and
\item 1-sparse constraint \verb`onesparse`.
\item unit 1-sparse constraint \verb`unitonesparse`.
\eit
Users may also implement their own losses and regularizers.

\paragraph{Shared memory parallelism.}

\verb|LowRankModels| takes advantage of Julia's \verb|SharedArray| data structure
to implement a fitting procedure that takes advantage of shared memory parallelism.
While Julia does not yet support threading, \verb|SharedArray|s in Julia allow 
separate processes on the same computer to access the same block of memory.
To fit a model using multiple processes, \verb|LowRankModels| loads the data
$A$ and the initial model $X$ and $Y$ into shared memory, broadcasts other problem
data (\eg, the losses and regularizers) to each process,
and assigns to each process a partition of the rows of $X$ and columns of $Y$.
At every iteration, each process updates its rows of $X$, its columns of $Y$,
and computes its portion of the objective function, synchronizing after each of these steps
to ensure that \eg the $X$ update is completed before the $Y$ update begins;
then the master process checks a convergence criterion and adjusts the step length.

\paragraph{Automatic modeling.}

\verb|LowRankModels| is capable of adding offsets to a GLRM, and of automatically scaling the loss 
functions, as described in \S\ref{s-gpca-scaling}.
It can also automatically detect the types of different columns of a data frame and
select an appropriate loss.
Using these features, \verb|LowRankModels| implements a method
\begin{verbatim}
  glrm(dataframe, k)
\end{verbatim}
that forms a rank $k$ model on a data frame, automatically selecting loss functions and regularization that suit the data well,
and ignoring any missing (\verb|NA|) element in the data frame.
This GLRM can then be fit with the function \verb|fit!|.

\paragraph{Example.}
As an example, we fit a GLRM to the Motivational States Questionnaire (MSQ) data set \cite{revelle1998}.
This data set measures 3896 subjects on 92 aspects of mood and personality type, 
as well as recording the time of day the data were collected. 
The data include real-valued, Boolean, and ordinal measurements, 
and approximately 6\% of the measurements are missing (\verb|NA|).

The following code loads the MSQ data set and encodes it in two dimensions:
\begin{verbatim}
  using RDatasets
  using LowRankModels
  # pick a data set
  df = RDatasets.dataset("psych","msq")
  # encode it!
  X,Y,labels,ch = fit(glrm(df,2))
\end{verbatim}

Figure~\ref{f-psychmsq} uses the rows of $Y$ as a coordinate system to plot some of the features of the data set.
Here we see the automatic embedding separates positive from negative emotions along the $y$ axis.
This embedding is notable for being interpretable despite 
having been generated completely automatically.
Of course, better embeddings may be obtained by a more careful choice of 
loss functions, regularizers, scaling, and embedding dimension $k$.

\begin{figure}[htb!]
\begin{center}
\includegraphics[width = 0.9\textwidth]{\figures/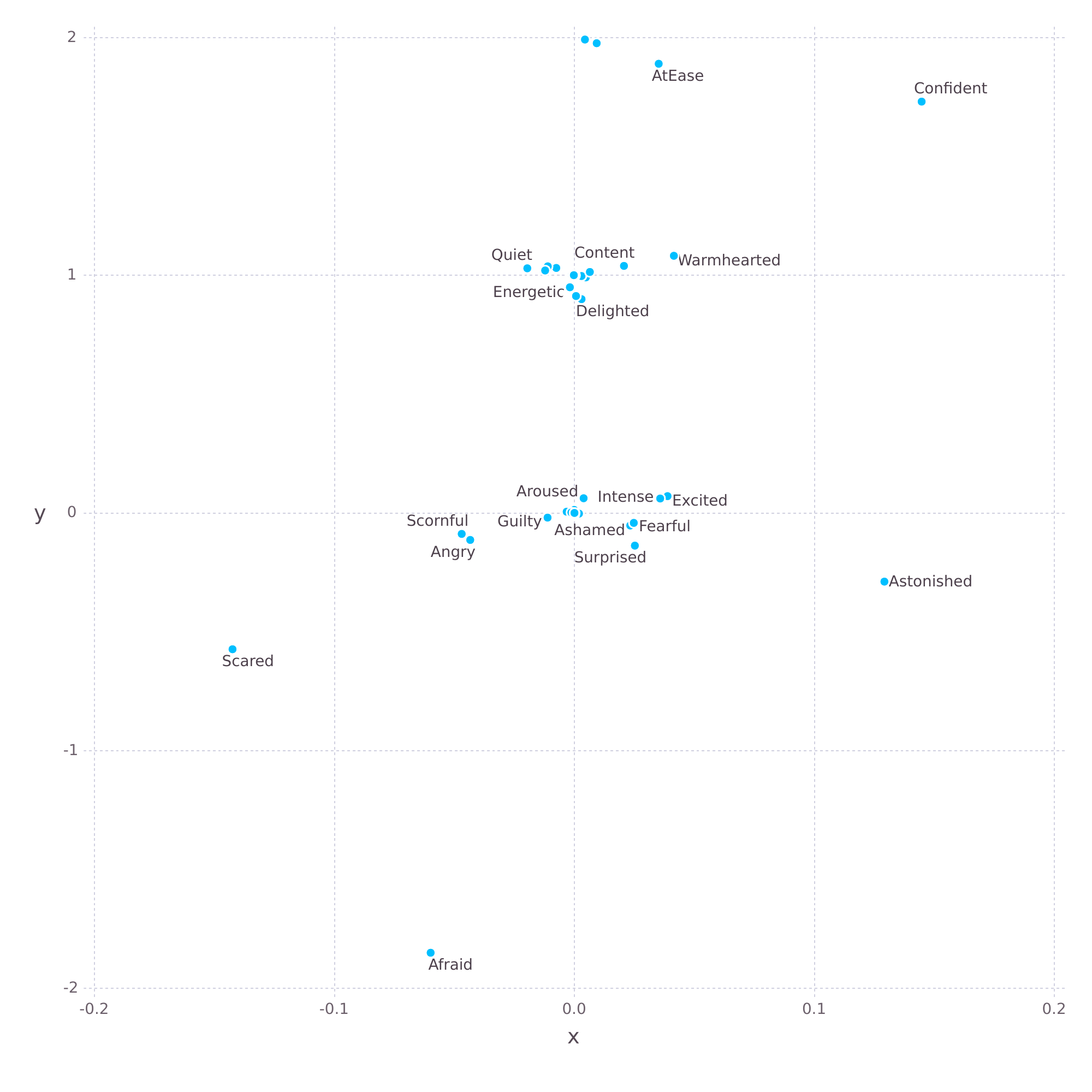}
\end{center}
\caption{\label{f-psychmsq}An automatic embedding of the MSQ \cite{revelle1998} 
data set into two dimensions.}
\end{figure}

\subsection{Spark implementation}

\texttt{SparkGLRM} is a code written in Scala,
built on the Spark cluster programming framework \cite{sparkpaper},
for modelling and fitting GLRMs.
The implementation is available on-line at
\begin{center}
\url{http://git.io/glrmspark}.
\end{center}

\paragraph{Design.} In \texttt{SparkGLRM},
the data matrix $A$ is split entry-wise across many machines, 
just as in \cite{hastiefastals}. 
The model ($X$, $Y$) is replicated and stored in memory on every machine. 
Thus the total computation time required to fit the model is proportional to 
the number of nonzeros divided by the number of cores, 
with the restriction that the model should fit in memory.
(The authors leave to future work an extension to models that do not fit
in memory, \eg, by using a parameter server \cite{factorbird}.)
Where possible, hardware acceleration (via breeze and BLAS)
is used for local linear algebraic operations. 

At every iteration, the current model is broadcast to all machines, 
so there is only one copy of the model on each machine. 
This particularly important in machines with many cores, 
because it avoids duplicating the model those machines.
Each core on a machine will process a partition of the input matrix, 
using the local copy of the model.

\paragraph{Usage.}
The user provides loss functions $L_{ij}(u,a)$
indexed by $i=0,\ldots, m-1$ and $j=0, \ldots,n-1$, 
so a different loss function can be defined for each column, or even for each entry. 
Each loss function is defined by its gradient (or a subgradient). 
The method signature is 
\begin{center}
\texttt{loss\_grad(i: Int, j: Int, u: Double, a: Double)}
\end{center}
whose implementation can be customized by particular $i$ and $j$.  
As an example, the following line implements squared error loss
($L(u,a) = 1/2 (u-a)^2$) for all entries: 
\begin{center}
\texttt{u - a} 
\end{center}

Similarly, the user provides functions implementing the proximal operator 
of the regularizers $r$ and $\tilde r$, 
which take a dense vector and perform the appropriate proximal operation. 

\paragraph{Experiments.}

We ran experiments on several large matrices. 
For size comparison, a very popular matrix in the recommender systems community is the Netflix Prize Matrix,
which has $17770$ rows, $480189$ columns, and $100480507$ nonzeros. 
Below we report results on several larger matrices, up to 10 times larger. 
The matrices are generated by fixing the dimensions and number of nonzeros per row, 
then uniformly sampling the locations for the nonzeros, 
and finally filling in those locations with a uniform random number in $[0,1]$.

We report iteration times using an Amazon EC2 cluster with 10 slaves and one master, 
of instance type ``c3.4xlarge". Each machine has 16 CPU cores and 30 GB of RAM. 
We ran \texttt{SparkGLRM} to fit two GLRMs on matrices of varying sizes.
Table~\ref{sparktable1} gives results for quadratically regularized PCA 
(\ie, quadratic loss and quadratic regularization) with $k=5$. 
To illustrate the capability to write and fit custom loss functions, 
we also fit a GLRM using a loss function that depends on the parity of $i+j$:
\[
L_{ij}(u,a) = \left\{\begin{array}{ll}
|u-a| & i+j~\mbox{is even} \\
(u-a)^2 & i+j~\mbox{is odd},
\end{array} \right .
\]
with $r(x) = \|x\|_1$ and $\tilde r(y) = \|y\|_2^2$, setting $k=10$.
(This loss function was chosen merely to illustrate the generality of the implementation. 
Usually losses will be the same for each row in the same column.)
The results for this custom GLRM are given in Table~\ref{sparktable2}.

\begin{table}[h]
\begin{center}
  \begin{tabular}{ | c | c | c | }
    \hline
    \textbf{Matrix size} & \textbf{\# nonzeros} & \textbf{Time per iteration (s)} \\ \hline
    $10^6 \times 10^6$  & $10^6$ & 7 \\ \hline
    $10^6 \times 10^6$  & $10^9$ & 11 \\ \hline
    $10^7 \times 10^7$  & $10^9$ &  227\\ \hline
  \end{tabular}
  \caption[Distributed proximal gradient GLRM]{\texttt{SparkGLRM} for quadratically regularized PCA, $k=5$.}
   \label{sparktable1}
\end{center}
\end{table}

\begin{table}[h]
\begin{center}

     \begin{tabular}{ | c | c | c | }
    \hline
    \textbf{Matrix size} & \textbf{\# nonzeros} & \textbf{Time per iteration (s)} \\ \hline
    $10^6 \times 10^6$  & $10^6$ &  9 \\ \hline
    $10^6 \times 10^6$  & $10^9$ &  13 \\ \hline
    $10^7 \times 10^7$  & $10^{9}$ & 294 \\ \hline
  \end{tabular}
	  \caption[Distributed proximal gradient GLRM]{\texttt{SparkGLRM} for custom GLRM, $k=10$.}
	   \label{sparktable2}
   \end{center}
\end{table}

The table gives the time per iteration.  
The number of iterations required for convergence depends on the size
of the ambient dimension.
On the matrices with the dimensions shown in Tables~\ref{sparktable1} and \ref{sparktable2},
convergence typically requires about 100 iterations, 
but we note that useful GLRMs often emerge after only a few tens of iterations.


\section*{Acknowledgements}
The authors are grateful to 
Chris De Sa,
Yash Deshpande,
Nicolas Gillis,
Maya Gupta,
Trevor Hastie,
Irene Kaplow,
Lester Mackey, 
Andrea Montanari,
Art Owen, 
Haesun Park,
David Price,
Chris R\'{e},
Ben Recht, 
Yoram Singer,
Nati Srebro,
Ashok Srivastava, 
Peter Stoica,
Sze-chuan Suen,
Stephen Taylor,
Joel Tropp,
Ben Van Roy,
and Stefan Wager
for a number of illuminating discussions and comments on early drafts of this paper,
and to Debasish Das and Matei Zaharia for their insights into creating a successful Spark implementation.
This work was developed with support from
the National Science Foundation Graduate Research Fellowship program (under Grant No. DGE-1147470),
the Gabilan Stanford Graduate Fellowship,
the Gerald J. Lieberman Fellowship,
and the DARPA X-DATA program.

\appendix
\section{Quadratically regularized PCA}\label{a-qpca}

In this appendix we describe some properties of the quadratically regularized PCA problem~(\ref{eq-qpca}),
\BEQ
\begin{array}{ll}
    \mbox{minimize} & \|A - XY\|_F^2 + \gamma\|X\|_F^2 + \gamma\|Y\|_F^2.
 \end{array}
\EEQ
In the sequel, we let $U \Sigma V^T = A$ be the SVD of $A$ and let $r$ be the rank of $A$.
We assume for convenience that
all the nonzero singular values $\sigma_1 > \sigma_2 > \cdots > \sigma_{r} > 0$ 
of $A$ are distinct.

\subsection{Solution}
Problem (\ref{eq-qpca}) is the only problem we will encounter that has an analytical solution. 
A solution is given by
\BEQ
\label{eq-qpca-soln}
X = \tilde U \tilde \Sigma^{1/2}, \qquad Y = \tilde \Sigma ^{1/2} \tilde V^T,
\EEQ
where $\tilde U$ and $\tilde V$ are defined as in (\ref{eq-svd-trunc}),
and $\tilde \Sigma = \diag((\sigma_1-\gamma)_+,\ldots, (\sigma_k-\gamma)_+)$.

To prove this, let's consider the optimality conditions of (\ref{eq-qpca}).
The optimality conditions are
\[
-(A - XY)Y^T + \gamma X = 0, \qquad
-(A - XY)^T X + \gamma Y^T = 0.
\]
Multiplying the first optimality condition on the left by $X^T$ and the 
second on the left by $Y$ and rearranging, we find
\[
X^T(A - XY)Y^T = \gamma X^T X, \qquad
Y (A - XY)^T X = \gamma Y Y^T,
\]
which shows, by taking a transpose, that $X^T X = Y Y^T$ at any stationary point.

We may rewrite the optimality conditions together as
\BEAS
\begin{bmatrix} -\gamma I & A \\ A^T & -\gamma I \end{bmatrix}
\begin{bmatrix} X \\ Y^T \end{bmatrix}
&=&
\begin{bmatrix} 0 & XY \\ (XY)^T & 0 \end{bmatrix}
\begin{bmatrix} X \\ Y^T \end{bmatrix}\\
&=&
\begin{bmatrix} X(YY^T) \\ Y^T(X^TX) \end{bmatrix}\\
&=&
\begin{bmatrix} X \\ Y^T \end{bmatrix} (X^TX),
\EEAS	
where we have used the fact that $X^T X = Y Y^T$.

Now we see that $(X,Y^T)$ lies in an \emph{invariant subspace} of the matrix
$\begin{bmatrix} -\gamma I & A \\ A^T & -\gamma I \end{bmatrix}$.
Recall that $V$ is an invariant subspace of a matrix $A$ if $AV = VM$ for some matrix $M$.
If $\rank(M) \leq \rank(A)$, we know that the eigenvalues of $M$ are eigenvalues of $A$,
and that the corresponding eigenvectors lie in the span of $V$.

Thus the eigenvalues of $X^TX$ must be eigenvalues of $\begin{bmatrix} -\gamma I & A \\ A^T & -\gamma I \end{bmatrix}$,
and $(X,Y^T)$ must span the corresponding eigenspace.
More concretely, notice that $\begin{bmatrix} -\gamma I & A \\ A^T & -\gamma I \end{bmatrix}$ is 
(symmetric, and therefore)
diagonalizable, with eigenvalues $- \gamma \pm \sigma_i$.
The larger eigenvalues $- \gamma + \sigma_i$ correspond to the eigenvectors $(u_i, v_i)$, 
and the smaller ones $- \gamma - \sigma_i$ to $(u_i,-v_i)$.

Now, $X^TX$ is positive semidefinite, so the eigenvalues shared by 
$X^TX$ and $\begin{bmatrix} -\gamma I & A \\ A^T & -\gamma I \end{bmatrix}$ must
be positive.
Hence there is some set $|\Omega| \leq k$ with $\sigma_i\geq\gamma$ for $i \in \Omega$ 
such that $X$ has have singular values $\sqrt{- \gamma + \sigma_i}$ for $i \in \Omega$.
(Recall that $X^T X = Y Y^T$, so $Y$ has the same singular values as $X$.)
Then $(X,Y^T)$ spans the subspace generated by 
the vectors $(u_i, v_i$ for $i \in \Omega$.
We say the stationary point $(X,Y)$ has active subspace $\Omega$. 
It is easy to verify that $XY = \sum_{i \in \Omega} u_i (\sigma_i - \gamma) v_i^T$.

Each active subspace gives rise to an orbit of stationary points.
If $(X,Y)$ is a stationary point, then $(XT,T^{-1}Y)$ is also a stationary point so long as
\[
-(A - XY)Y^TT^{-T} + \gamma X T = 0, \qquad
-(A - XY)^T X T + \gamma Y^T T^{-T} = 0,
\]
which is always true if $T^{-T} = T$, \ie, T is orthogonal.
This shows that the set of stationary points is invariant under orthogonal transformations.

To simplify what follows, we choose a representative element for each orbit.
Represent any stationary point with active subspace $\Omega$ by
\[
X = U_\Omega (\Sigma_\Omega-\gamma I)^{1/2}, \quad Y = (\Sigma_\Omega-\gamma I)^{1/2} V_\Omega^T,
\]
where by $U_\Omega$ we denote the submatrix of $U$ with columns indexed by $\Omega$,
and similarly for $\Sigma$ and $V$.
At any value of $\gamma$, let $k'(\gamma) = \max\{i: \sigma_i\geq\gamma\}$.
Then we have $\sum_{i=0}^k {k'(\gamma) \choose  i}$ (representative) stationary points,
one for each choice of $\Omega$ 
The number of (representative) stationary points is decreasing in $\gamma$;
when $\gamma > \sigma_1$, the only stationary point is $X = 0$, $Y = 0$. 

These stationary points can have quite different values.
If $(X,Y)$ has active subspace $\Omega$, then
\[
||A - XY||_F^2  + \gamma ( ||X||_F^2 + ||Y||_F^2 ) =
\sum_{i \notin \Omega} \sigma_i^2 
+ \sum_{i \in \Omega} \left( \gamma^2 + 2\gamma |\sigma_i - \gamma| \right).
\]
From this form, it is clear that we should choose $\Omega$ to include
the top singular values $i=1,\ldots,k'(\gamma)$.
Choosing any other subset $\Omega$ will result in a higher (worse) objective value:
that is, the other stationary points are \emph{not} global minima.

\subsection{Fixed points of alternating minimization}

\begin{theorem}
The quadratically regularized PCA problem (\ref{eq-qpca}) has only one local minimum, which is the global minimum. 
\end{theorem}

Our proof is similar to that of \cite{baldi1989}, who proved a related theorem for 
the case of PCA~(\ref{eq-pca}).

\begin{proof}
We showed above that every stationary point of (\ref{eq-qpca}) has the form
$XY = \sum_{i \in \Omega} u_i d_i v_i^T$, with 
$\Omega \subseteq \{1,\ldots,k'\}$, $|\Omega| \leq k$, and $d_i = \sigma_i - \gamma$.
We use the representative element from each stationary orbit described above, 
so each column of $X$ is $u_i \sqrt{d_i}$
and each row of $Y$ is $\sqrt{d_i} v_i^T$ for some $i \in \Omega$.
The columns of $X$ are orthogonal, as are the rows of $Y$.

If a stationary point is not the global minimum, then 
$\sigma_j > \sigma_i$ for some $i \in \Omega$, $j \not \in \Omega$.
Below, we show we can always find a descent direction if 
this condition holds, thus showing that the only local minimum is the global minimum.

Assume we are at a stationary point with $\sigma_j > \sigma_i$ for some $i \in \Omega$, $j \not \in \Omega$.
We will find a descent direction by perturbing $XY$ in direction $u_j v_j^T$.
Form $\tilde X$ by replacing the column of $X$ containing $u_i \sqrt{d_i}$ by $(u_i + \epsilon u_j) \sqrt{d_i}$,
and $\tilde Y$ by replacing the row of $Y$ containing $\sqrt{d_i} v_i^T$ by $\sqrt{d_i} (v_i + \epsilon v_j)^T$.
Now the regularization term increases slightly:
\BEAS
\gamma ( \|\tilde X\|_F^2 + \|\tilde Y\|_F^2 ) - \gamma ( \|X\|_F^2 + \|Y\|_F^2 )
&=& 
\sum_{i' \in \Omega, i' \ne i} \left( 2 \gamma t_{i'}\right) + 2 \gamma d_i (1 + \epsilon^2) - \sum_{i' \in \Omega} 2 \gamma t_{i'}\\
&=& 
2 \gamma d_i \epsilon^2.
\EEAS
Meanwhile, the approximation error decreases:
\BEAS
\|A - \tilde X \tilde Y\|_F^2 - \|A - XY\|_F^2
&=& 
\| u_i \sigma_i v_i^T + u_j \sigma_j v_j^T - (u_i + \epsilon u_j)d_i(v_i + \epsilon v_j)^T\|_F^2 - (\sigma_i - d_i)^2 - \sigma_j^2 \\
&=& 
\| u_i (\sigma_i - d_i) v_i^T + u_j (\sigma_j - \epsilon^2 d_i) v_j^T - \epsilon u_i d_i v_j^T - \epsilon u_j d_i v_i^T\|_F^2 \\
&& - (\sigma_i - d_i)^2 - \sigma_j^2 \\
&=& 
\left\| \begin{bmatrix} \sigma_i - d_i & -\epsilon d_i \\ -\epsilon d_i & \sigma_j - \epsilon^2 d_i \end{bmatrix}\right\|_F^2 - (\sigma_i - d_i)^2 - \sigma_j^2 \\
&=& 
(\sigma_i - d_i)^2 + (\sigma_j - \epsilon^2 d_i)^2 + 2\epsilon^2 d_i^2 - (\sigma_i - d_i)^2 - \sigma_j^2 \\
&=& 
-2 \sigma_j \epsilon^2 d_i + \epsilon^4 d_i^2 + 2\epsilon^2 d_i^2 \\
&=& 
2 \epsilon^2 d_i( d_i - \sigma_j) + \epsilon^4 d_i^2,
\EEAS
where we have used the rotational invariance of the Frobenius norm to arrive at the third equality above.
Hence the net change in the objective value in going from $(X,Y)$ to $(\tilde X,\tilde Y)$ is
\BEAS
2 \gamma d_i \epsilon^2 + 2 \epsilon^2 d_i( d_i - \sigma_j) + \epsilon^4 d_i^2 
&=& 2 \epsilon^2 d_i( \gamma + d_i - \sigma_j) + \epsilon^4 d_i^2 \\
&=& 2 \epsilon^2 d_i( \sigma_i - \sigma_j) + \epsilon^4 d_i^2,
\EEAS
which is negative for small $\epsilon$. Hence we have found a descent direction, showing that any stationary point with 
$\sigma_j > \sigma_i$ for some $i \in \Omega$, $j \not \in \Omega$ is not a local minimum.
\end{proof}

\bibliography{glrm}

\end{document}